%% file: main.tex
\documentclass{article}
\usepackage{tech2025}

\input{config}

\usepackage{graphicx}%
\usepackage{multirow}%
\usepackage{amsmath,amssymb,amsfonts}%
\usepackage{amsthm}%
\usepackage{mathrsfs}%
\usepackage[figuresright]{rotating}%
\usepackage[title]{appendix}%
\usepackage{xcolor}%
\usepackage{textcomp}%
\usepackage{manyfoot}%
\usepackage{booktabs}%
\usepackage{algorithm}%
\usepackage{algorithmicx}%
\usepackage{algpseudocode}%
\usepackage{program}%
\usepackage{listings}%

\usepackage{color}
\usepackage{float}
\usepackage{etoolbox}
\usepackage{bm}
\usepackage{enumitem}
\usepackage{pdfpages}
\usepackage{lineno}

\begin{document}

\makeatletter
\def\icmldate#1{\gdef\@icmldate{#1}}
\icmldate{\today}
\makeatother

\makeatletter
\fancypagestyle{fancytitlepage}{
  \fancyhead{}  
  \lhead{\includegraphics[height=1.5cm]{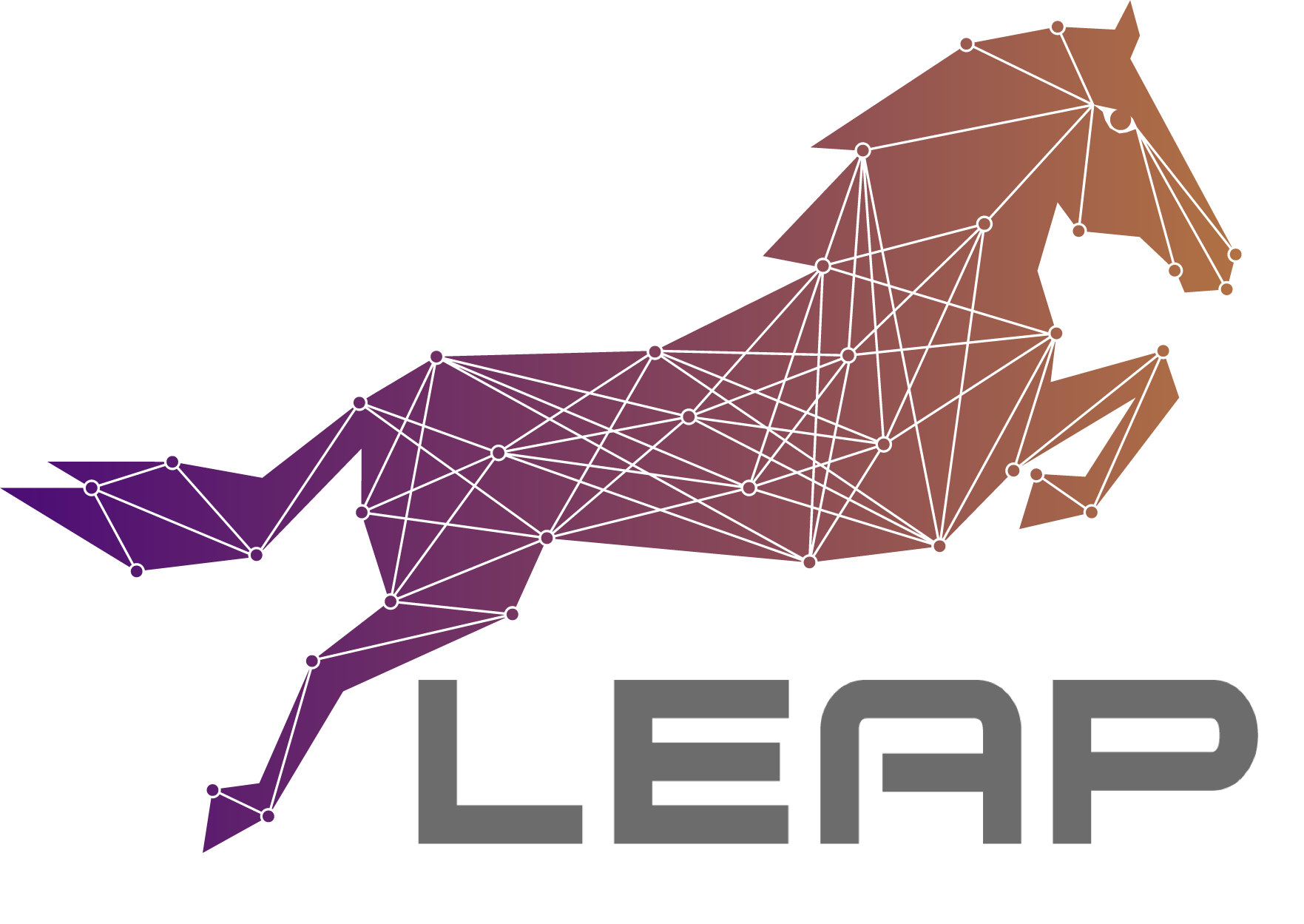}\hspace{5mm}}
  \rhead{\it \@icmldate}
  \cfoot{}
}
\makeatother


\icmltitlerunning{
    Emulating Human-like Adaptive Vision for 
    Efficient and Flexible Machine Visual Perception
}
\vskip 0.05in
\icmltitle{
    Emulating Human-like Adaptive Vision for \\[0.65ex]
    Efficient and Flexible Machine Visual Perception
}


\begin{icmlauthorlist}
\mbox{
  Yulin Wang$^{\dagger}$
  }
\mbox{
  Yang Yue$^{\dagger}$
  }
\mbox{
  Yang Yue$^{\dagger}$
  }
\mbox{
  Huanqian Wang
  }
\mbox{
  Haojun Jiang
  }
\mbox{
  Yizeng Han
  }
\mbox{
  Zanlin Ni
  }
\mbox{
  Yifan Pu
  }
\mbox{
  Minglei Shi
  }
\mbox{
  Rui Lu
  }
\mbox{
  Qisen Yang
  }
\mbox{
  Andrew Zhao
  }
\mbox{
  Zhuofan Xia
  }
\mbox{
  Shiji Song\ \!$^{\textrm{\Letter}}$
  }
\mbox{
  Gao Huang\ \!$^{\textrm{\Letter}}$
  }
\end{icmlauthorlist}
{
  \small
  Learning And Perception (LEAP) Lab, Department of Automation, Tsinghua University\\
  ${\dagger}$ Equal Contribution. \ \ \ \ 
  ${\textrm{\Letter}}$ Corresponding Authors: Shiji Song, Gao Huang.\\
  \texttt{\{yulin-wang, shijis, gaohuang\}@tsinghua.edu.cn}.
}




\begin{leapabstract}
    Vision is fundamental to our interpretation of the intricate physical world \cite{biederman1972perceiving, sperling1978attention, sagi1985and, moran1985selective, olveczky2003segregation, moore2003selective, najemnik2005optimal, carrasco2011visual, wolfe2017five}. 
    Computationally replicating human visual perception capabilities is crucial for modern artificial intelligence (AI), such as multimodal large language models (MLLM) \cite{alayrac2022flamingo, openai2023gpt4, team2023gemini, lu2024multimodal}, embodied AI agents \cite{kaufmann2023champion, zitkovich2023rt, open_x_embodiment_rt_x_2023, gehrig2024low}, and medical AI \cite{chen2020deep, xu2024whole, wang2024pathology, schafer2024overcoming}, and also carries significant implications for cognitive science \cite{lake2015human, orhan2024learning, vong2024grounded}. 
    Current methods have demonstrated notable success in addressing challenging vision tasks by continuously scaling up input complexity (\emph{e.g.,} spatial-temporal resolutions) and model size \cite{he2022masked, kirillov2023segment, pmlr-v202-dehghani23a, liu2024improved, oquab2024dinov}, at the price of dramatically growing resource demands. 
    Information-rich, high-dimensional visual inputs, large-scale neural networks, and efficiency have converged to an `impossible triangle' that impedes both future advancements of computer vision and its adoption in diverse real-world scenarios such as robotics, wearable devices, and industrial inspections, even posing a risk to human life by making high-latency decisions in safety-critical domains like autonomous vehicles and medical robots. 
    Here we find this dilemma may be rooted in the current prevailing representation learning paradigm established decades ago \cite{rumelhart1986learning, lecun1989handwritten, lecun1995convolutional, lecun1998gradient}: the model passively receives an input, and processes the whole input in its entirety at once, yielding computational and memory costs that scale linearly or quadratically with pixel numbers.
    To address this, we take inspiration from the human visual system and introduce an AdaptiveNN framework, aiming to drive a paradigm shift from `passive' to `active' vision models. AdaptiveNN formulates visual perception as a coarse-to-fine sequential decision-making process, progressively identifying and fixating on regions pertinent to a given task, incrementally combining information across fixations, and actively concluding its observation when sufficient to accomplish the task. 
    Hence, akin to human vision, large models can be employed for superior capabilities, yet their inference remains low-cost since they only process a minimally necessary subset of regions within the complex scenes. 
    We introduce a novel theoretical analysis integrating representation learning with self-rewarding reinforcement learning, which enables training the non-differentiable AdaptiveNN in end-to-end without relying on specialized task structures or additional annotations beyond standard objectives.
    AdaptiveNN is assessed across 17 benchmarks organized into 9 different tasks, including large-scale visual understanding, fine-grained recognition, visual search, processing images from real driving and medical scenarios, MLLM for language-driven embodied AI, and multiple side-by-side comparisons with humans. 
    AdaptiveNN reduces the inference cost of well-performing models by up to $28\times$ without sacrificing accuracy, especially effective for processing complicated real-world scenes, and for employing large models. It also exhibits marked behavioral flexibility to adapt to varying task instructions and fluctuating resource availability without re-training, and achieves strong interpretability through analyzing its fixation patterns. 
    Furthermore, the perceptual behaviors of AdaptiveNN are indistinguishable from people in many cases, uncovering its potential as a useful instrument for investigating human visual cognition.
    Our findings reveal that incorporating human-like adaptiveness offers a promising avenue toward the next generation of efficient, flexible, and interpretable machine vision paradigms, while also providing valuable insights for the cognitive science community.
    Code is available at \url{https://github.com/LeapLabTHU/AdaptiveNN}.

    
\end{leapabstract}


\vskip -0.1in
\section{Main}
\label{sec:main}

Through visual perception, humans interpret complex surrounding environments, learn knowledge about how the physical world works, connect language or concepts with tangible objects and scenes, and guide their behaviors \cite{biederman1972perceiving, sperling1978attention, sagi1985and, moran1985selective, olveczky2003segregation, moore2003selective, najemnik2005optimal, carrasco2011visual, wolfe2017five, lecun2022path}.
Computationally acquiring these visual perception capabilities of humans has been crucial for advancing modern artificial intelligence, such as developing multimodal large language models (MLLM) that understand visual inputs \cite{alayrac2022flamingo, openai2023gpt4, team2023gemini, lu2024multimodal}, embodied AI agents that perceive and interact with the real world \cite{kaufmann2023champion, zitkovich2023rt, open_x_embodiment_rt_x_2023, gehrig2024low}, and AI applications in pathology, radiology, and medical robots \cite{chen2020deep, xu2024whole, wang2024pathology, schafer2024overcoming}.
Computer vision also presents significant opportunities for exploring fundamental questions in cognitive science, such as the role of innateness in human vision \cite{lake2015human, orhan2024learning, vong2024grounded}. 

Over the past decades, computational visual perception models have exhibited substantial progress, approaching or even exceeding expert-level performance across a broad range of fields, including large-scale image recognition \cite{russakovsky2015imagenet, he2016deep, huang2017densely, dosovitskiy2021an, pmlr-v202-dehghani23a}, object detection \cite{zou2023object}, open-world visual recognition \cite{radford2021learning}, medical image analysis \cite{isensee2021nnu, tiu2022expert, zhou2023foundation, xu2024whole, wang2024pathology, schafer2024overcoming}, and multimodal content understanding \cite{alayrac2022flamingo, openai2023gpt4, team2023gemini, lu2024multimodal}.
These achievements are founded on the paradigm of representation learning, where parameterized functions are learned to transform raw pixelated images into semantically meaningful representations. This idea originated at least four decades ago \cite{rumelhart1986learning}, but has become dramatically more powerful today because of the breakthroughs in algorithms and hardware, enabling training vastly deeper and larger neural networks to effectively harness large-scale, fine-grained digital visual signals with much higher spatial-temporal resolution \cite{he2022masked, kirillov2023segment, pmlr-v202-dehghani23a, liu2024improved, oquab2024dinov}.
These advancements have sparked interest in deploying deep networks in diverse real-world applications, such as generalist multimodal AI copilots \cite{openai2023gpt4, team2023gemini, lu2024multimodal}, autonomous vehicles and robotics \cite{chen2020deep, kaufmann2023champion, zitkovich2023rt, open_x_embodiment_rt_x_2023, gehrig2024low}, wearable devices \cite{du2015shidiannao, bai2017smart}, mobile applications \cite{howard2017mobilenets, sandler2018mobilenetv2, huang2018condensenet}, and edge computing \cite{chen2019deep, wang2020convergence, murshed2021machine}.

However, models achieving state-of-the-art accuracy often fall short in meeting the demands of real-world applications that extend beyond simple performance metrics. For example, in scenarios like robotics, mobile AI copilots, and industrial inspections, the hardware deployment platforms typically face constraints on computational capability, memory space, and battery capacity, yet the AI systems usually necessitate acting in real-time and performing low-latency interactions with human users and physical environments. In contrast, the inference of large computer vision models demands substantial resources, as it involves activating millions or billions of parameters to process high-resolution images with high frame rates, leading to tremendous power consumption, considerable GPU memory requirements, and nontrivial time delays. These limitations make it challenging to deploy highly capable, scaled-up models in real systems, and may even pose a risk to human life by making high-latency decisions in safety-critical domains like autonomous driving and medical robots. While cloud computing could offer some solutions, it introduces notable network latency and dependence on high-bandwidth, real-time wireless communications. Furthermore, large-scale inference requests of computationally intensive models ultimately translate to a significant rise in carbon emissions, which should be minimized for environmental reasons \cite{bourzac2024fixing}.

\begin{figure*}[!t]
    \centering
    \includegraphics[width=\textwidth]{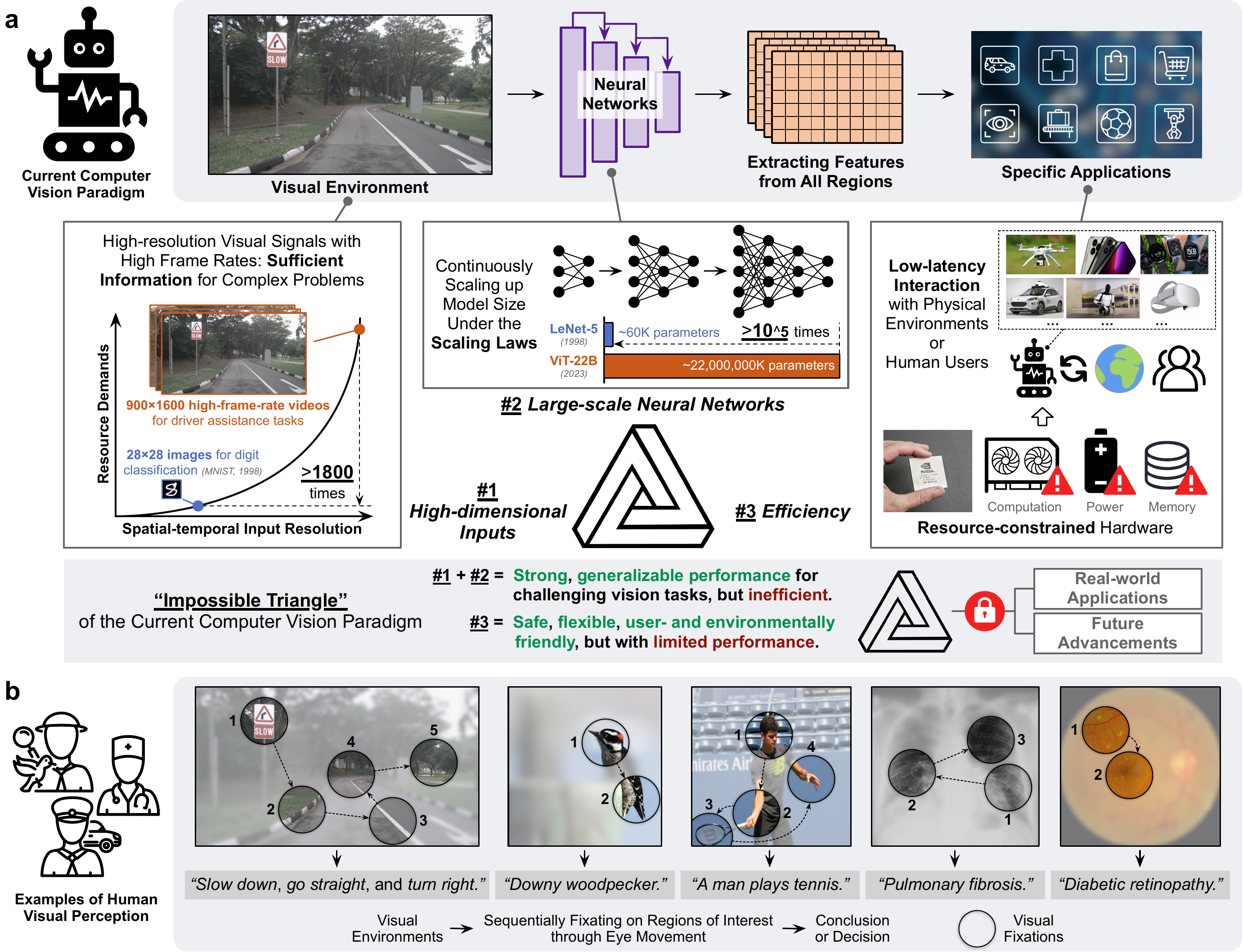}
    \caption{
        \textbf{The `impossible triangle' faced by the current paradigm of computational visual perception.}
        \textbf{(a) \textit{(top half)}} The current prevailing paradigm established decades ago \cite{rumelhart1986learning, lecun1989handwritten, lecun1995convolutional, lecun1998gradient}: a model processes the whole image in its entirety at once, with all pixels fed into neural networks simultaneously and processed parallelly, extracting features from all regions for downstream applications. All regions are equivalent in computation.
        \textbf{(a) \textit{(bottom half)}} However, an `impossible triangle' has emerged under this paradigm, which impedes both future advancements and adoption in diverse real-world scenarios. Specifically, continuously scaling up model size and input complexity (\emph{e.g.}, spatial-temporal resolutions) yields superior capabilities for addressing challenging real-world vision tasks, but usually compromises efficiency, leading to dramatically growing resource demands.
        \textbf{(b)} The human visual system circumvents this `impossible triangle' by utilizing an active and adaptive perception strategy, which does not process everything everywhere all at once. Instead, human vision only acquires information when and where it is needed, which is implemented by sequentially sampling the optic array, progressively directing a high-resolution fovea toward a few regions of interest through eye movement, until the observation is sufficient.
        \label{fig:fig1}
        }
    \vskip 0.2in
\end{figure*}

A foundational source of these aforementioned inefficiencies is rooted in a prevalent routine within the current computer vision community, which stems from a straightforward extension of the basic representation learning paradigm established decades ago \cite{rumelhart1986learning, lecun1989handwritten, lecun1995convolutional, lecun1998gradient}: models usually process a whole image or video in its entirety at once, where all pixels of the input are fed into a model simultaneously and processed parallelly, equivalent in computation (Fig. \ref{fig:fig1}a). Consequently, the computational complexity and memory requirements of the same model, for both training and inference, scale linearly with the number of pixels, and thus quadratically with respect to the image height or width. Historically, this posed little concern two or three decades ago, when small neural networks with merely thousands of parameters were employed to classify tiny images, such as 28$\times$28 black and white handwritten digits \cite{lecun1998mnist, lecun1989handwritten, lecun1995convolutional, lecun1998gradient}. However, this has evolved into a critical limitation in modern contexts, as current models have been enlarged by 5-6 orders of magnitude in terms of parameters, necessitating proficiency in processing complex scenes or videos sourced from real-world environments or the internet \cite{pmlr-v202-dehghani23a, oquab2024dinov}. For example, compared with 28$\times$28 images, 224$\times$224, a small yet reasonable size for common web images depicting individual objects or animals \cite{russakovsky2015imagenet}, results in a 64 times increase in computational and memory demands, while 900$\times$1600, a relatively small size for depicting common urban driving scenes, incurs a more than 1,800 times larger resource cost. The challenge of inefficiency becomes even more exacerbated with the recent revealing of the scaling laws of neural networks \cite{kaplan2020scaling, openai2023gpt4, chen2024internvl, oquab2024dinov}, that is, continuously scaling up model size may be essential for acquiring strong, generalizable capabilities across diverse tasks. This finding, coupled with the introduction of high-resolution inputs, is driving computational and memory requirements to unaffordable levels. In summary, the increasing demands of higher spatial-temporal resolution for inputs, the rise of larger-scale models, and the necessity of efficiency in real-world applications have formed an `impossible triangle' (Fig. \ref{fig:fig1}a), which emerges as a major bottleneck faced by the current paradigm of machine visual perception, and its impact is expected to further markedly intensify in the future. A paradigm shift may be necessary for either addressing the immediate pressing application needs or facilitating future advancements.

In this article, we seek to draw inspiration from the human visual system to break through the effectiveness-efficiency trade-off dilemma inherent in the current computer vision paradigm. When interpreting the complex surrounding environments, unlike prevailing neural networks, human vision does not process everything everywhere all at once. Instead, human vision adopts a much smarter active and selective strategy (Fig. \ref{fig:fig1}b), sequentially sampling visual inputs by shifting a small, high-resolution fovea toward a few regions or objects of interest, and constructing a perception of the visual environment by combining information from different fixations over time \cite{biederman1972perceiving, sagi1985and, olveczky2003segregation, najemnik2005optimal, carrasco2011visual, wolfe2017five}. This evolved visual system enables the effective filtration of pertinent signals from extraneous information \cite{ward2002fast, moore2003selective, ma2011behavior, henderson2017meaning}, markedly diminishing the complexity encountered in processing the vast spectrum of visual data presented by the environment \cite{wolfe2004attributes, hanning2023dissociable}. Ultimately, regardless of the complexity of the original visual environment, the resource demands of human visual perception are generally determined by the `bandwidth' and total number of fixations: the former has been pre-defined as a proper size for efficient processing, while the latter can be minimized by only acquiring information when and where it is essential for specific tasks. Thus, the human visual system not only incorporates tremendous numbers of neurons and demonstrates remarkable capabilities, but can also efficiently process the highly complex visual scenes presented by the real world, without being affected by the `impossible triangle' limitation (Fig. \ref{fig:fig1}a) faced by modern computer vision models.

As early as 2015, LeCun, Bengio, and Hinton (in the `The future of deep learning' section in ref. \cite{lecun2015deep}) have famously argued that: in the future, computer vision systems of AI are expected to attain much progress by emulating human vision to sequentially and actively decide where to look in an intelligent, task-specific way. However, nearly a decade later, the significant potential of developing human-like adaptive visual systems has not yet received adequate attention. Early studies \cite{mnih2014recurrent, ba2014multiple} have preliminarily indicated the promise of this direction using small models and tiny experiments, such as classifying handwritten digits, but there remains a huge gap between these initial efforts and modern large-scale neural networks and real-world-level application scenarios. More recently, several works have also sought to introduce adaptiveness into computer vision models \cite{huang2018multiscale, li2019improved, yang2020resolution, wang2020glance, huang2022glance, wang2021not, pan2021ia, rao2021dynamicvit, Evo-vit, fayyaz2022adaptive, yin2022vit, bolya2023token}, but most of them consider only an incomplete modeling of the adaptive capabilities of the human visual system, usually resulting in only modest improvements in computational efficiency. Besides, these approaches tend to focus on technical solutions tailored for specific network architectures or tasks, without offering in-depth, widely applicable theoretical and empirical insights that could guide the design and training of adaptive models in broader fields.
More details of existing works can be found in Supplementary Section \textcolor{blue}{A}. In conclusion, there is an urgent need to establish a unified computational framework that is underpinned by clear motivations, offers flexible generalizability across diverse network architectures and tasks, and is grounded in sound theoretical learning principles, demonstrating how the AI and computer vision communities can leverage adaptive visual perception models to address the effectiveness-efficiency trade-off dilemma behind the `impossible triangle' challenge (Fig. \ref{fig:fig1}a).

In response to these pressing needs, we develop a novel AdaptiveNN framework. AdaptiveNN presents a new computer vision paradigm that inherently incorporates human-like adaptive perception behaviors, and its formulation is general enough to be compatible with various network architectures (\emph{e.g.}, Transformers and convolutional neural networks) and tasks (\emph{e.g.}, as stand-alone perceptual models or as the basis of multimodal large language models, being applied to static images and videos, or interacting with dynamic environments, such as for robotics). 
In specific, given a visual environment (an image or video frame), AdaptiveNN formulates visual perception as a multi-step dynamic decision-making process, sequentially fixating on the regions of interest, incrementally combining information across fixations to build up a continuously updated internal representation, and actively determining when is sufficient to conclude observation. At each step, the model summarizes past information to decide if further observation is warranted; if so, the position of the next visual fixation will be selected, and the region around there will be processed with a high-capacity feature-extraction module to update the internal representation of AdaptiveNN accordingly. 
Similar to the human visual system, AdaptiveNN focuses resources selectively on some important parts of the visual environment captured by several fixations, whose number is dynamically adjusted depending on the difficulty of accomplishing the task at hand on top of each specific sample. The resource demands of its inference process are independent of the size, or complexity, of the visual environment to perceive. Hence, compared with the current prevailing paradigm that processes the full visual environment all at once, AdaptiveNN enables preserving the superior accuracy of large-scale neural networks with high-resolution inputs, but remains low-cost during inference by strategically selecting `where to look', thus minimally suffering from the effectiveness-efficiency trade-off dilemma in prior methods.
We build a theoretical analysis that enables training AdaptiveNN directly in end-to-end to maximize the performance measure of a given task, typically defined as a loss function. This optimization procedure is general as it does not rely on specialized task formats or additional annotations beyond the task objective itself. On the contrary, we prove that when considering optimizing perceptual behavior distributions for an arbitrary vision task, an integration of representation learning and self-rewarding reinforcement learning naturally emerges as a major learning principle.
The former trains the feature-extraction and representation-updating components of AdaptiveNN, while the latter addresses the non-differentiability of learning to select fixation positions and adaptively conclude observation.

Our comprehensive experimental evaluation reveals that AdaptiveNN, across a diverse spectrum of scenarios, demonstrates markedly improved energy efficiency, flexibility, and interpretability. These features align with the recognized advantages of human visual systems \cite{ward2002fast, wolfe2004attributes, najemnik2005optimal, ma2011behavior, mnih2014recurrent, henderson2017meaning, gottlieb2018towards, hanning2023dissociable}. Our key results include:
\begin{itemize}[itemsep=0pt,topsep=0pt,parsep=0pt]
    \item First, on seven popular benchmarks, ranging from general large-scale real-world visual understanding to fine-grained visual recognition, AdaptiveNN with active perception capabilities achieves a computational cost reduction of $4-8\times$ without compromising accuracy, compared to existing `passive' \cite{lecun2015deep} vision models. When considering a more general visual perception scenario in the wild -- processing non-object-centric road-scene images collected on real moving vehicles, AdaptiveNN exhibits a remarkable speedup of $28\times$ in the sign recognition task. 
    \item Second, our model is distinctive in its human-like flexibility. Confronting circumstances where the quantities of available resources vary dynamically, AdaptiveNN can adjust its inference cost online (by simply varying the statistical distributions of fixation numbers) without necessitating additional training, yielding a favorable efficiency-effectiveness trade-off among a wide range. This enables AdaptiveNN to dynamically make full use of all available resources or obtain the required performance with minimal power consumption. Besides, to mimic the cases where the demands of vision tasks are highly diversified, we consider a visual search scenario where the categories and numbers of targets can be flexibly changed. AdaptiveNN maintains an average success of $\sim90\%$ consistently across all different tasks, while existing methods \cite{mnih2014recurrent, ba2014multiple} usually struggle for $\sim20\%$.
    \item Third, visualization results uncover that the model's visual fixation patterns offer a critical window into interpreting the decision-making processes of our model, which is also a major strategy for understanding human vision \cite{ward2002fast, gottlieb2018towards, valliappan2020accelerating}. Moreover, we demonstrate the applicability of AdaptiveNN in applications where interpretability holds vital importance, such as medical diagnosis. In pneumonia detection, for example, the visual fixations of AdaptiveNN learned with only classification labels succeed in localizing lung lesions, and these results are consistent with the judgment of human clinicians.
    \item Last but not least, to thoroughly uncover the potential of our framework, we employ AdaptiveNN to establish an embodied multimodal large language model (MLLM). The MLLM receives text instructions as control signals, adaptively perceives the visual environments, and accordingly interacts with the environment to execute complex robot manipulation tasks. In alignment with previous findings, AdaptiveNN reduces inference cost by $4-6\times$ without sacrificing performance, exhibits marked behavioral flexibility to adapt to diverse task instructions and fluctuating resource availability, and offers enhanced interpretability.
\end{itemize}
In summary, we believe that these superiorities of AdaptiveNN demonstrate a practical new avenue toward the next generation of energy-efficient, flexible, and interpretable computational visual perception paradigms.

Beyond the aforementioned merits, we believe that AdaptiveNN also emerges as a potent computational instrument for probing into human behavioral and learning processes. For example, we evaluate humans and AdaptiveNN side by side on the same tests of visual perception behaviors, where our model is learned exclusively on large-scale, object-centric visual recognition tasks. Our quantitative analysis reveals that, in many cases, AdaptiveNN performs mostly consistent with human vision, in terms of both the locations of visual regions it fixates on and the difficulty level it assesses to accomplish the given task based on each individual visual environment. AdaptiveNN produces nuanced human-like patterns such as being attracted by faces, hands, human bodies, or human actions. Hence, highly human-like behaviors of actively observing objects and scenes are learnable through being trained to efficiently fulfill routine vision tasks like recognition, without the guidance of other innate inductive biases (\emph{e.g.}, biases concerning objects, agents, space, and biological motion \cite{kellman1983perception, spelke1992origins, spelke1994initial, viola2004can, simion2011processing, ullman2012simple, stahl2015observing, reynolds2018development}). These insights suggest AdaptiveNN's potential as a valuable tool for advancing the understanding of some fundamental questions in human visual cognition.

\clearpage

\begin{figure*}[!t]
    \centering
    \vskip -0.4in
    \includegraphics[width=\textwidth]{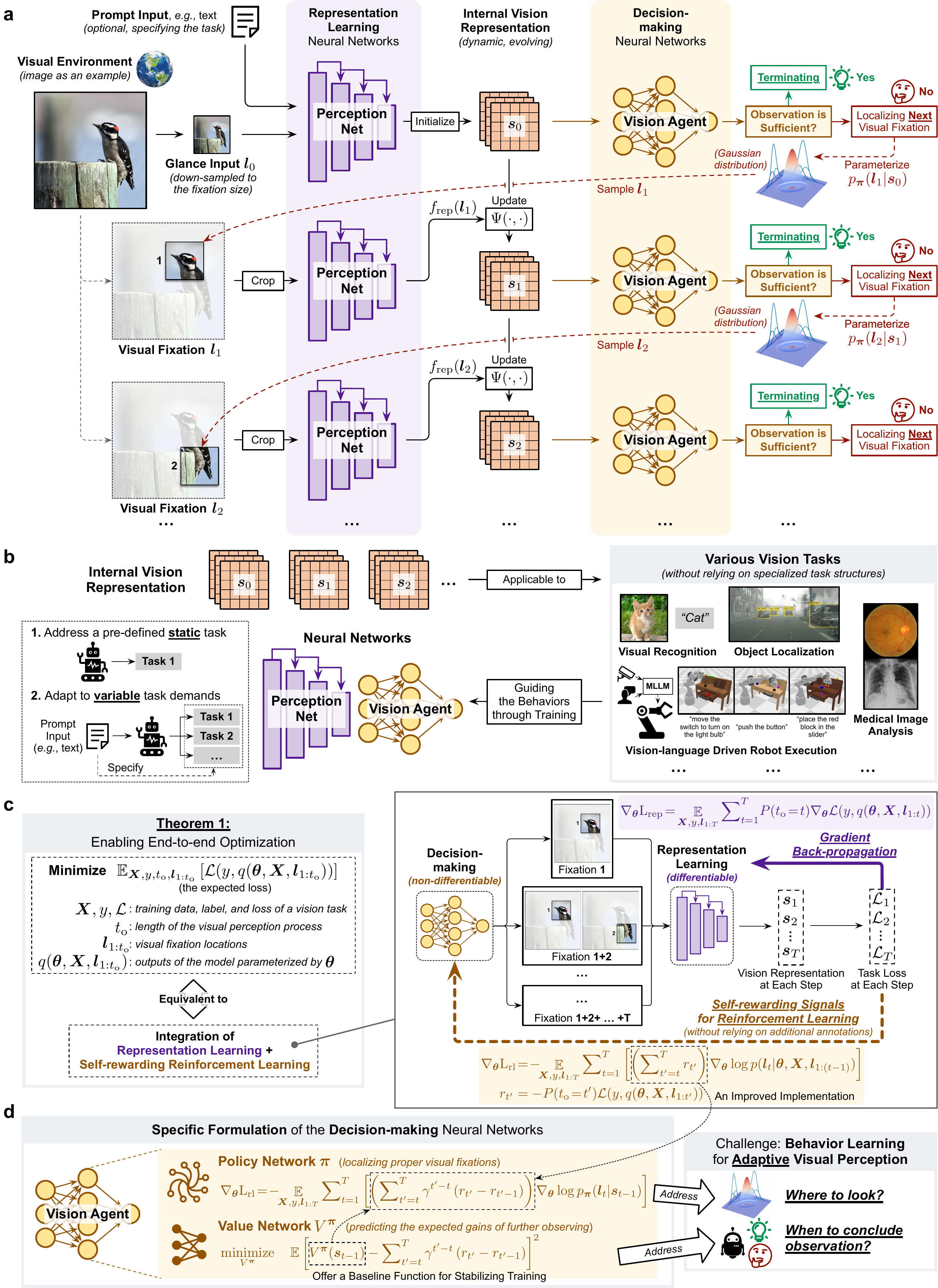}
    \vskip -0.3in
\end{figure*}

\begin{figure*}[!t]
    \caption{
        \textbf{Schematic overview of AdaptiveNN.}
        \label{fig:fig1_2}
    \textbf{(a)} The architecture and inference procedure of AdaptiveNN. The model iteratively identifies new valuable regions to fixate on, and actively determines the appropriate time to conclude its observation. The information from all processed fixations is incrementally combined, forming a dynamic, evolving internal vision representation. The full sequential perception procedure initiates with a quick glance at the visual environment, emulating the coarse-to-fine paradigm of the conscious perception of human vision \cite{navon1977forest, chen1982topological, hochstein2002view, ganel2003visual, oliva2006building, peelen2024predictive}.
    \textbf{(b)} AdaptiveNN is compatible with a broad range of vision tasks, including both pre-defined static tasks and tasks with variable demands specified by prompt inputs (\emph{e.g.}, text). The behaviors of AdaptiveNN are learned under task-driven supervision signals. 
    \textbf{(c)} The training of AdaptiveNN is challenging as it incorporates both continuous and discrete optimization. We address this by developing a novel theoretical analysis that decomposes the expected loss into an integration of representation learning and self-rewarding reinforcement learning objectives. Our method enables training AdaptiveNN in end-to-end without relying on specialized task formats or additional annotations beyond standard objectives.
    \textbf{(d)} In implementation, we formulate the vision agent as the combination of a policy network $\bm{\pi}$ and a value network $V^{\bm{\pi}}$, which addresses `where to look' and `when to conclude observation' simultaneously, and also facilitates a stabilized reinforcement learning process.
    }
\end{figure*}

\section{AdaptiveNN}

In this section, we first briefly describe the inference procedure and major components of our proposed AdaptiveNN framework (Section \ref{sec:arch_AdaNN}). Then we introduce its theoretical learning principles (Section \ref{sec:formulation_AdaNN}). More details can be found in Section \ref{sec:method}.

\subsection{Framework}
\label{sec:arch_AdaNN}

\noindent \textbf{Inference of AdaptiveNN}.
We start by describing AdaptiveNN's overall inference procedure. AdaptiveNN aims to drive a paradigm shift from `passive' to `active and adaptive' computer vision models. The key insight behind AdaptiveNN is to mimic the human visual system, modeling visual perception as a coarse-to-fine sequential decision-making process, rather than only receiving inputs `passively' and processing all input regions in parallel.

Specifically, as shown in Fig. \ref{fig:fig1_2}a, consider a generic visual environment structured as an $H\!\times\!W$ scene. 
AdaptiveNN constructs a perception of it by recurrently attending to several selected locations within it, and incrementally combining information from these fixations over time to build up a dynamic, evolving internal vision representation of the scene. This procedure is formulated as a sequential decision-making process. 
At each step, a \emph{Vision Agent} processes the current composite vision representation, and determines whether the observation on the environment is sufficient enough to be terminated based on the information of previous steps and the task demands.
If more information needs to be acquired from the environment, \emph{Vision Agent} will select the next location to fixate on; conversely, the perception process on the environment will not proceed, with the current vision representation leveraged to address the task of interest. Each selected visual fixation will be processed by a high-capacity representation learning neural network (\emph{Perception Net}) for extracting discriminative local features to update the internal vision representation. Notably, without loss of generality, a visual fixation is defined as a $P\!\times\!P$ patch ($P\!<\!H,W$) to be compatible with most modern deep learning scenarios \cite{lecun2015deep, he2016deep, huang2017densely, dosovitskiy2021an}. Furthermore, the full sequential process initiates with a quick glance, where a network coarsely processes an unknown scene in a down-sampled scale to establish an initial representation. 
This design is introduced inspired by the prominent theory that human vision operates in a global-to-local, coarse-to-fine manner \cite{navon1977forest, chen1982topological, hochstein2002view, ganel2003visual, oliva2006building, peelen2024predictive}, where humans' initial conscious perception (vision at a glance \cite{hochstein2002view}) matches a high-level, generalized, abstract scene interpretation, while later vision guides serial eye movements to attend to low-level, specific, fine receptive fields, incorporating the detailed information available there into conscious perception.

In short, mimicking human vision, AdaptiveNN observes a complex visual environment through iteratively localizing and processing visual fixations, and actively deciding when its knowledge about the scene is adequate for fulfilling the given task. 

\vskip 0.1in

\noindent \textbf{Components of AdaptiveNN.}
Here we briefly describe the major components of AdaptiveNN. Their detailed architectures are deferred to Section \ref{sec:imple_detail_arch}.

\textbf{Visual fixations} $\bm{l}_{1},\ldots,\bm{l}_{t}$ (at $1^{\textnormal{st}},\ldots,t^{\textnormal{th}}$ steps).
AdaptiveNN never senses the visual environment in its entirety. In contrast, it extracts information from a sequence of smaller, bandwidth-limited inputs corresponding to certain local regions of the environment, named visual fixations, denoted by $\bm{l}_{1},\ldots,\bm{l}_{t}$. AdaptiveNN actively determines the locations of $\bm{l}_{1},\ldots,\bm{l}_{t}$ step by step, under the goal of maximizing their contributions to the task of interest, until sufficient information has been acquired. The small bandwidth of visual fixations ensures that the resource demands of AdaptiveNN can be controlled independently of the size, or complexity of the original visual environments, and will not grow dramatically with higher spatial-temporal input resolution. Consequently, visual perception can be efficient even when employing large-scale neural networks to perceive intricate real-world scenes with high frame rates. Furthermore, since the fixations are strategically localized to focus on the important visual content and new fixations will be continuously introduced until the observation is sufficient, the model performance can be maximally preserved. In some scenarios, the performance may even be improved by eliminating task-irrelevant information interference.
Additionally, although we consider the most general form of square patches as fixations to ensure the generality of our framework, more advanced fixation formats may be adopted for optimization toward specific models of tasks (\emph{e.g.}, multi-scale mixed visual fixations).

\textbf{Perception Net} $f_{\textnormal{rep}}$
is a representation learning backbone network that converts raw pixelated image inputs into deep representations with semantic meanings. As aforementioned, high-capacity, large-scale models can be employed as $f_{\textnormal{rep}}$, to obtain strong visual processing capabilities. Since $f_{\textnormal{rep}}$ only needs to process the bandwidth-limited visual fixation, its inference still enjoys superior efficiency.

\textbf{Internal vision representation} $\bm{s}_1, \ldots, \bm{s}_t$
is maintained during the whole visual perception process, and dynamically updated utilizing the features extracted from each visual fixation by $f_{\textnormal{rep}}$, namely
\vskip -0.07in

\begin{equation}
    \bm{s}_t = \Psi(\bm{s}_{t-1}, f_{\textnormal{rep}}(\bm{l}_{t})),
\end{equation}

\vskip 0.03in \noindent
where $\Psi(\cdot, \cdot)$ denotes the updating operator (see Section \ref{sec:imple_detail_arch} for its implementation details).
The internal representation $\bm{s}_t$ summarizes the information from the history of all past observations, encoding the model's current knowledge of the environment. It serves two critical purposes. First, as shown in Fig. \ref{fig:fig1_2}b, $\bm{s}_t$ is the output of the AdaptiveNN framework, and the information within it will be utilized to fulfill the given vision task (feeding $\bm{s}_t$ into a task-specific head, detailed in Section \ref{sec:imple_detail_arch}). Second, $\bm{s}_t$ provides necessary information for decision-making in the sequential adaptive visual perception process, \emph{i.e.}, deciding whether to conclude observation now, and where to look next.
Both these two abilities are acquired through being trained to accomplish the vision task of interest (Fig. \ref{fig:fig1_2}b).

\textbf{Vision agent}
is a decision-making neural network that receives the internal vision representation $\bm{s}_1, \ldots, \bm{s}_t$ as inputs. At each step of the sequential perception process, it makes two decisions: assessing whether to terminate the ongoing observation and, if necessary, determining the subsequent visual fixation location. To achieve both of them simultaneously, we formulate the vision agent as the combination of a policy network $\bm{\pi}$ and a value network $V^{\bm{\pi}}$ (Fig. \ref{fig:fig1_2}d). 
This formulation is naturally derived from the theoretical learning principles of AdaptiveNN, which will be discussed in Section \ref{sec:formulation_AdaNN} coupled with the training algorithm of $\bm{\pi}$ and $V^{\bm{\pi}}$. Here we first introduce the inference process of $\bm{\pi}$ and $V^{\bm{\pi}}$.
At $t^{\textnormal{th}}$ step of inference, the outputs of $\bm{\pi}$ parameterize a distribution from which we can sample the location of $\bm{l}_{t+1}$, namely 
\vskip -0.07in

\begin{equation}
    \label{eq:sampling_fix}
    \bm{l}_{t+1} \sim p_{\bm{\pi}}(\bm{l}_{t+1}\lvert\bm{s}_{t}).
\end{equation}

\vskip 0.03in \noindent
Paired with $\bm{\pi}$, the value network $V^{\bm{\pi}}$ employs $\bm{s}_{t}$ to predict the expected gains of performing further observation on top of $\bm{s}_{t}$ (\emph{i.e.}, further updating $\bm{s}_{t}$) using $\bm{\pi}$, yielding a state value $V^{\bm{\pi}}(\bm{s}_t)$. We compare $V^{\bm{\pi}}(\bm{s}_t)$ with a threshold $\eta_t$. If $V^{\bm{\pi}}(\bm{s}_t)\leq\eta_t$, we are indicated that further observing is not valuable enough, and the sequential perception process will be concluded. Otherwise, $V^{\bm{\pi}}(\bm{s}_t)>\eta_t$ reveals that more fixations may yield significant improvements, and thus the new fixation $\bm{l}_{t+1}$ will be processed, evoking the ${(t+1)}^{\textnormal{th}}$ step. The value of $\eta_t$ is solved on the validation data, and can be adjusted online to vary the average resource demands of AdaptiveNN without additional training (see Section \ref{sec:imple_detail_inference}). Notably, the outputs of $\bm{\pi}$ and $V^{\bm{\pi}}$ consider both the current specific situations as the observation on each particular visual environment progresses, as well as the demands of the given vision task. The former has been encoded into $\bm{s}_{t}$, while the latter is attained through the training process (see Section \ref{sec:formulation_AdaNN}), where $\bm{\pi}$ and $V^{\bm{\pi}}$ can either learn to address a pre-defined static task, or learn to adapt to variable task demands on top of a prompt input (\emph{e.g.}, text), as depicted in Fig. \ref{fig:fig1_2}b. Moreover, it is noteworthy that $V^{\bm{\pi}}(\bm{s}_t)$ reflects the model's subjective assessments, namely whether the perception process of AdaptiveNN itself is worth proceeding, while $\eta_t$ determining if $V^{\bm{\pi}}(\bm{s}_t)$ is sufficiently small represents the objective constraints imposed by the external environment, \emph{e.g.}, the extent to which the overall available resources for visual perception are adequate in the current circumstance. This decoupled modeling of subjective and objective factors enables more flexible usage of our framework.
\vskip 0.1in

\noindent \textbf{Compatibility with various network architectures and vision tasks.}
Importantly, the major goal of developing AdaptiveNN is to facilitate a paradigm shift toward active and adaptive visual perception models. Therefore, its formulation has been designed to be general and flexible. For example, various off-the-shelf network architectures, such as Transformers and convolutional neural networks, can be conveniently deployed as the feature-extraction module of AdaptiveNN. Moreover, as shown in Fig. \ref{fig:fig1_2}b, the internal vision representation of AdaptiveNN does not adopt a strong assumption on its application scenario, and may be implemented under diverse task settings, \emph{e.g.}, employing AdaptiveNN as stand-alone perceptual models or as the basis of multimodal large language models, being applied to static images and videos, or interacting with dynamic environments, such as for robotics. In Section \ref{sec:results}, we provide comprehensive evaluation results to support these claims.

\subsection{Theoretical learning principles}
\label{sec:formulation_AdaNN}

Training AdaptiveNN incorporates both continuous (\emph{e.g.}, extracting feature from visual fixations) and discrete (\emph{e.g.}, learning to select fixation positions and adaptively conclude observation) optimization. This can not be straightforwardly solved by standard algorithms like gradient back-propagation. To address optimization challenges, we present a theorem that enables training AdaptiveNN in end-to-end (Fig. \ref{fig:fig1_2}c).
\vskip 0.1in

\noindent \textbf{Formulation.}
Given an AdaptiveNN model parameterized by $\bm{\theta}$ and a visual environment $\bm{X}$ to perceive, we refer to the distribution of the locations of visual fixation $\bm{l}_{1},\ldots,\bm{l}_{t}$ as $p(\bm{l}_{1:t}\lvert\bm{\theta}, \bm{X})$. On top of this, given a vision task, the model's outputs at $t^{\textnormal{th}}$ step for accomplishing the task (stemming from the internal vision representation $\bm{s}_t$) are denoted as $q(\bm{\theta}, \bm{X}, \bm{l}_{1:t})$, \emph{e.g.}, output logits for classification. Then, for a label $y$ associated with $\bm{X}$, which is defined upon the task, assume that we have a performance measure (typically a loss function) $\mathcal{L}(y, q(\bm{\theta}, \bm{X}, \bm{l}_{1:t}))$, such as the cross-entropy loss for classification and the mean squared error for regression. 
\vskip 0.1in

\noindent \textbf{Optimization objective}.
During training, AdaptiveNN focuses on learning a model capable of sequentially attending to proper visual fixations within a complex visual environment, and extracting information from these fixations to accomplish the vision task of interest. Its optimization objective is defined as minimizing the expected performance measure of the task, namely
\vskip -0.07in

\begin{equation}
    \label{eq:basic_train_obj}
    \textnormal{minimize} \quad
    \mathrm{L}(\bm{\theta}) = \mathbb{E}_{\bm{X}, y, t_{\textnormal{o}} \sim p(t_{\textnormal{o}})} 
        \int_{\bm{l}_{1:t_{\textnormal{o}}}} 
        p(\bm{l}_{1:t_{\textnormal{o}}}\lvert\bm{\theta}, \bm{X}) 
        \mathcal{L}(y, q(\bm{\theta}, \bm{X}, \bm{l}_{1:t_{\textnormal{o}}})).
\end{equation}

\vskip 0.03in \noindent
Here $t_{\textnormal{o}}\!\sim\!p(t_{\textnormal{o}}),\ t_{\textnormal{o}}\!\in\!\{1,\ldots,T\} $ indicates that during training, the total length $t_{\textnormal{o}}$ of the sequential perception process is sampled from a fixed prior distribution $p(t_{\textnormal{o}})$, which reflects the training process's statistical-level preference on the perception procedure's length. This consideration is introduced to add analytical flexibility to our model. Besides, note that we do not explicitly formulate the actions of actively concluding observation in Eq. (\ref{eq:basic_train_obj}). Conversely, we will demonstrate that the ability to evaluate when the observation is sufficient can be conveniently acquired on top of the model learned by minimizing Eq. (\ref{eq:basic_train_obj}).


\begin{theorem}
    \label{theorem:1}
    (see Section \ref{sec:theory} for proof)
    The gradients of $\mathrm{L}(\bm{\theta})$ can be decomposed into a combination of representation learning and self-rewarding reinforcement learning objectives:
    \begin{equation}
        \nabla_{\bm{\theta}}\mathrm{L}(\bm{\theta}) = 
            \nabla_{\bm{\theta}}\mathrm{L}_{\textnormal{rep}}(\bm{\theta}) + 
            \nabla_{\bm{\theta}}\mathrm{L}_{\textnormal{rl}}(\bm{\theta}),
    \end{equation}
    where
    \begin{equation}
        \label{eq:theorem_1}
        \begin{split}
            \nabla_{\bm{\theta}}\mathrm{L}_{\textnormal{rep}} 
            &= 
            \underbrace{
            \mathbb{E}_{\bm{X}, y, \bm{l}_{1:T}} 
            \sum\nolimits_{t=1}^{T}
            P(t_{\textnormal{o}}=t)
            \nabla_{\bm{\theta}}  \mathcal{L}(y, q(\bm{\theta}, \bm{X}, \bm{l}_{1:t}))
            }_{\textnormal{representation learning}},
                \\
            \nabla_{\bm{\theta}}\mathrm{L}_{\textnormal{rl}}
            &=
            \underbrace{
                -\mathbb{E}_{\bm{X}, y, \bm{l}_{1:T}} 
                \sum\nolimits_{t=1}^{T}
                \left[
                    \left(
                        \sum\nolimits_{t'=t}^{T} r_{t'}
                    \right)
                    \nabla_{\bm{\theta}}  \log p(\bm{l}_{t} \lvert \bm{\theta}, \bm{X}, \bm{l}_{1:(t-1)}) 
                \right]
                }_{\textnormal{self-rewarding reinforcement learning}},
                \\
            & \ \ \ \ \ \ \ \ r_{t'} = -P(t_{\textnormal{o}}=t') \mathcal{L}(y, q(\bm{\theta}, \bm{X}, \bm{l}_{1:t'})).
        \end{split}
    \end{equation}
\end{theorem}

In Eq. (\ref{eq:theorem_1}), $\nabla_{\bm{\theta}}{\mathrm{L}_{\textnormal{rep}}}$ is a standard form of representation learning, namely minimizing the task loss over the features extracted from $\bm{l}_{1},\ldots,\bm{l}_{t}$ by the model. Additionally, $\nabla_{\bm{\theta}}{\mathrm{L}_{\textnormal{rl}}}$ boiling down to a form of policy gradients in reinforcement learning \cite{Schulmanetal_ICLR2016}, where $p(\bm{l}_{t} \lvert \bm{\theta}, \bm{X}, \bm{l}_{1:(t-1)})$ is the action distribution, $r_{t'}$ is the reward received at each time step, and $\sum\nolimits_{t'=t}^{T} r_{t'}$ is the cumulative reward following the execution of an action $\bm{l}_{t}$. Since $r_{t'}$ is defined using the negative values of task loss of the model itself, we name $\mathrm{L}_{\textnormal{rl}}$ as the self-rewarding reinforcement learning objective. 

In conclusion, Theorem \ref{theorem:1} reveals that when considering minimizing the expected loss of AdaptiveNN over a vision task, an integration of representation learning and self-rewarding reinforcement learning objectives naturally emerges. The former trains the model to extract deep representations from input visual fixations, while the latter guides the model to strategically select fixation locations within the complex visual environment to minimize the loss. Notably, both of them only leverage the standard task loss, without relying on specialized task formats or additional annotations.
\vskip 0.1in

\noindent \textbf{Specific learning algorithm}.
Given Theorem \ref{theorem:1}, $\nabla_{\bm{\theta}}{\mathrm{L}_{\textnormal{rep}}}$ can be directly utilized as the gradient signals for learning feature-extraction modules. For the policy gradients $\nabla_{\bm{\theta}}{\mathrm{L}_{\textnormal{rl}}}$, as reinforcement learning problems are usually more challenging to solve, we propose an augmented version of its basic formulation. First, we introduce a pre-defined discount factor $\gamma \in [0,1]$ \cite{mnih2015human, Schulmanetal_ICLR2016, schulman2017proximal} and a differential form of rewards, aiming to achieve a flexible modeling of balancing long-term and short-term returns, as well as to stabilize the training process. 
Thus, on top of Eq. (\ref{eq:sampling_fix}) and Eq. (\ref{eq:theorem_1}), the policy gradient rule for updating the model can be expressed as
\vskip -0.07in

\begin{equation}
    \label{eq:rl_target_1}
    \begin{split}
        \nabla_{\bm{\theta}}\mathrm{L}_{\textnormal{rl}}
        &=
            -\mathbb{E}_{\bm{X}, y, \bm{l}_{1:T}} 
            \sum\nolimits_{t=1}^{T}
            \left[
                \left(
                    \sum\nolimits_{t'=t}^{T}  \gamma^{t'-t} \left(r_{t'} - r_{t'-1}\right)
                \right)
                \nabla_{\bm{\theta}}  \log p_{\bm{\pi}}(\bm{l}_{t}\lvert\bm{s}_{t-1})
            \right],
            \\
        & \ \ \ \ \ \ \ \ \ \ \ \ \ \ \ \ \ \ r_{t'} = -P(t_{\textnormal{o}}=t') \mathcal{L}(y, q(\bm{\theta}, \bm{X}, \bm{l}_{1:t'})),
    \end{split}
\end{equation}

\vskip 0.03in \noindent
where we have (see Section \ref{sec:theory} for the proof)
\vskip -0.07in

\begin{equation}
    \label{eq:rl_target_limit}
    \begin{split}
        \lim_{\gamma \to 0} \nabla_{\bm{\theta}}\mathrm{L}_{\textnormal{rl}}
        &=
            -\mathbb{E}_{\bm{X}, y, \bm{l}_{1:T}} 
            \sum\nolimits_{t=1}^{T}
            r_{t} \nabla_{\bm{\theta}}  \log p_{\bm{\pi}}(\bm{l}_{t}\lvert\bm{s}_{t-1}),
        \\
        \lim_{\gamma \to 1} \nabla_{\bm{\theta}}\mathrm{L}_{\textnormal{rl}}
        &=
            -\mathbb{E}_{\bm{X}, y, \bm{l}_{1:T}} 
            \sum\nolimits_{t=1}^{T}
            r_{T} \nabla_{\bm{\theta}}  \log p_{\bm{\pi}}(\bm{l}_{t}\lvert\bm{s}_{t-1}).
        \\
    \end{split}
\end{equation}

\vskip 0.03in \noindent
When $\gamma \to 0$, the strategy for selecting the next visual fixation tends to be fully short-sighted and is only optimized to maximize the immediate reward $r_{t}$. Conversely, $0<\gamma<1$ tends to encourage perception strategies that maximally attain the goal within a limited number of fixations. When $\gamma=1$, AdaptiveNN only focuses on maximizing the final reward $r_{T}$, corresponding to the scenarios where abundant resources or energy are available, while the perception process can leverage as many visual fixations as possible to accomplish the task. 

Moreover, we introduce a value network $V^{\bm{\pi}}$ to offer a baseline function for reinforcement learning \cite{sutton1999policy, schulman2017proximal}, which can effectively stabilize training by reducing gradient estimation variance \cite{mnih2015human, silver2016mastering}. The learning objective of $V^{\bm{\pi}}$ is to predict the expected gains of further observing at each step:
\vskip -0.07in

\begin{equation}
    \label{eq:rl_target_2}
    \mathop{\textnormal{minimize}}_{V^{\bm{\pi}}} \quad \mathbb{E}\left[
        V^{\bm{\pi}}(\bm{s}_{t-1}) - 
        \sum\nolimits_{t'=t}^{T}  \gamma^{t'-t} \left(r_{t'} - r_{t'-1}\right)
        \right]^2.
\end{equation}

\vskip 0.03in \noindent
Besides, with this goal, $V^{\bm{\pi}}(\bm{s}_{t-1})$ provides a reasonable proxy measure for adaptive termination, as stated in Section \ref{sec:arch_AdaNN}. For example, a relatively small $V^{\bm{\pi}}(\bm{s}_{t-1})$ indicates that even if the model processes more visual fixations, the loss $\mathcal{L}$ measuring the performance of the given task will not show notable further reduction. Hence, it is reasonable to consider concluding observation at that time.

\clearpage
\begin{figure*}[!t]
    \centering
    \includegraphics[width=\textwidth]{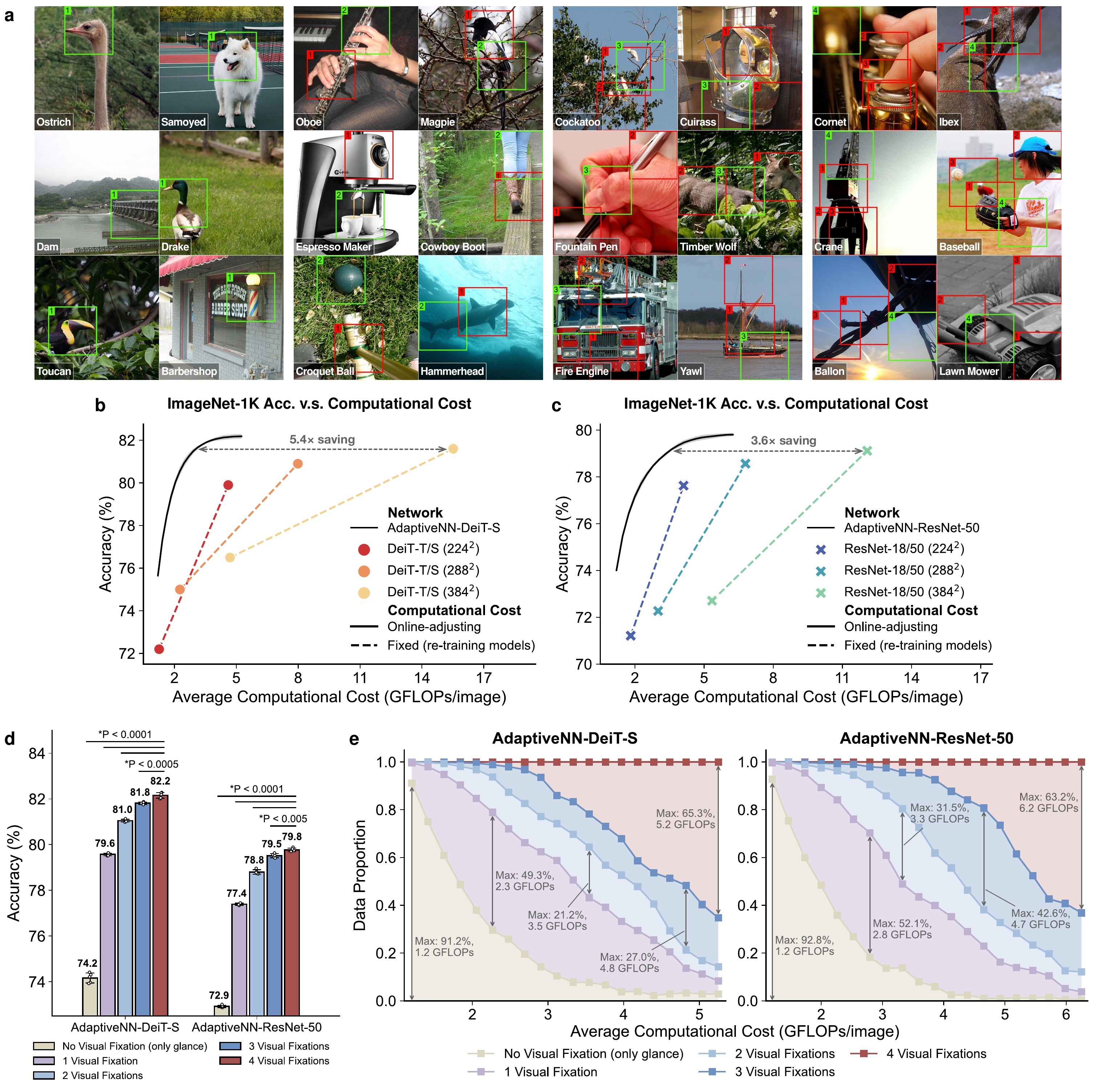}
    \caption{
        \textbf{Results of ImageNet large-scale real-world visual understanding.}
        \textbf{(a)} Qualitative assessment showcasing the visual fixations localized by AdaptiveNN(-DeiT-S), with boxes marking the locations of fixations and colors indicating the model's decision to conclude (green) or continue (red) observation at each step. Step indices are presented at the top left of the boxes. Ground truth labels are displayed at the bottom left of the images.
        \textbf{(b-c)} Quantitative comparisons of AdaptiveNN and traditional non-adaptive models on top of identical backbones: Top-1 validation accuracy versus average computational cost for inferring the model. To obtain non-adaptive models with varying costs, we consider two common approaches: adjusting model sizes and input resolutions.
        \textbf{(d)} Relationship between validation accuracy and the number of visual fixations, assuming that all samples utilize the same number of visual fixations. 
        \textbf{(e)} Proportions of data that utilize different numbers of visual fixations, set against different budget constraints for computational costs.
        *One-way analysis of variance (ANOVA) with Tukey's honestly significant difference (HSD) test. Error bars show the standard deviations of five independent trials with different random seeds.
        \label{fig:fig2}
        }
    \vskip 0.2in
\end{figure*}

\section{Results}
\label{sec:results}

We comprehensively evaluate AdaptiveNN on 17 benchmarks organized into 9 different tasks. These benchmarks include large-scale visual understanding, fine-grained visual recognition, visual search, processing images from real driving and medical scenarios, multimodal large language models for language-driven embodied robot execution, and side-by-side comparisons of humans' visual perception behaviors with our model's capabilities. 
See Section \ref{sec:intro_tasks} for more details of these tasks (including data collection, annotations, metrics, and other setups). 
Employing a wide spectrum of evaluation tasks enables us to arrive at a more complete picture of the characteristics, efficacy, and potential values of AdaptiveNN.

\subsection{Large-scale real-world visual understanding}

Our first evaluation of AdaptiveNN considers the foundational visual understanding task, namely mapping pixelated visual signals (\emph{e.g.}, objects and scenes) to abstract concepts.
We employ the ImageNet image recognition benchmark \cite{russakovsky2015imagenet}, which is widely acknowledged for its critical role in assessing the efficacy of machine learning methods \cite{he2016deep, huang2017densely, dosovitskiy2021an}. Comprising over 1.28 million images classified into 1,000 categories according to the WordNet hierarchy \cite{miller1995wordnet}, ImageNet encompasses a diverse array of visual content, including various objects, buildings, humans, animals, scenes, etc., offering a robust platform for evaluating visual understanding capabilities. Our model is trained to correctly classify the input image. Notably, the feature-extraction networks within AdaptiveNN are compatible with most existing backbones (see Section \ref{sec:imple_detail_arch} for details). To demonstrate the generalizability of AdaptiveNN, we deploy two examples, ResNet (convolutional network) \cite{he2016deep} and DeiT (vision Transformer) \cite{touvron2021training}, each representing a wide range of popular architectures. 

Fig. \ref{fig:fig2}a illustrates the learned visual perception behaviors of AdaptiveNN when applied to the ImageNet visual recognition task. Observations reveal that both the locations of visual fixations and the length of observation processes for different samples are reasonable and interpretable. Our model acquires the capability of fixating on the class-discriminative regions, such as the heads of animals, the principal structures of musical instruments, and the functional parts like knobs and nozzles on coffee machines. Moreover, in scenarios involving complex or atypical visual inputs, AdaptiveNN adjusts by extending the duration of observation to enhance the accuracy of its predictions. This adaptive behavior is particularly evident when the objects of interest are small, located at a significant distance from the camera, or depicted from uncommon perspectives, showcasing only parts of their entirety.

Quantitatively, introducing human-like adaptive visual perception to computer vision models substantially enhances both their energy efficiency and adaptability. This can be illustrated by Fig. \ref{fig:fig2}b-\ref{fig:fig2}c, and Supplementary Data Tab. \textcolor{blue}{2-5}, where the performance of our model, equipped with AdaptiveNN, is compared against the traditional, non-adaptive counterparts on top of the same backbones. The sole distinction lies in the implementation of AdaptiveNN. DeiT-S and ResNet-50 maximally achieve validation accuracies of 81.6\% and 79.1\% at the computational costs of 15.5 and 12.1 GFLOPs per image. In contrast, AdaptiveNN-DeiT-S and AdaptiveNN-ResNet-50 performs on par with them at 2.86 and 3.37 GFLOPs per image, which are 5.4$\times$ and 3.6$\times$ more efficient, respectively. Moreover, the computational cost of AdaptiveNN can be flexibly adjusted online, resulting in a favorable balance between efficiency and effectiveness across broad ranges. This adaptability is in contrast to non-adaptive models, which typically require retraining to achieve similar performance adjustments.

Fig. \ref{fig:fig2}d and Supplementary Data Tab. \textcolor{blue}{6-7} report the validation accuracies with all samples processed using the same number of visual fixations. Progressively leveraging more fixations improves accuracy significantly (all P$<$0.005), yet the effects gradually diminish, indicating increased difficulty in further boosting accuracy upon a decent performance. Fig. \ref{fig:fig2}e and Supplementary Data Tab. \textcolor{blue}{8-9} illustrate how AdaptiveNN adapts to the dynamically varied quantities of available resources. In scenarios where computational resources are abundant, the model can afford to allocate numerous visual fixations to most samples, thereby optimizing the overall accuracy. In contrast, when computational resources are constrained, prioritization is given to the more challenging samples, while other samples are allocated fewer resources to compensate for the limited budget. This strategic allocation underscores AdaptiveNN's adaptability in managing resource distribution to maximize performance efficiency.

\subsection{Fine-grained visual recognition}

Beyond the general-purpose large-scale visual understanding task, we further probe into AdaptiveNN's nuanced visual discriminative capabilities using six fine-grained recognition tasks. These tasks are characterized by the small differences between classes and significant variations within each class, such as differentiating between visually very close species of birds or pets against highly diversified backgrounds. Accomplishing them necessitates AdaptiveNN to localize and identify minor, task-dependent signals out of an extensive or even overwhelming multitude of irrelevant visual information. This ability to filter and focus on pertinent details mirrors a key strength of human visual systems \cite{wolfe2004attributes, najemnik2005optimal, ma2011behavior, mnih2014recurrent, henderson2017meaning, gottlieb2018towards, hanning2023dissociable}, revealing the model's potential to approach the nuanced perceptual capabilities observed in human cognition.

Extended Data Fig. \ref{fig:fig3}a and Supplementary Data Tab. \textcolor{blue}{10-15} summarize the quantitative evaluation results. Similar to Fig. \ref{fig:fig2}b and \ref{fig:fig2}c, our model's performance in terms of energy efficiency and adaptability is benchmarked against that of conventional, non-adaptive models. Both AdaptiveNN and the baselines are fine-tuned from the checkpoints pre-trained on ImageNet. AdaptiveNN dramatically saves the computational cost of the model without sacrificing accuracy (multiple of reduction: 6.2$\times$, 6.1$\times$, 7.6$\times$, 8.2$\times$, 5.8$\times$, 6.3$\times$). This efficiency gain surpasses those observed on ImageNet, underscoring our model's human-like proficiency in fixating on and leveraging nuanced discriminative features.

Furthermore, the behaviors of our model demonstrate good interpretability. As depicted in Extended Data Fig. \ref{fig:fig3}b-\ref{fig:fig3}e, AdaptiveNN autonomously learns to localize the details valuable for fine-grained recognition, such as the beaks of birds, the car lights, the airplane engines, the propellers, etc. This proficiency is notably achieved through training that is guided solely by image-level category labels, without explicit instructions on the spatial details to focus on. In some difficult scenarios, where the primary discriminative features may be concealed or indistinct, our model can actively determine to observe with more fixations, seeking additional secondary features as alternative cues for accurate classification.

\subsection{Efficient processing of visual data from real driving scenarios}

The ImageNet and fine-grained recognition datasets are standard visual understanding benchmarks collected from the Internet. Consequently, in general, many images within them have been centered toward the relevant objects or content by human photographers and users. Nevertheless, AdaptiveNN does not rely on this object-centric precondition. Similar to human visual systems, AdaptiveNN is applicable to more general and complex scenarios, for example, efficiently processing non-object-centric images collected in the wild without specified pre-processing. To demonstrate this, we evaluate our model with the traffic sign recognition task on the Swedish traffic signs dataset (STSD) \cite{larsson2011using}. The dataset consists of high-resolution road-scene images collected on real moving vehicles. The traffic signs of interest are usually very small, distributed diversely, and not clear in many cases, presenting a realistic challenge.

For this more generalized task of visual perception in natural environments, AdaptiveNN markedly outperforms traditional non-adaptive models, achieving efficiency gains greater than an order of magnitude. Illustrated in Fig. \ref{fig:fig4}a and Supplementary Data Tab. \textcolor{blue}{16-17}, the strongest baseline, ResNet-50 with 960$^2$ inputs, acquires an accuracy of 90.2\% with $\sim$76 GFLOPs/image, while our model performs on par with it using only $\sim$2.7 GFLOPs/image, yielding an 27.9$\times$ reduction of inference cost. This substantial enhancement can be elucidated through the qualitative analysis in Fig. \ref{fig:fig4}b. The visual fixations localized by our model adaptively center its `retina' on the small, task-relevant regions within the expansive, intricate, and cluttered visual scenes, mirroring the efficiency characteristic of human visual perception. Conversely, conventional non-adaptive models typically process all pixels equivalently, which is inefficient and vulnerable to overfitting. Moreover, when AdaptiveNN initially misidentifies the location of traffic signs, it tends to recognize its error, infer the signs' true locations based on the current information, and attempt to rectify this mistake in subsequent fixations.

\begin{figure*}[!t]
    \centering
    \vskip -0.45in
        \includegraphics[width=\textwidth]{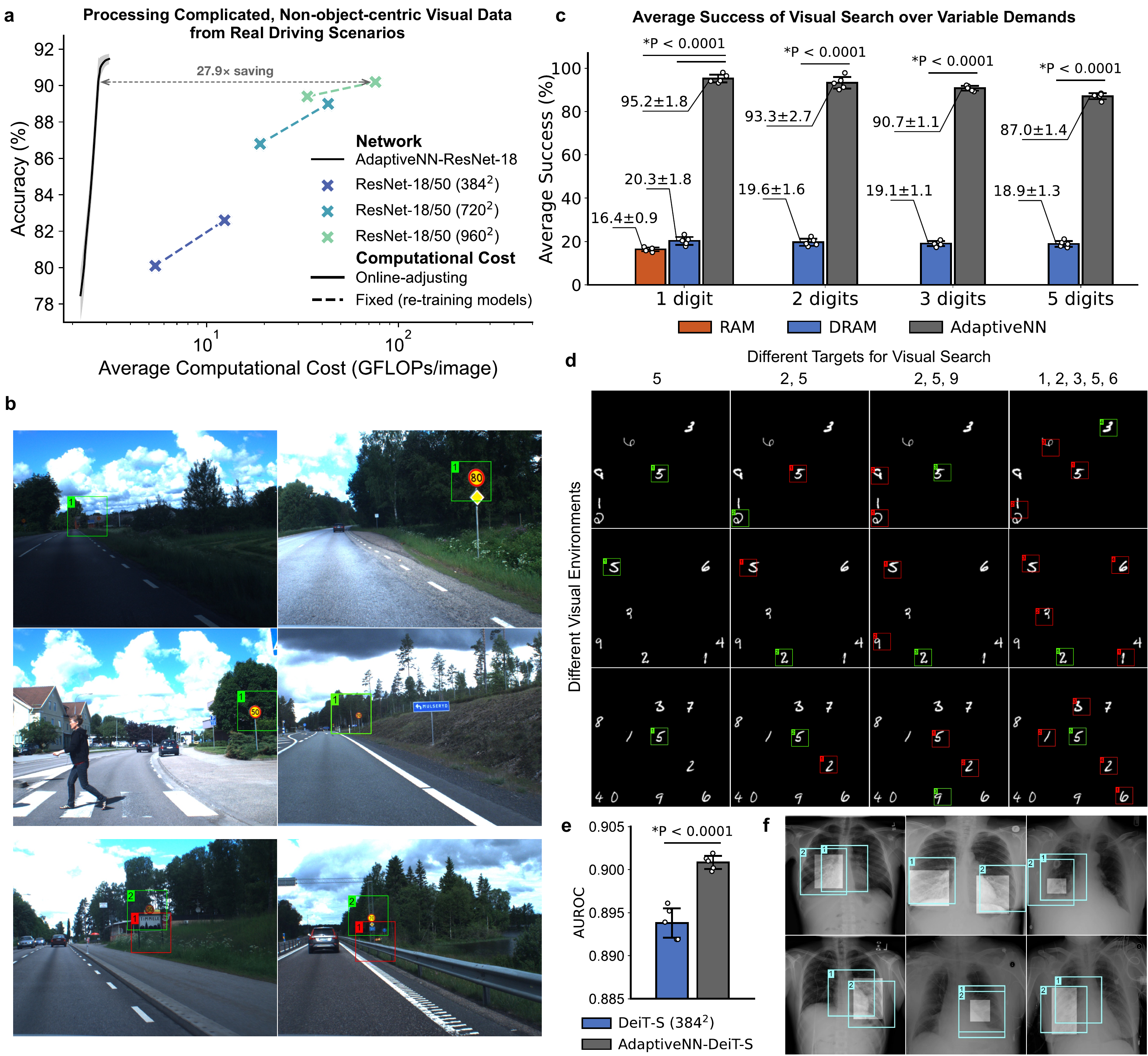}
    \caption{
        \textbf{Assessment of AdaptiveNN in more general visual perception scenarios: processing images from real driving and medical scenarios, and visual search tasks with variable demands.}
        \textbf{(a)} Comparisons of AdaptiveNN and conventional non-adaptive models in processing complicated, non-object-centric real-world scenes: Top-1 validation accuracy versus average computational cost for inferring the model (log-scale). We consider the traffic sign recognition task on the Swedish traffic signs dataset (STSD) \cite{larsson2011using}, composed of 960$\times$1,280 road-scene images collected on real moving vehicles. ResNets are deployed as backbones since convolutional networks tend to be more efficient for processing high-resolution inputs.
        \textbf{(c)} Average success rates of visual search tasks. Here `$n$ digits' indicates the number of target digits, while the bars indicate the mean $\pm$ standard deviations of five randomly generated visual search tasks with various target categories (yet maintaining a constant number of targets). Success is defined as accurately retrieving exactly all the digits specified by a given task.
        \textbf{(e)} Area under the receiver operating characteristic curve (AUROC) of the RSNA pneumonia detection task \cite{shih2019augmenting}. All models are trained to predict the presence or absence of pneumonia based on image-level labels. Here we do not perform adaptive termination in AdaptiveNN and mainly focus on the AUROC after processing all fixations, since efficiency may not be a major focus of medical diagnosis tasks.
        \textbf{(b-f)} Qualitative evaluation results corresponding to \textbf{(a-c)}. Boxes represent visual fixation locations, with colors indicating the model's decision to either continue (red) or terminate (green) observation at that step. Step indices are annotated at the upper left corner of each box. Particularly, in \textbf{f}, lighter boxes show the pneumonia regions annotated by human clinicians (this localization information is not utilized for training).
        *Independent samples $t$-test. Except for \textbf{(c)}, all error bars show the standard deviations of five independent trials with different random seeds.
        \label{fig:fig4}
        }
    \vskip 0.2in
\end{figure*}

\subsection{Addressing vision tasks with flexible requirements}

Even when confronting the same visual environment, humans can flexibly adjust their visual perception behaviors, such as the locations and numbers of fixations, in response to the specific requirements of the task at hand \cite{najemnik2005optimal, ma2011behavior, wolfe2017five}. To investigate whether AdaptiveNN can acquire such human-like adaptability, we considered a visual search scenario: we generate 224$^2$ images, each randomly populated with 6 to 10 digits against a black background without repetition of digits. A model is trained to identify the locations of certain given digits within each input, where the categories and number of targets are assumed to be flexibly changed. Each specified setup of targets is defined as an individual visual search task.

Fig. \ref{fig:fig4}c and Supplementary Data Tab. \textcolor{blue}{18} summarizes the quantitative evaluation results. We estimate the success rate of retrieving exactly all the targets demanded by a visual search task in various visual environments, and report its expected value over different tasks. AdaptiveNN maintains an average success accuracy of approximately 90\% consistently with a varying number of searching targets. On the contrary, existing popular models that aim to mimic human sequential visual perception, like RAM \cite{mnih2014recurrent} and DRAM \cite{ba2014multiple}, generally do not exceed a success rate of around 20\%, markedly underperforming AdaptiveNN by $>$4.5$\times$ in most instances. Moreover, Fig. \ref{fig:fig4}d illustrates AdaptiveNN's capability to adaptively modulate its fixation selection and observation termination strategies conditioned on each input and specific visual task. It does not fixate on more regions after all targets have been localized, and intriguingly, it learns to efficiently identify two adjacent targets using a single fixation. In general, these observations show that AdaptiveNN can acquire a robust human-like adaptability in task-specific visual perception behaviors, considerably outperforming previous works.

\subsection{Interpretability-critical tasks: image processing in medical scenarios}

A predominant merit of AdaptiveNN is its capacity for enhanced interpretability through examining its visual fixation patterns (as shown in Fig. \ref{fig:fig2}a, \ref{fig:fig4}b, \ref{fig:fig4}d, and Extended Data Fig. \ref{fig:fig3}b-\ref{fig:fig3}e). Built upon this insight, we further evaluate our model's utility in vision tasks where interpretability is of vital importance. Specifically, we take the medical diagnosis scenario of detecting pneumonia from chest X-ray images as a representative example \cite{shih2019augmenting}. AdaptiveNN is trained only using the image-level labels indicating the presence or absence of pneumonia. As demonstrated in Fig. \ref{fig:fig4}e, it exhibits a significantly superior AUROC (area under the receiver operating characteristic curve) on validation data than the conventional non-adaptive model (P$<$0.0001). Furthermore, despite the absence of explicit localization guidance during training, the visual fixations identified by AdaptiveNN (see Fig. \ref{fig:fig4}f) align closely with the pulmonary opacity regions annotated by human clinicians (18 board-certified radiologists from 16 institutions). This concordance reveals the potential value of AdaptiveNN in developing AI applications that demand not only precision but also good interpretability, such as medical applications.

\begin{figure*}[!t]
    \centering
    \includegraphics[width=\textwidth]{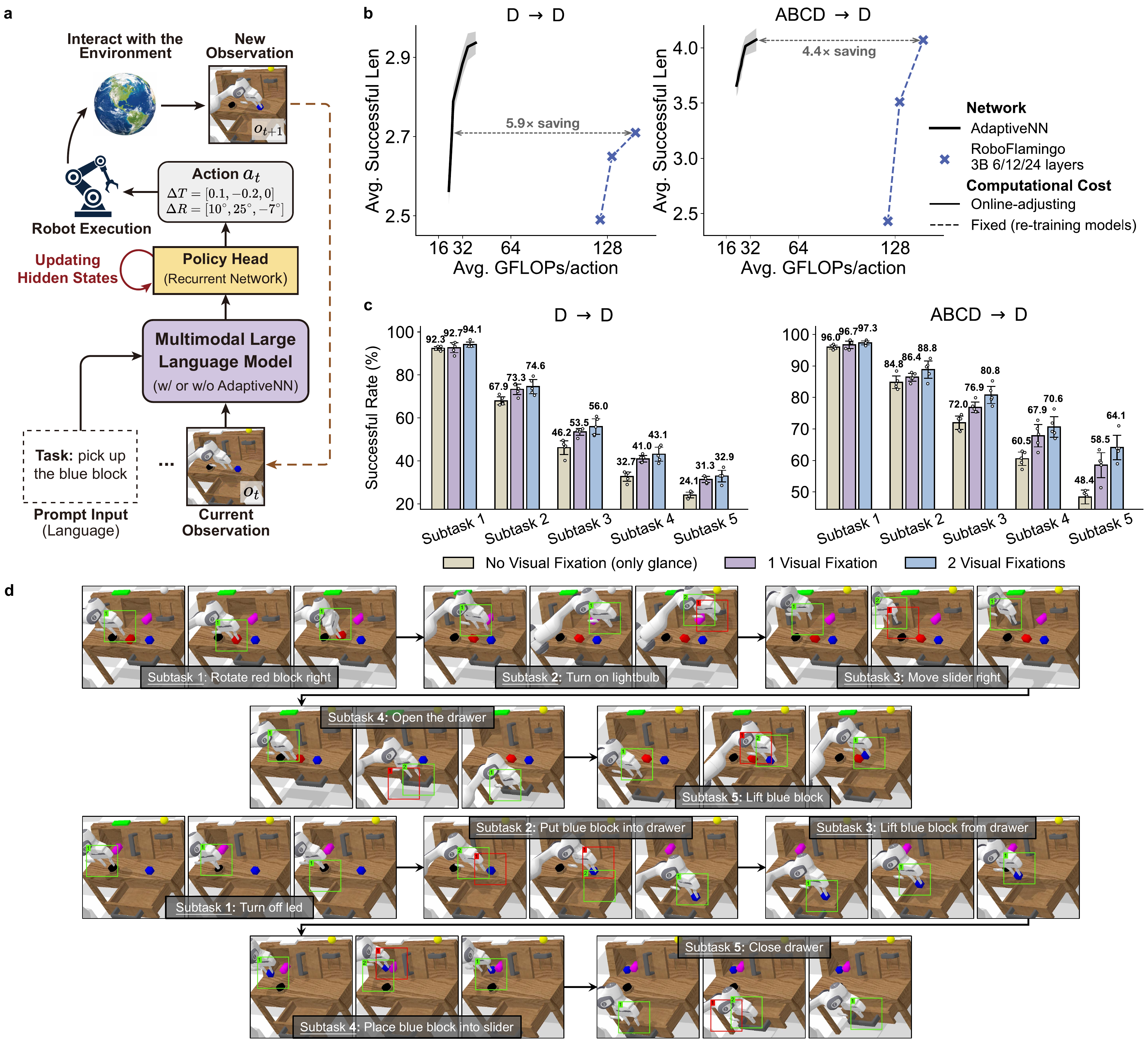}
    \caption{
        \textbf{Performance of the embodied multimodal large language models (MLLM) based on AdaptiveNN.}
        \textbf{(a)} A schematic overview of the embodied MLLM. The prompt input specifics the task, and the MLLM iteratively perceives the environment to execute appropriate robotic actions, \emph{i.e.}, six-degree-of-freedom (6-DoF) transformation vectors in 3D space. The next observation reflects the outcome of the preceding actions. A recurrent policy head integrates information from all previous observations.
        \textbf{(b)} Comparisons of AdaptiveNN-based MLLM and non-adaptive MLLM using identical backbones on CALVIN: average successful length (of 1000 5-task sequences) versus average computational cost for inferring the model. For the non-adaptive models, computational costs are modulated by adjusting model sizes. D$\to$D and ABCD$\to$D indicate different scales of training data.
        \textbf{(c)} Relationship between average successful rates of each subtask within task sequences and the number of visual fixations, assuming that all samples utilize the same number of visual fixations. 
        \textbf{(d)} Qualitative assessment of two representative 5-task sequences. Boxes marks the fixation locations and colors indicates the model's decision to conclude (green) or continue (red) observation at each step. Step indices are presented at the top left of the boxes. The prompt inputs for specifying tasks are displayed within the black boxes.
        All error bars show the standard deviations of five independent trials with different random seeds.
        \label{fig:embodiedAI_results}
        }
    \vskip 0.4in
\end{figure*}

\subsection{Embodied multimodal large language models based on AdaptiveNN}

The formulation of AdaptiveNN is general enough to be deployed as the perceptual module of an embodied agent that interacts with dynamic physical environments. To demonstrate its versatility, we deploy AdaptiveNN on top of a multimodal large language model (MLLM). As illustrated in Fig. \ref{fig:embodiedAI_results}a, the MLLM receives language as prompt inputs to specify the task of interest, observe the current environment (with or without AdaptiveNN), and update the recurrent policy network to execute an appropriate action, subsequently affecting the environments, evoking the next observation-to-action process. The MLLM's network architecture is based on RoboFlamingo \cite{li2024visionlanguage}, as detailed in Extended Data Fig. \ref{fig:embodiedAI_details}a and Section \ref{sec:imple_detail_arch_embodiedAI}. We assess the performance of AdaptiveNN-based MLLM using the CALVIN Long-Horizon Multi-Task Language Control benchmark (LH-MTLC) \cite{mees2022calvin}, where an agent aims to successfully complete task sequences, each comprising five subtasks described in natural language. The model performance is quantified as the average successful length (0 to 5) across 1000 task sequences, as detailed in Extended Data Fig. \ref{fig:embodiedAI_details}b. We consider two benchmark settings using identical validation tasks and different training data scales (\emph{i.e.}, D$\to$D, ABCD$\to$D).

Fig. \ref{fig:embodiedAI_results}b and Supplementary Data Tab. \textcolor{blue}{19-20} demonstrate that AdaptiveNN saves the average computational cost by $4.4-5.9\times$ without sacrificing effectiveness compared to the non-adaptive baselines. Besides, AdaptiveNN is notably more flexible in adjusting its computational cost online without necessitating retraining.
Supplementary Data Tab. \textcolor{blue}{21-22} further compare AdaptiveNN and the baselines in terms of the success rates of different types of tasks.
Fig. \ref{fig:embodiedAI_results}c and Supplementary Data Tab. \textcolor{blue}{23-24} report the performance corresponding to each fixed number of visual fixations, depicting a progressively increasing trend of average successful length, which is more pronounced for large-scale, diverse training data such as ABCD$\to$D.
Representative qualitative results are presented in Fig. \ref{fig:embodiedAI_results}d. AdaptiveNN succeeds in learning to fixate on the task-relevant objects specified by the input language prompts, as well as to capture their interactions with the robotic operational components. In other words, AdaptiveNN dynamically determines `where to look' based on both the visual environments themselves and the variable task prompts.
Moreover, when the required action involves precise and fine-grained control, our model tends to leverage more visual fixations for more careful perception. Otherwise, relatively fewer fixations will be adopted to maximally save the cost.
In summary, AdaptiveNN exhibits significantly improved computational efficiency, adaptability, and interpretability. These marked merits are consistent with our previous findings.

\begin{figure*}[!ht]
    \centering
    \vskip -0.6in
    \includegraphics[width=\textwidth]{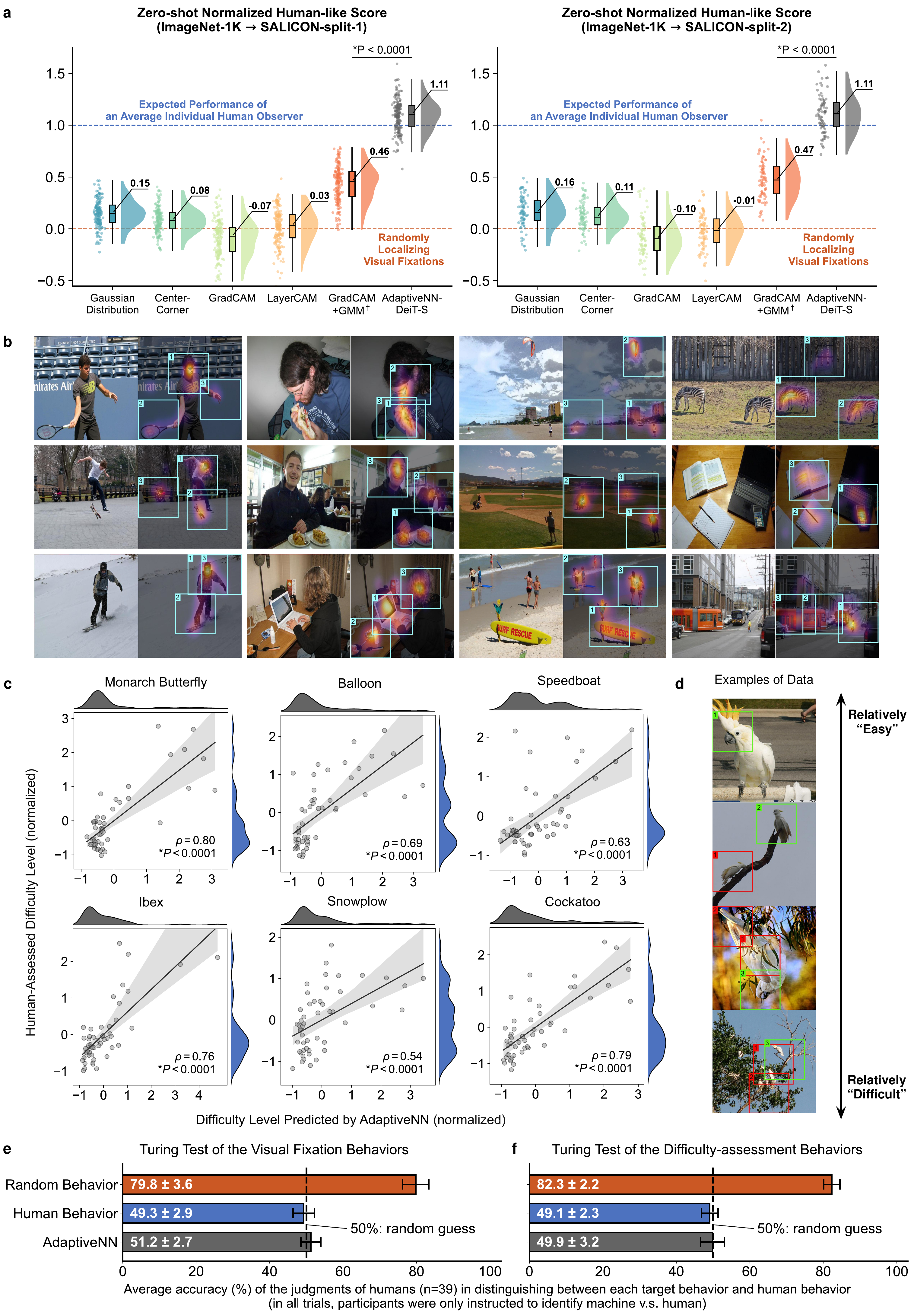}
    \vskip -0.5in
\end{figure*}

\begin{figure*}[!ht]
    \caption{
        \textbf{Behavioral comparisons between AdaptiveNN and human vision.} 
        \label{fig:fig5}
    \textbf{(a)} Normalized human-like scores, which quantify the probability that the ground truth gazing centers of human vision (whose distribution is estimated by averaging across $\sim$60 observers' visual perception behaviors) fall into the visual fixation regions localized by AdaptiveNN or comparative strategies. 
    The raw results are normalized with respect to the expected performance of selecting fixation regions with the gazing locations of an individual human observer (1.0 on the $y$-axis) and uniformly at random (0.0 on the $y$-axis). 
    Each point represents the result over a mini-batch of data, while boxplots depict the distribution of results.
    The evaluation is based on the SALICON dataset \cite{jiang2015salicon}. Our model is trained on ImageNet, having never seen the data in SALICON. This `zero-shot' paradigm evaluates directly transferring AdaptiveNN's perceptual behaviors to novel, complex environments, with a fixed number of fixations, mirroring the collection procedure of human gazing centers.
    Baselines for comparison incorporate selecting fixation regions using i) pre-defined rules; ii) class activation maps (CAMs); iii) CAMs augmented with a Gaussian mixture model (GMM). See Supplementary Section \textcolor{blue}{B} for the details of these baselines. $^\dagger$: GMM introduces additional computation.
    *Independent samples $t$-test.
    \textbf{(b)} Qualitative comparisons between the ground truth density maps of human gazing centers (heat maps) and AdaptiveNN fixation regions (boxes). Boxes indicate fixation locations, with step indices annotated at the upper left corner of each box.
    \textbf{(c)} Correlation of human-assessed difficulty scores (averaged across $n$=10 subjects) and difficulty levels (state values) evaluated by the \emph{Vision Agent} of AdaptiveNN. Without loss of generality, the state values are taken from the first step of sequential perception processes. Results are based on six representative categories of data in the ImageNet validation set.
    $\rho$: Pearson correlation coefficients. *Correlation $t$-test.
    \textbf{(d)} Visualization examples of the `easy' and `difficult' data identified by AdaptiveNN.
    \textbf{(e-f)} Results of `visual Turing tests'. Human judges ($n$=39) are randomly given paired examples of visual perception behaviors from `humans' and `one within \{AdaptiveNN, humans, random behaviors\}'. They are instructed to identify the machine (even in cases when the pairs of `random v.s. human' or `human v.s. human' are given, which serve as control groups for comparison). 
    Bars show the mean accuracy across human judges and the corresponding 95\% confidence interval. Ideal performance is 50\%, where the machine is indistinguishable from human behaviors in these binary choice tasks.
    }
    \vskip 0.2in
\end{figure*}

\subsection{Comparisons between AdaptiveNN and human visual perception}

AdaptiveNN also emerges as a potent computational tool for probing human visual cognition under controlled experimental conditions. This can be uncovered with the marked consistency between humans and AdaptiveNN in side-by-side evaluations on the same tests of visual perception behaviors. Our experimental protocols are detailed in Section \ref{sec:intro_tasks_vs_human}.

First, we examine the consistency of the locations of visual regions that humans and AdaptiveNN fixate on. We assess the spatial-wise adaptiveness of human vision using the SALICON benchmark dataset \cite{jiang2015salicon}, which comprises images each observed by approximately 60 participants recruited via Amazon Mechanical Turk (AMT). These participants were asked to freely view each image for 5 seconds without specific instructions on where to direct their gaze, allowing for an unbiased recording of their visual gazing center points. In Fig. \ref{fig:fig5}a, we utilize the aggregate density map of all $\sim$60 observers' gazing points as the ground truth, against which we evaluate the probability that the real focal centers of human vision fall into the visual fixation regions localized by AdaptiveNN. Here AdaptiveNN adopts the same experimental paradigm as humans: it produces a fixed number of fixations, and it is learned on ImageNet, having never seen or been trained using the data in SALICON. In addition to our model, we also report the performance corresponding to the expected behavior of selecting visual fixation regions following the gazing locations of an arbitrary single human (one of $\sim$60 observers) or uniformly at random, denoted by 1.0 (human) and 0.0 (random) on the $y$-axis of Fig. \ref{fig:fig5}a, yielding a metric named `normalized human-like score' (see Section \ref{sec:intro_tasks_vs_human} for details).

Fig. \ref{fig:fig5}a and Supplementary Data Tab. \textcolor{blue}{25-26} summarize the quantitative results on the two splits of SALICON. In terms of the alignment with the ground truth spatial-adaptive human visual perception behaviors, AdaptiveNN performs on par with or surpasses that of an average individual human observer, consistently registering normalized human-like scores exceeding 1.0. In contrast, pre-defined fixation localization policies and CAM-based methods yield scores ranging from -0.1 to 0.4, generally failing to significantly outperform a random selection strategy. Furthermore, Fig. \ref{fig:fig5}b offers a qualitative comparison between the ground truth density maps of human gazing locations (heat maps) and the fixation regions selected by AdaptiveNN (boxes). Our model produces human-like patterns in many cases, frequently being attracted by faces, hands, human bodies, human actions, or objects intimately associated with human activity, such as food, computers, skateboards, tennis rackets, and buses. This underscores our model's ability to emulate sophisticated perceptual behaviors that reflect complex human visual strategies. Rather interestingly, these human-like patterns emerge from solely being trained on the ImageNet image recognition task, without reliance on the typical inductive biases innate to human cognition (\emph{e.g.}, biases concerning objects, agents, space, and biological motion \cite{kellman1983perception, spelke1992origins, spelke1994initial, viola2004can, simion2011processing, ullman2012simple, stahl2015observing, reynolds2018development}).

Second, orthogonal to investigating spatial adaptiveness, we examine the extent to which AdaptiveNN aligns with human judgments in assessing which visual environments are more challenging for a given task and necessitate more thorough scrutiny. Specifically, human participants ($n$=10) were tasked with rating each image within six representative categories of the ImageNet validation set, according to the difficulty of classifying each image. In Fig. \ref{fig:fig5}c and Supplementary Data Tab. \textcolor{blue}{27}, these human-assessed difficulty scores are normalized on an individual basis, averaged across participants, and then compared against the normalized state values predicted by the AdaptiveNN learned on ImageNet, which reflect our model's assessments of each image's difficulty level. The evaluations made by AdaptiveNN demonstrate a strong correlation with human judgments (all P$<$0.0001; Pearson correlation coefficient $\rho\!\in\![0.54, 0.80]$). Fig. \ref{fig:fig5}d further presents qualitative examples of the relatively `easy' and `difficult' data identified by our model. The visualizations are generally reasonable; images depicted from typical perspectives with clear, relevant content tend to be deemed `easy'. These findings suggest that our model closely approximates human-like proficiency in dynamically allocating visual perception resources across varied visual environments -- a critical characteristic of human visual systems.

Finally, we establish several `visual Turing tests' \cite{lake2015human} to compare AdaptiveNN with human vision. In these tests, human judges are given paired examples of visual perception behaviors from humans and our model, and instructed to identify which come from the machine. The results are evaluated using the accuracy of human judgments: 50\% indicates perfectly human-like behaviors that are indistinguishable from humans, while 100\% represents the worst case. Each participant ($n$=39) has completed 216 trials with blocked feedback to investigate both the spatial-wise visual fixation behaviors and the sample-wise visual difficulty assessment behaviors of AdaptiveNN, with the judgments analyzed individually and in aggregate. Importantly, we randomly replace the `machine' behaviors of some trials with `human' or `random' behaviors without letting participants know, and separately evaluate the accuracy of these trials, establishing two randomized control groups as baselines for comparison. 
The detailed procedure of `visual Turing tests' is illustrated in Extended Data Fig. \ref{fig:visual_turing_details}a. Some examples of the trials can be found in Supplementary Data Fig. \textcolor{blue}{1-2}.

The results are summarized in Fig. \ref{fig:fig5}e-\ref{fig:fig5}f, Extended Data Fig. \ref{fig:visual_turing_details}b, Supplementary Data Tab. \textcolor{blue}{28-29}, and Supplementary Data Fig. \textcolor{blue}{3-4}. In both scenarios, human judges achieve only 50-51\% accuracies in correctly identifying `AdaptiveNN v.s. human', which do not acceptably outperform random guessing in statistics (t(38) = 0.90, -0.09, P = 0.37, 0.93). Additionally, these `machine v.s. human' judgments do not exhibit a significant difference from the 49-50\% accuracies of the `human v.s. human' baselines (t(38) = 0.97, 0.40, P = 0.33, 0.69). In contrast, `random v.s. human' results in considerably easier Turing test tasks (accuracies $\ge$80\%). These observations demonstrate that in general, AdaptiveNN approaches an indistinguishable level from the adaptive perceptual behaviors of human vision.

\subsection{Ablation studies and analysis of the design of AdaptiveNN}

In pursuit of a comprehensive understanding of our work, Extended Data Fig. \ref{fig:fig6} and Supplementary Data Tab. \textcolor{blue}{30-49} establishes a series of evaluations uncovering that the components of AdaptiveNN function as we expect, and that our design markedly outperforms alternative choices. 

Extended Data Fig. \ref{fig:fig6}a and Supplementary Data Tab. \textcolor{blue}{30-37} examine the effectiveness of a broad array of possible strategies for localizing visual fixations. The reinforcement learning algorithm of AdaptiveNN achieves significantly higher validation accuracies than the most competitive baseline across all the scenarios (all P$<$0.0001), especially with limited numbers of fixations. Intriguingly, although GardCAM has been widely used as a feasible algorithm to visualize the regions relevant to the decision-making of deep networks \cite{zhou2023foundation, wang2023incorporating}, its application in selecting visual fixations does not yield competitive performance against AdaptiveNN, even though it is augmented with a Gaussian mixture model and additional computation. Moreover, other possible methods for training the fixation selection policy, such as spatial transformer net and Gumbel-Softmax, do not exhibit the potential to approach reinforcement learning. They fail to secure noteworthy gains over pre-defined non-adaptive policies like random or Gaussian sampling.

In Extended Data Fig. \ref{fig:fig6}b and Supplementary Data Tab. \textcolor{blue}{38-45}, we show that the state values predicted by the \emph{Vision Agent} of AdaptiveNN are strongly correlated with the test loss of the validation data. This correlation indicates that, for a given test sample whose label is unknown, we can leverage its associated state values as reliable proxies of how far the outputs of our model are from the accurate prediction. This phenomenon is highly consistent with our goal of introducing the value network (Fig. \ref{fig:fig1_2}d). In this sense, Extended Data Fig. \ref{fig:fig6}c and Supplementary Data Tab. \textcolor{blue}{46-49} provide further evidence supporting that the strategy of concluding the observation processes of the samples exhibiting smaller rather than larger state values is beneficial for a higher overall computational efficiency. This strategy underscores the efficacy of the value network in guiding the allocation of computational resources toward optimizing model performance.

Extended Data Fig. \ref{fig:fig6}d-\ref{fig:fig6}e present system-level comparisons against representative state-of-the-art methods for enhancing the energy efficiency of deep networks. Extended Data Fig. \ref{fig:fig6}d focuses on the recently proposed algorithms that leverage the spatial redundancy of visual data, whereas Extended Data Fig. \ref{fig:fig6}e considers existing multi-exit models characterized by an online-adjustable computational cost. AdaptiveNN outperforms all of them by marked margins when consuming less or comparable amounts of computation, even though the major motivation of our work is to emulate the visual perception behaviors of humans to drive a paradigm shift from `passive' to `active and adaptive' vision models, instead of attaining optimal engineering performance.

\section{Discussion}
\label{sec:discussion}

Human vision is distinguished by its remarkable flexibility to adapt to spatial regions with different content, varying complexities of visual environments, diverse task demands, and fluctuating resource availability for perception. 
In contrast, current machine vision models mainly adopt `passive' paradigms, which usually perceive everything everywhere in parallel with an identical computational graph, regardless of the specific characteristics of variable visual environments, tasks and resources.
This lack of adaptive adjustment results in an `impossible triangle' formed by high-dimensional visual inputs, large-scale neural networks, and efficiency: under the scaling laws, the first two tend to be essential for complex real-world vision problems, but they significantly compromise efficiency. 
This inherent limitation impedes both future advancements and application in diverse real-world scenarios.
In this article, we aim to address this issue by enabling neural networks to computationally emulate the adaptive behaviors of human visual systems, driving a paradigm shift from `passive' to `active and adaptive' vision models.

Specifically, we establish an AdaptiveNN framework that perceives scenes by sequentially fixating on pertinent regions, incrementally integrating information across fixations, and actively concluding its observation to accomplish the task of interest. Theoretical analyses suggest that such models can be trained using reinforcement learning without specialized supervision, relying solely on simple task-driven objectives. Our resulting method, AdaptiveNN, substantially reduces the computational cost of well-performing computer vision models by up to $28\times$ without sacrificing accuracy (Fig. \ref{fig:fig2}b, \ref{fig:fig2}c, \ref{fig:fig4}a, \ref{fig:embodiedAI_results}b, and Extended Data Fig. \ref{fig:fig3}a). 
Moreover, it exhibits human-like flexibility in adjusting its inference cost online without necessitating additional training (Fig. \ref{fig:fig2}b, \ref{fig:fig2}c, \ref{fig:fig4}a, \ref{fig:embodiedAI_results}b, and Extended Data Fig. \ref{fig:fig3}a), as well as in customizing its perceptual strategies conditioned on variable task demands through modifying the training objective (Fig. \ref{fig:fig4}c, \ref{fig:fig4}d) or introducing natural language prompts as inputs (Fig. \ref{fig:embodiedAI_results}).
Additionally, AdaptiveNN is also distinctive in its enhanced interpretability through analyzing its fixation patterns (Fig. \ref{fig:fig2}a, \ref{fig:fig4}b, \ref{fig:fig4}d, \ref{fig:fig4}f, \ref{fig:embodiedAI_results}d, and Extended Data Fig. \ref{fig:fig3}b-\ref{fig:fig3}e), in a manner akin to understanding human visual systems \cite{ward2002fast, itti2001computational, henderson2003human, valliappan2020accelerating}. 
These aforementioned favorable attributes align closely with the extensively recognized advantages of human visual systems \cite{ward2002fast, wolfe2004attributes, najemnik2005optimal, ma2011behavior, mnih2014recurrent, henderson2017meaning, gottlieb2018towards, han2021dynamic, hanning2023dissociable}. In this sense, we believe that our work significantly advances the resolution of the open challenge raised by LeCun, Bengio, and Hinton \cite{lecun2015deep}, opening up a new horizon for developing more energy-efficient, adaptable, and interpretable computer vision models. These properties are critical for realistic scenarios such as wearable devices, mobile phones, robotics, embedded devices, autonomous vehicles, and medical AI applications.

The design and theoretical analyses of AdaptiveNN have carefully avoided involving strong assumptions or specialized implementation configurations. As a consequence, it is compatible with a wide array of state-of-the-art representation learning backbones \cite{ronneberger2015u, he2016deep, huang2017densely, touvron2021training, liu2021swin, tan2021efficientnetv2, dong2022cswin}, which can be readily incorporated as the feature-extraction module within our model. Furthermore, the outputs of AdaptiveNN interface seamlessly with various vision tasks, such as recognition, medical diagnosis or prognosis \cite{zhou2023foundation}, segmentation \cite{isensee2021nnu}, locating visual objects \cite{redmon2016you, liu2016ssd}, and embodied multimodal large language models (MLLM) \cite{li2024visionlanguage}.
In this article, we first demonstrate our model's generalizability using several representative backbone networks: ResNet (convolutional neural network) \cite{he2016deep} and DeiT (vision Transformer) \cite{touvron2021training}, through the lens of two common, foundational elements of diverse vision tasks: `what' and `where' \cite{larochelle2010learning}, namely semantic understanding and element localization (Fig. \ref{fig:fig2}, \ref{fig:fig4}, and Extended Data Fig. \ref{fig:fig3}). 
To demonstrate the versatility of AdaptiveNN, we further deploy it as the perceptual module of a language-driven embodied MLLM (Fig. \ref{fig:embodiedAI_results}).
Our considerations of building up a general framework not only verify the extensive applicability of our findings, but also facilitate a focused and comprehensive examination of the benefits derived from formulating human-like adaptive visual perception. Additionally, we believe that the sufficient flexibility of our approach offers promising avenues for further extensions.

Our results are also appealing in their contributions to the ongoing discourse on human visual cognition, particularly concerning the debate over the role of innateness in learning perceptual behaviors. This debate has persisted for centuries, questioning whether certain visual behaviors are inherent at birth or learned through experience \cite{locke1847essay, leibniz1997new, zaadnoordijk2022lessons, orhan2024learning}. Some developmental psychologists have postulated that innate biases, such as those related to objects, agents, space, and biological motion \cite{kellman1983perception, spelke1992origins, spelke1994initial, viola2004can, simion2011processing, ullman2012simple, stahl2015observing, reynolds2018development}, may shape the process of learning from the environment. Conversely, others argue that visual capabilities can develop in the absence of such biases, heavily influenced by the richness of the developing child's experience \cite{elman1996rethinking, orhan2024learning}. Our efforts revisit this age-old `nature versus nurture' debate from a modern perspective: we demonstrate the possibility of investigating these aforementioned claims via rigorous computational simulations. AdaptiveNN emerges as a general, scalable, and sufficiently human-like proxy that learns from visual data with maximally eliminating the innate biases or abilities in humans. By being trained solely on real-world visual tasks like ImageNet object recognition, AdaptiveNN not only achieves beyond human-level accuracies \cite{russakovsky2015imagenet} (Fig. \ref{fig:fig2}), but also exhibits mostly indistinguishable behaviors from humans (Fig, \ref{fig:fig5}), in terms of either the `eye movement' patterns in novel scene observation or assessing the `difficulty levels' of various visual environments. These findings suggest that many adaptive behaviors and basic capabilities of human vision could indeed be acquired through routine visual tasks, without necessitating strong innate biases. In this sense, we are among the first to leverage advanced AI methods like deep networks and reinforcement learning to explore fundamental cognitive science questions under controlled experimental conditions \cite{bambach2018toddler, orhan2020self, orhan2024learning, vong2024grounded}. Additionally, we hope that our work will inspire new interdisciplinary collaborations between machine learning and broader fields, given the critical role of eye movements in probing into human vision and mind \cite{keller1983mind, rayner1998eye, henderson2003human, hayhoe2005eye, najemnik2005optimal, land2009vision, valliappan2020accelerating}, and their widespread applications to various research communities, such as visual content analysis \cite{itti2001computational}, graphic or web designs \cite{nielsen2010eyetracking, bylinskii2017learning}, driving \cite{land2009looking}, gaming \cite{smith2006use}, and medical research \cite{jones2008absence, bigolin2022reflacx}.

In conclusion, the encouraging findings of this article demonstrate a new avenue for developing deep learning methodologies inspired by human vision. The efficacy of AdaptiveNN and its behavioral alignment with human visual systems underscore its potential. We anticipate future explorations in this direction to benefit both the AI and cognitive science communities -- promising not only to foster the creation of next-generation computer vision models that are efficient, adaptable, and interpretable, but also to provide powerful computational tools for investigating human behavioral and learning processes.

We also believe our work offers valuable insights into equipping computer vision models with adaptive sequential `reasoning'-like perception capabilities using reinforcement learning, analogous to the approach employed in DeepSeek-R1 \cite{guo2025deepseek}. Towards this direction, we demonstrate how to model visual perception tasks as sequential decision procedures, and reveal why and how such models should be trained using reinforcement learning. Our resulting models can adaptively employ a larger number of strategically selected visual fixations to solve more challenging vision tasks.

\section{Methods}
\label{sec:method}

\subsection{Theoretical learning principles of AdaptiveNN}
\label{sec:theory}

Here we present the proof of Theorem \ref{theorem:1} and Eq. (\ref{eq:rl_target_limit}).

\begin{theorem*}
    The gradients of $\mathrm{L}(\bm{\theta})$ can be decomposed into a combination of representation learning and self-rewarding reinforcement learning objectives:
    \begin{equation}
        \nabla_{\bm{\theta}}\mathrm{L}(\bm{\theta}) = 
            \nabla_{\bm{\theta}}\mathrm{L}_{\textnormal{rep}}(\bm{\theta}) + 
            \nabla_{\bm{\theta}}\mathrm{L}_{\textnormal{rl}}(\bm{\theta}),
    \end{equation}
    where
    \begin{equation}
        \begin{split}
            \nabla_{\bm{\theta}}\mathrm{L}_{\textnormal{rep}} 
            &= 
            \underbrace{
            \mathbb{E}_{\bm{X}, y, \bm{l}_{1:T}} 
            \sum\nolimits_{t=1}^{T}
            P(t_{\textnormal{o}}=t)
            \nabla_{\bm{\theta}}  \mathcal{L}(y, q(\bm{\theta}, \bm{X}, \bm{l}_{1:t}))
            }_{\textnormal{representation learning}} 
                \\
            \nabla_{\bm{\theta}}\mathrm{L}_{\textnormal{rl}}
            &=
            \underbrace{
                -\mathbb{E}_{\bm{X}, y, \bm{l}_{1:T}} 
                \sum\nolimits_{t=1}^{T}
                \left[
                    \left(
                        \sum\nolimits_{t'=t}^{T} r_{t'}
                    \right)
                    \nabla_{\bm{\theta}}  \log p(\bm{l}_{t} \lvert \bm{\theta}, \bm{X}, \bm{l}_{1:(t-1)}) 
                \right]
                }_{\textnormal{self-rewarding reinforcement learning}},
                \\
            & \ \ \ \ \ \ \ \ r_{t'} = -P(t_{\textnormal{o}}=t') \mathcal{L}(y, q(\bm{\theta}, \bm{X}, \bm{l}_{1:t'})).
        \end{split}
    \end{equation}
\end{theorem*}
\begin{proof}
Taking derivatives of $\mathrm{L}(\bm{\theta})$ with respect to $\bm{\theta}$, we have
\begin{small}
\begin{equation}
    \label{eq:derivatives_of_f}
    \begin{split}
        \nabla_{\bm{\theta}}\mathrm{L} &= 
            \mathbb{E}_{\bm{X}, y, t_{\textnormal{o}} \sim p(t_{\textnormal{o}})} \left[ 
                \int_{\bm{l}_{1:t_{\textnormal{o}}}} 
                p(\bm{l}_{1:t_{\textnormal{o}}}\lvert\bm{\theta}, \bm{X}) 
                \frac{
                        \partial \mathcal{L}(y, q(\bm{\theta}, \bm{X}, \bm{l}_{1:t_{\textnormal{o}}}))
                    }{
                        \partial \bm{\theta}
                    }
            \right. \\
             & \ \ \ \ \ \ \ \ \ \ \ \ \ \ \ \ \ \ \ \ \ 
            + \left.
                \int_{\bm{l}_{1:t_{\textnormal{o}}}} 
                \mathcal{L}(y, q(\bm{\theta}, \bm{X}, \bm{l}_{1:t_{\textnormal{o}}}))
                \frac{
                    \partial p(\bm{l}_{1:t_{\textnormal{o}}}\lvert\bm{\theta}, \bm{X}) 
                    }{
                        \partial \bm{\theta}
                    }
            \right] \\
        & = \mathbb{E}_{\bm{X}, y, t_{\textnormal{o}} \sim p(t_{\textnormal{o}})} 
            \int_{\bm{l}_{1:t_{\textnormal{o}}}} 
            p(\bm{l}_{1:t_{\textnormal{o}}}\lvert\bm{\theta}, \bm{X}) 
            \left[
                \frac{
                        \partial \mathcal{L}(y, q(\bm{\theta}, \bm{X}, \bm{l}_{1:t_{\textnormal{o}}}))
                    }{
                        \partial \bm{\theta}
                    }
            \right. \\
        & \ \ \ \ \ \ \ \ \ \ \ \ \ \ \ \ \ \ \ \ \ 
        + \left.
            \mathcal{L}(y, q(\bm{\theta}, \bm{X}, \bm{l}_{1:t_{\textnormal{o}}}))
            \frac{
                \partial \log p(\bm{l}_{1:t_{\textnormal{o}}}\lvert\bm{\theta}, \bm{X}) 
                }{
                    \partial \bm{\theta}
                }
        \right] \\
        & = \mathbb{E}_{\bm{X}, y, t_{\textnormal{o}} \sim p(t_{\textnormal{o}})} 
            \int_{\bm{l}_{1:t_{\textnormal{o}}}} 
            \int_{\bm{l}_{t_{\textnormal{o}}+1:T}}
            p(\bm{l}_{t_{\textnormal{o}+1}:T}
                \lvert\bm{\theta}, \bm{X}, \bm{l}_{1:t_{\textnormal{o}}})
            \\ 
            & \ \ \ \ \ \ \ \ \ \ \ \ \ \ \ \ \ \ \ \ \ 
            p(\bm{l}_{1:t_{\textnormal{o}}}\lvert\bm{\theta}, \bm{X}) 
            \left[
                \frac{
                        \partial \mathcal{L}(y, q(\bm{\theta}, \bm{X}, \bm{l}_{1:t_{\textnormal{o}}}))
                    }{
                        \partial \bm{\theta}
                    }
            \right. \\
        & \ \ \ \ \ \ \ \ \ \ \ \ \ \ \ \ \ \ \ \ \ 
        + \left.
            \mathcal{L}(y, q(\bm{\theta}, \bm{X}, \bm{l}_{1:t_{\textnormal{o}}}))
            \frac{
                \partial \log p(\bm{l}_{1:t_{\textnormal{o}}}\lvert\bm{\theta}, \bm{X}) 
                }{
                    \partial \bm{\theta}
                }
        \right] \\
        & = \mathbb{E}_{\bm{X}, y, t_{\textnormal{o}} \sim p(t_{\textnormal{o}})} 
            \int_{\bm{l}_{1:T}} 
            p(\bm{l}_{1:T}\lvert\bm{\theta}, \bm{X})
            \left[
                \frac{
                        \partial \mathcal{L}(y, q(\bm{\theta}, \bm{X}, \bm{l}_{1:t_{\textnormal{o}}}))
                    }{
                        \partial \bm{\theta}
                    }
            \right. \\
        & \ \ \ \ \ \ \ \ \ \ \ \ \ \ \ \ \ \ \ \ \ 
        + \left.
            \mathcal{L}(y, q(\bm{\theta}, \bm{X}, \bm{l}_{1:t_{\textnormal{o}}}))
            \frac{
                \partial \log p(\bm{l}_{1:t_{\textnormal{o}}}\lvert\bm{\theta}, \bm{X}) 
                }{
                    \partial \bm{\theta}
                }
        \right],
    \end{split}
\end{equation}
\end{small}
where $T$ is the maximum possible value of $t_{\textnormal{o}}$.
Since $t_{\textnormal{o}}$ and $\bm{l}_{1:t_{\textnormal{o}}}$ are mutually independent random variables, we have
\begin{small}
\begin{equation}
    \label{eq:derivatives_of_f_2}
    \begin{split}
        \nabla_{\bm{\theta}}\mathrm{L} & = \mathbb{E}_{\bm{X}, y, \bm{l}_{1:T}} 
            \left[
                \mathbb{E}_{t_{\textnormal{o}} \sim p(t_{\textnormal{o}})}
                \frac{
                        \partial \mathcal{L}(y, q(\bm{\theta}, \bm{X}, \bm{l}_{1:t_{\textnormal{o}}}))
                    }{
                        \partial \bm{\theta}
                    }
            \right. \\
        & \ \ \ \ \ \ \ \ \ \ 
        + \left.
            \mathbb{E}_{t_{\textnormal{o}} \sim p(t_{\textnormal{o}})}
            \mathcal{L}(y, q(\bm{\theta}, \bm{X}, \bm{l}_{1:t_{\textnormal{o}}}))
            \frac{
                \partial \log p(\bm{l}_{1:t_{\textnormal{o}}}\lvert\bm{\theta}, \bm{X}) 
                }{
                    \partial \bm{\theta}
                }
        \right].
    \end{split}
\end{equation}
\end{small}
Moreover, note that $\log p(\bm{l}_{1:t_{\textnormal{o}}}\lvert\bm{\theta}, \bm{X})$ can be factorized as:
\begin{small}
\begin{equation}
    \label{eq:log_factorize}
    \begin{split}
        \log p(\bm{l}_{1:t_{\textnormal{o}}}\lvert\bm{\theta}, \bm{X})
         =& \log p(\bm{l}_{1}\lvert\bm{\theta}, \bm{X}) + \log p(\bm{l}_{2} \lvert \bm{\theta}, \bm{X}, \bm{l}_{1})\\
        & + \cdots + \log p(\bm{l}_{t_{\textnormal{o}}}\lvert\bm{\theta}, \bm{X}, \bm{l}_{1:t_{\textnormal{o}} - 1}),
    \end{split}
\end{equation}
\end{small}
which can be considered as solving the state distribution over a Markov chain. Then, we have:
\begin{small}
\begin{equation}
    \label{eq:log_factorize_v2}
    \begin{split}
        & \mathbb{E}_{t_{\textnormal{o}} \sim p(t_{\textnormal{o}})}
        \mathcal{L}(y, q(\bm{\theta}, \bm{X}, \bm{l}_{1:t_{\textnormal{o}}}))
        \frac{
            \partial \log p(\bm{l}_{1:t_{\textnormal{o}}}\lvert\bm{\theta}, \bm{X}) 
            }{
                \partial \bm{\theta}
            } \\
        = & \sum_{t'=1}^{T}
        \left[
            P(t_{\textnormal{o}}=t')
            \mathcal{L}(y, q(\bm{\theta}, \bm{X}, \bm{l}_{1:t'}))
            \sum_{t=1}^{t'}
            \frac{
                \partial \log p(\bm{l}_{t} \lvert \bm{\theta}, \bm{X}, \bm{l}_{1:(t - 1)})  
                }{
                    \partial \bm{\theta}
                }
        \right]\\
        = & \sum_{t=1}^{T}
        \left[
            \left(
                \sum_{t'=t}^{T} P(t_{\textnormal{o}}=t') \mathcal{L}(y, q(\bm{\theta}, \bm{X}, \bm{l}_{1:t'}))
            \right)
            \frac{
                \partial 
                \log p(\bm{l}_{t} \lvert \bm{\theta}, \bm{X}, \bm{l}_{1:(t - 1)}) 
                }{
                    \partial \bm{\theta}
                }
        \right]
    \end{split}
\end{equation}
\end{small}
Furthermore, combining Eq. (\ref{eq:derivatives_of_f_2}) and Eq. (\ref{eq:log_factorize_v2}), we finally obtain
\begin{small}
\begin{equation}
    \label{eq:final_decomposed_target}
    \begin{split}
         &\nabla_{\bm{\theta}}\mathrm{L} = 
        \underbrace{
        \mathbb{E}_{\bm{X}, y, \bm{l}_{1:T}} 
        \sum_{t=1}^{T}
        P(t_{\textnormal{o}}=t)
        \frac{
                \partial \mathcal{L}(y, q(\bm{\theta}, \bm{X}, \bm{l}_{1:t}))
            }{
                \partial \bm{\theta}
            }
            }_{\textnormal{representation learning objective},\ \nabla_{\bm{\theta}}\mathrm{L}_{\textnormal{rep}}} \\
        & 
        + \underbrace{
            \mathbb{E}_{\bm{X}, y, \bm{l}_{1:T}} 
            \sum_{t=1}^{T}
            \left[
                \left(
                    \sum_{t'=t}^{T} P(t_{\textnormal{o}}=t') \mathcal{L}(y, q(\bm{\theta}, \bm{X}, \bm{l}_{1:t'}))
                \right)
                \frac{
                    \partial 
                    \log p(\bm{l}_{t} \lvert \bm{\theta}, \bm{X}, \bm{l}_{1:(t - 1)}) 
                    }{
                        \partial \bm{\theta}
                    }
            \right]
        }_{\textnormal{self-rewarding reinforcement learning objective},\ \nabla_{\bm{\theta}}\mathrm{L}_{\textnormal{rl}}},
    \end{split}
\end{equation}
\end{small}
which proves Theorem \ref{theorem:1}.
\end{proof}

\noindent \textbf{Proof of Eq. (\ref{eq:rl_target_limit})}.
It is obvious that
\begin{small}
\begin{equation}
    \label{eq:proof_limit_1}
    \begin{split}
        \lim_{\gamma \to 0}\sum\nolimits_{t'=t}^{T}  \gamma^{t'-t} \left(r_{t'} - r_{t'-1}\right)
        &=
        r_{t} - r_{t-1},
        \\
        \lim_{\gamma \to 1}\sum\nolimits_{t'=t}^{T}  \gamma^{t'-t} \left(r_{t'} - r_{t'-1}\right)
        &=
        r_{T} - r_{t-1}.
    \end{split}
\end{equation}
\end{small}
On top of Eq. (\ref{eq:rl_target_1}), we actually have
\begin{small}
\begin{equation}
    \label{eq:proof_limit_2}
    \begin{split}
        &\mathbb{E}_{\bm{l}_{t}} r_{t-1}
        \nabla_{\bm{\theta}}{
            \log p_{\bm{\pi}}(\bm{l}_{t}\lvert\bm{s}_{t-1}) 
            }\\
        = &r_{t-1} \int_{\bm{l}_{t}} 
        p_{\bm{\pi}}(\bm{l}_{t}\lvert\bm{s}_{t-1})
        \frac{1}{p_{\bm{\pi}}(\bm{l}_{t}\lvert\bm{s}_{t-1})}
        \nabla_{\bm{\theta}}{
            p_{\bm{\pi}}(\bm{l}_{t}\lvert\bm{s}_{t-1})
            } \\
        = & r_{t-1}
        \frac{
            \partial \int_{\bm{l}_{t}} p_{\bm{\pi}}(\bm{l}_{t}\lvert\bm{s}_{t-1})
        }{
            \partial \bm{\theta}
        }=0.
    \end{split}
\end{equation}
\end{small}
Eq. (\ref{eq:rl_target_limit}) can be obtained by combining Eq. (\ref{eq:proof_limit_1}) and Eq. (\ref{eq:proof_limit_2}).

\subsection{Evaluation tasks for AdaptiveNN}
\label{sec:intro_tasks}

Here we describe the 9 different tasks used for evaluating AdaptiveNN, each associated with one or more datasets, yielding 17 benchmarks in total. For all tasks, we held out 20\% of training data to perform a hyper-parameter search, and then put this data back to the training set, reporting final results. When involved, we consider the number of floating point operations (FLOPs) as the measure of computational cost for the inference of a model.

\subsubsection{Computer vision tasks}
\label{sec:intro_tasks_cv}

\noindent \textbf{Large-scale real-world visual understanding: ImageNet}.
ImageNet is a large-scale and diverse dataset of high-quality Internet images \cite{russakovsky2015imagenet}. Each image is annotated with a label of its category. The categories are organized according to the WordNet hierarchy \cite{miller1995wordnet}, covering a wide range of common visual content, including objects, buildings, humans, animals, scenes, etc. ImageNet is a very popular benchmark for evaluating deep learning methodologies \cite{he2016deep, huang2017densely, dosovitskiy2021an, orhan2024learning}, and  has been instrumental in advancing computer vision and machine learning research. In this article, we adopted the standard training-validation split, with $\sim$1,280,000 images for training, 50,000 images for validation, and 1,000-class annotations. Following the common practice, we used validation accuracy as the performance metric.
\vskip 0.1in

\noindent \textbf{Fine-grained visual recognition: six benchmarks}.
To examine whether AdaptiveNN has similar capabilities to humans in terms of filtering subtle task-relevant information out of large quantities of noise, we considered six fine-grained visual recognition tasks. These tasks are marked by small inter-class variations (such as distinguishing among visually closely related bird species), and large intra-class variations (such as highly diversified backgrounds and viewpoints). Here we describe the six corresponding datasets we used. For all of them, we adopted the standard training-validation split, and use validation accuracy as the performance metric
\begin{itemize}[itemsep=0pt,topsep=0pt,parsep=0pt]
    \item Caltech-UCSD birds-200-2011 (CUB-200-2011) \cite{wah2011caltech} is one of the most widely-used fine-grained categorization dataset. It consists of 11,788 images of 200 subcategories belonging to birds, 5,994 for training and 5,794 for testing.
    \item North America Birds (NABirds) \cite{van2015building} contains 48,562 annotated photographs of 400 species of commonly observed birds in North America. Each species has more than 100 photographs, including annotations for males, females, and juveniles. All the data is divided into 555 visual categories.
    \item Oxford-IIIT Pet \cite{parkhi2012cats} is a 37 category pet dataset with $\sim$200 images for each class. The images are highly diversified in scale, pose, and lighting.
    \item Stanford Dogs \cite{khosla2011novel} contains 20,580 images of 120 breeds of dogs from around the world. The dataset is divided into 12,000 images for training and 8,580 images for validation.
    \item Stanford Cars \cite{krause20133d} contains 16,185 images of 196 classes of cars. The data is divided into 8,144 training images and 8,041 validation images. The categories are typically built at the level of make, model, and year.
    \item FGVC-Aircraft \cite{maji2013fine} is a benchmark that contains 10,200 images of 102 different classes of aircraft, where each class has 100 images. The data is organized in a four-level hierarchy, namely model, variant, family, and manufacturer.
\end{itemize}
\vskip 0.1in

\noindent \textbf{Efficient processing of visual data from real driving scenarios: STSD}.
Similar to human visual systems, AdaptiveNN is not only able to process relatively object-centric visual data such as ImageNet and fine-grained classification datasets, but also applicable to more general visual perception scenarios. For example, it can process non-object-centric, complex images collected in the wild without specified pre-processing. As a representative example, we considered the task of recognizing traffic signs on the Swedish traffic signs dataset (STSD) \cite{larsson2011using}. The dataset consists of 960$\times$1,280 road-scene images, captured from real moving vehicles, and the task is to recognize the existence and types of the speed limit signs. Note that the targets of interest are generally small, diversely distributed, and sometimes not clear. In this article, we used two subsets comprising 747 and 648 images for training and validation, respectively. Validation accuracy is used as the performance metric.
\vskip 0.1in

\noindent \textbf{Visual search with diversified task demands: localizing arbitrary digits in multi-digit images}.
To investigate whether AdaptiveNN has the human-like adaptability of customizing visual perception behaviors conditioned on different task demands, we considered a visual search scenario where the categories and number of targets are assumed to be flexibly changed. Specifically, we created a digit localization dataset by generating 224$\times$224 images, each randomly populated with 6 to 10 28$\times$28 MNIST digits \cite{lecun1998mnist} against a black background without repetition of digits. We established a large-scale dataset with 500,000 images for training and 50,000 images for validation. To define a visual search task, we specified arbitrary numbers and classes of digits, and trained our model to identify the locations of these specified digits within each input image. This requires a model to not only recognize correct targets, but also accurately localize multiple targets in a single image. To measure the performance of a given model, we randomly defined many visual tasks, and obtained the average success rate on the validation set. Notably, one success means retrieving exactly all the digits demanded by a task from an input, while the success rate of a task is defined as the number of successes divided by the number of all samples.
\vskip 0.1in


\noindent \textbf{Image processing in medical scenarios: RSNA pneumonia detection}.
To demonstrate the efficacy of AdaptiveNN in applications where interpretability holds vital importance, a pneumonia detection scenario was considered. We used the RSNA Pneumonia dataset, which consists of $\sim$30,000 frontal view chest radiographs \cite{shih2019augmenting}. Each image in the dataset is annotated with image-level labels indicating the presence or absence of pneumonia, as well as bounding boxes for pulmonary opacity which are visual signals for the disease. The annotations are provided by 18 board-certified radiologists from 16 institutions. In this article, we leveraged the image-level labels to train AdaptiveNN, and compared the locations it fixates on with the pulmonary opacity identified by clinicians to assess its interpretability. The dataset was randomly divided into training and validation sets, following a ratio of 85\% and 15\%, respectively. The model's diagnostic accuracy was quantified through the area under the receiver operating characteristic curve (AUROC) on the validation data.

\subsubsection{Embodied AI tasks}

\noindent \textbf{CALVIN long-horizon multi-task language control benchmarks}.
We adopt CALVIN \cite{mees2022calvin} to construct the benchmarks for validating the performance of our multi-task, language-guided embodied agent. Within CALVIN, the agent is tasked with executing sequences of actions, each consisting of five subtasks defined through natural language instructions. The model's effectiveness is measured by its average successful length across 1,000 task sequences, with scores ranging from 0 to 5 based on the number of subtasks completed successfully, as detailed in Extended Data Fig. \ref{fig:embodiedAI_details}b. The CALVIN dataset is organized into four distinct environmental subsets, labeled A through D, each characterized by unique visual backgrounds and object arrangements. Each of these subsets encompasses approximately 24,000 robot manipulation trajectories accompanied by language annotations. We train our embodied multimodal large language models on these language-annotated trajectories. To thoroughly evaluate the model's ability to imitate and generalize, we conduct experiments under two scenarios: 1) D$\to$D: training and testing within the same environment, and 2) ABCD$\to$D: training on data from all four environments while testing on a single target domain.

\subsubsection{Comparisons with human visual perception behaviors} \label{sec:intro_tasks_vs_human}
To demonstrate the potential of AdaptiveNN as a valuable tool for investigating human visual cognition, we evaluated humans and AdaptiveNN side by side on the same tests of visual perception behaviors. Specifically, these tests were designed under two goals, namely 1) spatial-wise, examining the locations of visual regions that a human/model fixates on; and 2) sample-wise, examining the difficulty level that a human/model assesses to accomplish the given task based on each individual visual environment. To attain these goals, we conducted three groups of experiments, as described below.

\emph{First}, through the lens of spatial-wise adaptiveness, we investigated the consistency of the locations of visual fixations selected by AdaptiveNN and humans. We employed the saliency in context (SALICON) benchmark \cite{jiang2015salicon}, which consists of 10,000 training images and 5,000 validation images. Every image is annotated with a map of the centers of human gazing. Each gazing center is treated as a single point in the map. The maps of gazing centers were collected based on the paid Amazon Mechanical Turk (AMT) crowdsourcing marketplace, with each image observed by $\sim$60 subjects. All participants had normal or corrected-to-normal vision and normal color vision. The images were presented to each subject in a random order, where each image was presented for 5 seconds. The subjects were instructed to explore the image freely by looking at anywhere they wanted to look, with no further instructions on where they should look in the images. The gazing center locations were obtained by a 100 Hz resampling and processed by excluding the fast-moving data corresponding to saccade processes. 

To compare humans and AdaptiveNN in terms of spatial-wise adaptive visual perception, a metric named `normalized human-like score' was defined. For each image, the average density map of the gazing centers of all $\sim$60 observers was used as the ground truth distribution of the real focal centers of human vision. We let AdaptiveNN select $n$ visual fixation regions on top of each image, mimicking the process of freely observing the image for a fixed length of time. Then, we obtained the probability that the ground truth human gazing centers fall into the visual fixation regions localized by AdaptiveNN, denoted as $p_n^{\textnormal{AdaNN}}$. Similarly, consider sampling visual fixation regions following the gazing center distribution of an arbitrary single person (within $\sim$60 observers), or fully randomly, and then taking the expectations. As a result, we had the corresponding expected probabilities $\mathbb{E}[p_n^{\textnormal{Single-human}}]$ and $\mathbb{E}[p_n^{\textnormal{Random}}]$, respectively. Built upon this, we defined that
\vskip -0.07in

\begin{equation}
    \label{eq:human_like_score}
    \textnormal{normalized human-like score} = \frac{p_n^{\textnormal{AdaNN}} - \mathbb{E}[p_n^{\textnormal{Random}}]}{\mathbb{E}[p_n^{\textnormal{Single-human}}] - \mathbb{E}[p_n^{\textnormal{Random}}]}.
\end{equation}

\vskip 0.03in \noindent
Notably, Eq. (\ref{eq:human_like_score}) = 1 indicates that the consistency between AdaptiveNN and the average characteristics of the spatial-wise visual fixation behaviors of people is approximately the same as the level of an average individual human observer. On the other hand, Eq. (\ref{eq:human_like_score}) = 0 provides a baseline of randomly fixating on the visual environments. In our implementation, normalized human-like scores were calculated upon mini-batches of data sampled from the dataset to reasonably reflect their values across different sets of visual environments. We adopted $n=3$ and a batch size of 64. Moreover, we reported the results on top of the two splits of SALICON (split-1/2 corresponds to the train/val split of SALICON). They were not particularly distinguished since our model had never been trained on SALICON.

\emph{Second}, through the lens of sample-wise adaptiveness, we investigated whether our model is consistent with humans in judging which visual environments are relatively easier or more difficult for a given task, and should be paid less or more attention to observing them. To achieve this, we started by measuring the judgments of difficulty level from humans. Specifically, ten volunteers (aged between 18 and 35) participated in our experiment (we verified that further increasing the number of subjects does not significantly affect our findings). All of them had normal or corrected-to-normal vision and normal color vision. 
Our studies were approved by the THU S\&T Ethics Committee (AI), protocol THU-03-2024-0006, and obtained informed consent.
We selected six representative categories of images from the ImageNet validation set. The participants were instructed to assign a 0-to-100 score to each image according to the difficulty level of the visual recognition task built upon this image, where smaller scores mean easier. The order of different categories and the order of images within each category were both randomized for each participant. Each image was presented to a participant for 5 seconds, after which a corresponding difficulty score was recorded. There was a practice session before formal trials for the participants to get familiar with our experimental paradigm, which was identical to the formal trials in all configurations but the scores were not recorded. After the experiment, the scores of each category were normalized on a per-participant basis and averaged across participants. This human-assessed difficulty level was compared with the normalized state values predicted by AdaptiveNN, which reflect our model's judgments on the difficulty level of each visual environment.

\emph{Third}, we further developed several `visual Turing tests' \cite{lake2015human}, leveraging the straightforward human judgments to compare the visual perception behaviors of AdaptiveNN with those of humans. In these tests, real human judges tried to identify the machine, given paired examples of human and machine behaviors. Driven by the previous discussions, our `visual Turing tests' probed into both the spatial-wise visual fixation behaviors and the sample-wise visual difficulty-assessment behaviors of our model. For the former, we took the ground truth density map of human gazing centers for each image from SALICON, and sampled a sequence of three visual fixation regions, as human behaviors against the machine behaviors of the three fixations selected by AdaptiveNN. For the latter, the normalized and averaged human-assessed difficulty scores acquired as aforementioned, and the normalized state values predicted by AdaptiveNN, were rescaled to [0, 100] on a per-class basis, as human and machine behaviors, respectively. In each trial, a human judge was given two paired groups of images (three in each) in a random order, one group comprising human behaviors and the other comprising machine behaviors. The human judge was informed to identify `which group of images reflect the visual perception behaviors of a machine'. See Supplementary Data Fig. \textcolor{blue}{1-2} for the representative examples of our trials.

The full procedure of `visual Turing tests' is deatailed in Extended Data Fig. \ref{fig:visual_turing_details}a.
For each `visual Turing test' concerning the spatial-wise or sample-wise adaptive visual perception behaviors, we considered three types of trials: i) human v.s. machine, as described above; ii) human v.s. human; and iii) human v.s. random. For each trial of ii) and iii), we replaced the group of images corresponding to `machine' with samples depicting the behaviors of human vision or randomly generated behaviors, yet the participant was still told to distinguish between human v.s. machine. We established 36 trials for each of i)-iii), yielding totally 108 trials for each of the two `visual Turing tests'. These 108 trials were shuffled for every participant, such that ii) and iii) provided randomized control groups as baselines for comparison, and also offered information to validate whether our experimental setups were reasonable. 39 volunteers (aged between 18 and 40), with normal or corrected-to-normal vision and normal color vision, participated in our experiment. Our studies were approved by the THU S\&T Ethics Committee (AI), protocol THU-03-2024-0006, and obtained informed consent. We verified that further increasing the number of subjects does not significantly affect our findings.
There was a practice session before real trials. After all trials, each accuracy of i)-iii) was calculated per participant and aggregated across participants. Notably, 50\% accuracy indicates that the sort of behaviors is indistinguishable from those of humans (perfectly human-like), while 100\% suggests the inverse case.
\vskip 0.1in

\subsection{Implementation details of AdaptiveNN}
\label{sec:imple_detail}

In this section, we describe the implementation details of our method, including its inference procedure, network architectures, and training algorithms. These materials are organized for ease of understanding.

\subsubsection{Inference procedure}
\label{sec:imple_detail_inference}

\noindent \textbf{Termination criteria for the sequential perception process}.
The values of $\{\eta_1, \eta_2, \cdots\}$ reflect whether the overall quantity of available resources for visual perception, such as computation, time, or energy, is sufficient in the current circumstance. When $\{\eta_1, \eta_2, \cdots\}$ are large, AdaptiveNN generally tends to leverage relatively fewer fixations to observe various visual environments. Conversely, their small values indicate that our model can employ more fixations for visual processing on average. Notably, this only corresponds to the overall situation. In both cases, AdaptiveNN can efficiently perform uneven computation allocation across different visual environments by utilizing the predicated state values of \emph{Vision Agent}.

Incorporating these considerations of the effectiveness-efficiency trade-off, we argue that the values of $\{\eta_1, \eta_2, \cdots\}$ should be determined through maximizing the performance of AdaptiveNN under a fixed amount of total cost. Specifically, consider a set of visual environments $\mathcal{D}$ and the metrics of performance and costs, $\mathcal{P}(\cdot)$ and $\mathcal{C}(\cdot)$, which are defined with respect to $\mathcal{D}$, $\{\eta_1, \eta_2, \cdots\}$ and an AdaptiveNN model parameterized by $\bm{\theta}$. Given a budget $B>0$, $\{\eta_1, \eta_2, \cdots\}$ can be obtained by solving the optimization problem:
\vskip -0.07in

\begin{equation}
    \label{eq:solve_thres}
    \mathop{\textnormal{maximize}}_{\eta_1, \eta_2, \cdots}\ \  \mathcal{P}(\bm{\theta}, \mathcal{D}, \{\eta_1, \eta_2, \cdots\}),\quad 
    \textnormal{subject to}\ \ \mathcal{C}(\bm{\theta}, \mathcal{D}, \{\eta_1, \eta_2, \cdots\}) \leq B.
\end{equation}

\vskip 0.03in \noindent
Notably, by considering a series of varied $B$, we can collect a group of different thresholds $\{\eta_1, \eta_2, \cdots\}$ associated with the model $\bm{\theta}$. As a consequence, the cost of AdaptiveNN can be flexibly adjusted online without additional training by simply adjusting these thresholds. In our implementation, we instantiate $\mathcal{D}$, $\mathcal{P}(\cdot)$, and $\mathcal{C}(\cdot)$, as the training set of vision tasks, the accuracy, area under curve (AUC) score, or negative expected mean squared error, and the amount of computation for inference. However, the flexibility of our formulation allows these definitions to adapt to more diversified demands of various tasks, such as introducing task-specific performance metrics as $\mathcal{P}(\cdot)$, or leveraging the latency or energy consumption on given hardware devices as $\mathcal{C}(\cdot)$. Besides, we solve problem (\ref{eq:solve_thres}) using the genetic algorithm \cite{wang2021not}.

\subsubsection{Network architecture for computer vision tasks}
\label{sec:imple_detail_arch}

\textbf{Perception networks}.
In AdaptiveNN, the perception net $f_{\textnormal{rep}}$ is formulated as feature extractors with flexible architectures. In general, most existing deep learning backbones can be deployed as them. In our implementation, we mainly consider two representative examples, ResNet \cite{he2016deep} and DeiT \cite{touvron2021training}. ResNet processes input images with the alternatively stacked convolutional blocks and pooling layers, while DeiT splits images into 2D patches, each of which is embedded into a token and processed through multi-head self-attention layers and multilayer perceptron \cite{vaswani2017attention}. Both of them leverage residual connections \cite{he2016deep}. These architectures represent a wide range of popular modern deep networks for extracting embeddings from visual data. 
In visual recognition and medical diagnosis tasks, we adopt the first three network stages of ResNet or the first eight blocks of DeiT for the initial processing of the down-sampled glance inputs. We employ another full ResNet/DeiT network for processing visual fixations, due to its markedly different scales from the glance inputs. 
The internal vision representation of AdaptiveNN is fed into a task-specific head whose architecture adopts the final stage of ResNet or four DeiT blocks, in each corresponding scenario. 
In the visual search scenario, for fair comparisons with the baseline methods \cite{mnih2014recurrent, ba2014multiple}, the two perception nets consist of two and three convolutional layers, respectively, while the task-specific head is a multilayer perceptron.
\vskip 0.1in

\noindent \textbf{Vision agent}.
The \emph{Vision Agent} is defined on top of the internal vision representation $\bm{s}_t$ with the size $C\!\times\!H_{\textnormal{f}}\!\times\!W_{\textnormal{f}}$. It is formulated as the combination of a policy network $\bm{\pi}$ and a value network $V^{\bm{\pi}}$. The architecture of $V^{\bm{\pi}}$ and $\bm{\pi}$ is both a sequential composition of a $C\!\to\!C$ depth-wise convolutional layer with $3\!\times\!3$ kernels, a $C\!\to\!128$ dense convolutional layer with $1\!\times\!1$ kernels, a feature-flatten layer, and a multilayer perceptron with corresponding different output neurons. A Gaussian error linear unit (GELU) \cite{dosovitskiy2021an} is added after each intermediate convolutional or linear layer for introducing nonlinearity. The output of $V^{\bm{\pi}}$ is a scalar $V^{\bm{\pi}}(\bm{s}_t)$. The outputs of $\bm{\pi}$ parameterize a distribution $p_{\bm{\pi}}(\cdot\lvert\bm{s}_{t})$, from which we sample the location $\bm{l}_{t+1}$ of the ${(t+1)}^{\textnormal{th}}$ visual fixation $\bm{l}_{t+1}$. During training, we consider $p_{\bm{\pi}}(\cdot\lvert\bm{s}_{t})$ as a Gaussian distribution, whose mean is output by $\bm{\pi}$ and standard deviation is pre-defined as a hyperparameter. At test time, $p_{\bm{\pi}}(\cdot\lvert\bm{s}_{t})$ is set to be a Dirac delta distribution centered at the outputs of $\bm{\pi}$ for a deterministic inference process.
\vskip 0.1in

\noindent \textbf{Feature updating and reusing}.
As aforementioned, the perception network $f_{\textnormal{rep}}$ is activated on top of the $P\!\times\!P$ visual fixation $\bm{l}_{t}$. This yields local features $\bm{s}^{\textnormal{local}}_t=f_{\textnormal{rep}}(\bm{l}_{t})\!\in\!\mathbb{R}^{C\!\times\!P_{\textnormal{f}}\!\times\!P_{\textnormal{f}}}$, which is employed to update the internal vision representation $\bm{s}_{t-1}$ to obtain $\bm{s}_{t}$. For simplicity, assume that $\tilde{\bm{s}}^{\textnormal{local}}_t\!\in\!\mathbb{R}^{C\!\times\!P_{\textnormal{f}}^2}$ and $\tilde{\bm{s}}_{t-1}\!\in\!\mathbb{R}^{C\!\times\!H_{\textnormal{f}}W_{\textnormal{f}}}$ denote the flattened versions of $\bm{s}^{\textnormal{local}}_t$ and $\bm{s}_{t-1}$, respectively. Then the updating operation $\bm{s}_t = \Psi(\bm{s}_{t-1}, f_{\textnormal{rep}}(\bm{l}_{t}))$ can be expressed as
\vskip -0.07in

\begin{equation}
    \label{eq:update_s}
    \tilde{\bm{s}}_{t} = \tilde{\bm{s}}_{t-1} + \tilde{\bm{s}}^{\textnormal{local}}_t \cdot \bm{\mathrm{W}}, \quad \bm{\mathrm{W}} \in \mathbb{R}^{P_{\textnormal{f}}^2 \times H_{\textnormal{f}}W_{\textnormal{f}}}.
\end{equation}

\vskip 0.03in \noindent
We find that this simple rule is able to work reasonably well in various scenarios, where it is combined with the normalization layers and nonlinear blocks introduced by other components of AdaptiveNN. 

To construct the transformation matrix $\bm{\mathrm{W}}$, we mainly consider two principles: spatial-wise correlations and semantic-level feature importance. For the former, we constrain that the features within $\bm{s}^{\textnormal{local}}_t$ can only be utilized to update the features in $\bm{s}_{t-1}$ that are spatially close to them. Specifically, given the feature located at $i^{\textnormal{th}}$ row and $j^{\textnormal{th}}$ column of $\bm{s}^{\textnormal{local}}_t$, we can always find its corresponding coordinates $(x_{ij}, y_{ij})$ at $\bm{s}_{t-1}$, since $\bm{s}_{t-1}$ is the overall representation of the visual environment $\bm{X}$ while $\bm{s}^{\textnormal{local}}_t$ is the embeddings extracted from a region of $\bm{X}$. Suppose that $(i', j')$ denotes the coordinates of a feature in $\bm{s}_{t-1}$.  We let
\vskip -0.07in

\begin{equation}
    \label{eq:W_spatial_correlation}
    \bm{\mathrm{W}}_{(i-1)P_{\textnormal{f}} + j, (i'-1)W_{\textnormal{f}}+j'} = 0, \quad \neg (\lvert x_{ij}-i' \lvert \leq n^{\textnormal{update}}\land\lvert y_{ij}-j' \lvert \leq n^{\textnormal{update}}).
\end{equation}

\vskip 0.03in \noindent
By examining all possible values of $i, i', j, j'$, Eq. (\ref{eq:W_spatial_correlation}) ensures that each feature vector within $\bm{s}^{\textnormal{local}}_t$ can only have effects on its $(2n^{\textnormal{update}}+1)\!\times\!(2n^{\textnormal{update}}+1)$ surrounding features in $\bm{s}_{t-1}$, and hence introduces the constraints of spatial-wise correlations to Eq. (\ref{eq:update_s}). In our implementation, we simply fix $n^{\textnormal{update}}=2$. For the nonzero elements of $\bm{\mathrm{W}}$, we fill them with feature-conditional weights, aiming to model the diverse semantic-level importance of different features. We take the $k^{\textnormal{th}}$ column of $\tilde{\bm{s}}^{\textnormal{local}}_t$, $(\tilde{\bm{s}}^{\textnormal{local}}_t)_{:, k}$, and feed it into a multilayer perceptron to obtain a weight matrix $\bm{\mathrm{v}}^{k}$:
\vskip -0.07in

\begin{equation}
    \begin{split}
    \tilde{\bm{\mathrm{v}}}^{k} &= \textnormal{MLP}\left(
        (\tilde{\bm{s}}^{\textnormal{local}}_t)_{:, k}
    \right)\in\mathbb{R}^{(2n^{\textnormal{update}}+1)^2},
    \\
    \bm{\mathrm{v}}^{k} &= \textnormal{reshape}(\tilde{\bm{\mathrm{v}}}^{k})
    \in\mathbb{R}^{(2n^{\textnormal{update}}+1)\!\times\!(2n^{\textnormal{update}}+1)}.
    \end{split}
\end{equation}

\vskip 0.03in \noindent
Then, on top of Eq. (\ref{eq:W_spatial_correlation}), we further define
\vskip -0.07in

\begin{equation}
    \label{eq:semantic_importance}
    \begin{split}
    \bm{\mathrm{W}}_{(i-1)P_{\textnormal{f}} + j, (i'-1)W_{\textnormal{f}}+j'} &= \bm{\mathrm{v}}^{(i-1)P_{\textnormal{f}} + j}_{ 
    \lfloor x_{ij}-i' \rceil + n^{\textnormal{update}}+1, \lfloor y_{ij}-j' \rceil + n^{\textnormal{update}}+1}, \\
    \lvert x_{ij}-i' \lvert \leq n^{\textnormal{update}} &\land \lvert y_{ij}-j' \lvert \leq n^{\textnormal{update}}.
    \end{split}
\end{equation}

\vskip 0.03in \noindent
Combining Eq. (\ref{eq:W_spatial_correlation}) and Eq. (\ref{eq:semantic_importance}), the final form of $\bm{\mathrm{W}}$ updates $\bm{s}_{t-1}$ using the features from ${\bm{s}}^{\textnormal{local}}_t$ corresponding to the neighboring image regions of each element of $\bm{s}_{t-1}$, modeling the inherent spatial continuity of visual data. In addition, the intensity of this updating operation is flexible as it is learnable on top of each specific feature vector in ${\bm{s}}^{\textnormal{local}}_t$.

Furthermore, built upon a similar idea of leveraging spatial-wise correlations, we note that before processing the next visual fixation $\bm{l}_{t+1}$, some information pertinent to $\bm{l}_{t+1}$ have already been incorporated into the corresponding locations of ${\bm{s}}_t$, \emph{e.g.}, by previous steps of sequential perception. Therefore, we propose to accomplish the feature-extracting process of $\bm{s}^{\textnormal{local}}_{t+1}=f_{\textnormal{rep}}(\bm{l}_{t+1})$ on the basis of reusing the existing relevant information in ${\bm{s}}_t$, rather than fully from scratch. To implement this idea, we take the features from ${\bm{s}}_t$ following the same relative sizes and locations as these of $\bm{l}_{t+1}$ with respect to $\bm{X}$. Then, we feed the features into a multilayer perceptron and add the outputs to the tokens of $\bm{l}_{t+1}$ after the input layer of $f_{\textnormal{rep}}$ as learnable context embeddings. Interestingly, this design is not limited to a technique that facilitates the efficient reuse of computation. As a matter of fact, it is inspired by modeling the presaccadic attention in human vision -- the phenomenon of automatically deploying attention to the upcoming fixation location before the eyes start to move, improving visual sensitivity at the saccade target at the price of lowered perceptual sensitivity at other (non-target) locations \cite{li2021different, hanning2023dissociable}.

\subsubsection{Network architecture for embodied AI tasks}
\label{sec:imple_detail_arch_embodiedAI}

The architecture of our embodied multimodal large language models mainly follows RoboFlamingo \cite{li2024visionlanguage}. A pre-trained OpenFLamingo 3B \cite{awadalla2023openflamingo} is utilized as the backbone network. As detailed in Extended Data Fig. \ref{fig:embodiedAI_details}a, each two adjacent network blocks coupled with the shared vision encoder are employed as the perception net of AdaptiveNN. The visual tokens from both fixation and glance inputs, along with language tokens, are fed into the layers of the large language model (LLM) to extract a joint vision-language representation for the robot tasks.
The robotic policy head adopts an LSTM network followed by a multilayer perceptron \cite{li2024visionlanguage}. The architecture of the policy network and value network of AdaptiveNN's vision agent is both a multilayer perceptron.

\subsubsection{Training algorithm}
\label{sec:imple_detail_training}

\textbf{Representation learning for computer vision tasks}.
Here we describe how we carry out the representation learning of AdaptiveNN ($\mathrm{L}_{\textnormal{rep}}$ in Eq. (\ref{eq:theorem_1})) in computer vision tasks. Without loss of generality, we assume that $t_{\textnormal{o}} \sim p(t_{\textnormal{o}})$ follows a uniform distribution. Note that we can also consider more complex distributions to develop algorithms tailored for the specified demands of application scenarios, which may contribute to further performance improvements. However, we find that simple uniform distribution can already work reasonably well across various scenarios, and the results are sufficient to support our findings. Therefore, in this work, we always adopt $t_{\textnormal{o}} \sim \textnormal{unif}\{1,T\}$. This also demonstrates that our model's effectiveness does not rely on strong assumptions relevant to $p(t_{\textnormal{o}})$.

In our implementation, we augment $\mathrm{L}_{\textnormal{rep}}$ with two advanced representation learning techniques. The final loss to minimize can be written as: 
\vskip -0.07in

\begin{equation}
    \label{eq:final_rep_loss}
    \mathrm{L}_{\textnormal{rep}} + \alpha\mathcal{L}_{\textnormal{regularization}}(y, f_{\textnormal{rep}}(\bm{X}_{\textnormal{d}})) + \beta\sum_{t=1}^{T-1}\mathcal{L}_{\textnormal{self-distillation}}(q_{T}, q_{t}),
\end{equation}

\vskip 0.03in \noindent
where $\alpha, \beta$ are coefficients, and we simply fix $\alpha=2, \beta = 1$ in this article. The second term $\mathcal{L}_{\textnormal{regularization}}$ in Eq. (\ref{eq:final_rep_loss}) is a regularization loss for the perception network $f_{\textnormal{rep}}$ that processes visual fixations \cite{wang2022adafocus}. We feed the down-sampled version $\bm{X}_{\textnormal{d}}$ of $\bm{X}$ into $f_{\textnormal{rep}}$ and obtain a direct loss using its outputs and the label $y$. This regularization technique addresses the slow convergence issue of $f_{\textnormal{rep}}$ caused by only seeing visual fixations $\{\bm{l}_{1}, \bm{l}_{2}, \ldots\}$ during training, at the price of slightly increasing the train-test discrepancy of the inputs of $f_{\textnormal{rep}}$.
The third term $\mathcal{L}_{\textnormal{self-distillation}}$ in Eq. (\ref{eq:final_rep_loss}) corresponds to a self-distillation technique \cite{han2023dynamic}, where $q_{t}=q(\bm{\theta}, \bm{X}, \bm{l}_{1:t})$ denotes the outputs of the model at $t^{\textnormal{th}}$ step. This technique leverages the final outputs $q_{T}$ of AdaptiveNN to guide the learning of intermediate outputs $q_{1}, \ldots, q_{T-1}$. It improves the performance of AdaptiveNN equipped with a smaller number of fixations, while only introducing negligible additional training costs.
\vskip 0.1in

\noindent \textbf{Reinforcement learning for computer vision tasks}.
Similar to representation learning, we assume $t_{\textnormal{o}} \sim \textnormal{unif}\{1,T\}$. The off-the-shelf proximal policy optimization algorithm (PPO) \cite{schulman2017proximal} using generalized advantage estimation \cite{Schulmanetal_ICLR2016} is deployed to accomplish the reinforcement learning procedure. 
\vskip 0.1in

\noindent \textbf{Training algorithms for embodied AI tasks}.
The training of the embodied multimodal large language models basically follows RoboFlamingo \cite{li2024visionlanguage}. For AdaptiveNN, we simply adopt Eq. (\ref{eq:theorem_1}) as the training objective, which works reasonably well. Other implementation details are the same as computer vision tasks.
\vskip 0.1in

\subsection*{Data availability}
Most data used in this study are publicly available, including ImageNet \cite{russakovsky2015imagenet} (\url{https://www.image-net.org/}), CUB-200-2011 \cite{wah2011caltech} (\url{https://www.vision.caltech.edu/datasets/cub_200_2011/}), NABirds \cite{van2015building} (\url{https://dl.allaboutbirds.org/nabirds}), Oxford-IIIT Pet \cite{parkhi2012cats} (\url{https://www.robots.ox.ac.uk/~vgg/data/pets/}), Stanford Dogs \cite{khosla2011novel} (\url{https://paperswithcode.com/dataset/stanford-dogs}), Stanford Cars \cite{krause20133d} (\url{https://paperswithcode.com/dataset/stanford-cars}), FGVC-Aircraft \cite{maji2013fine} (\url{https://www.robots.ox.ac.uk/~vgg/data/fgvc-aircraft/}), STSD \cite{larsson2011using} (\url{https://www.cvl.isy.liu.se/research/datasets/traffic-signs-dataset/}), MNIST \cite{lecun1998mnist} (\url{https://paperswithcode.com/dataset/mnist}), RSNA pneumonia detection \cite{shih2019augmenting} (\url{https://www.rsna.org/rsnai/ai-image-challenge/rsna-pneumonia-detection-challenge-2018}), CALVIN \cite{mees2022calvin} (\url{https://github.com/mees/calvin}), SALICON \cite{jiang2015salicon} (\url{http://salicon.net}), and MIT1003 \cite{mit1003} (\url{https://saliency.tuebingen.ai/}). 
A minimum dataset for our `visual Turing tests' is provided in Supplementary Fig. \textcolor{blue}{12-13}.

\subsection*{Code availability}
Implementation code is available at \url{https://github.com/LeapLabTHU/AdaptiveNN} \cite{yang_yue_le_yang_2025_16810996}.


\subsection*{Acknowledgements}
G.H. is supported by the National Key R\&D Program of China under Grant 2024YFB4708200, the National Natural Science Foundation of China under Grants U24B20173 and 62276150, and the Scientific Research Innovation Capability Support Project for Young Faculty under Grant ZYGXQNJSKYCXNLZCXM-I20.
S.S. is supported by the National Natural Science Foundation of China under Grant  42327901.
We thank S. Zhang, M. Yao and Y. Wu for helpful discussions and comments on an earlier version of this paper. 

\subsection*{Author contributions}
G.H. and S.S. initiated and supervised the project. 
Y.W., Y.Y. (le-y22@mails.tsinghua.edu.cn), Y.Y. (yueyang22@mails.tsinghua.edu.cn) and G.H. contributed to the conception and design of the work.
Y.W., Y.Y. (le-y22@mails.tsinghua.edu.cn), Y.Y. (yueyang22@mails.tsinghua.edu.cn), H.W., H.J., Y.H. and Z.N. contributed to the technical implementation.
Y.W., Y.Y. (le-y22@mails.tsinghua.edu.cn), Y.P., M.S., R.L. and Q.Y. contributed to the data acquisition and organization.
Y.W., Y.Y. (le-y22@mails.tsinghua.edu.cn), Y.Y. (yueyang22@mails.tsinghua.edu.cn), H.W., A.Z. and Z.X. analyzed the results.
All authors contributed to the drafting and revising of the manuscript.

\subsection*{Competing interests}
The authors declare no competing interests.

\makeatletter
\let\ICML@old@fnum@figure\fnum@figure
\let\ICML@old@makecaption\@makecaption
\renewcommand{\fnum@figure}{Extended Data Fig.\ \thefigure}

\setcounter{figure}{0}
\renewcommand{\thefigure}{A\arabic{figure}}
\renewcommand{\theHfigure}{A\arabic{figure}}

\crefname{figure}{Extended Data Fig.}{Extended Data Figs.}
\Crefname{figure}{Extended Data Fig.}{Extended Data Figs.}
\renewcommand{\figureautorefname}{Extended Data Fig.}

\clearpage

\begin{figure*}[!h]
    \centering
    \vskip 0.6in
        \includegraphics[width=1\textwidth]{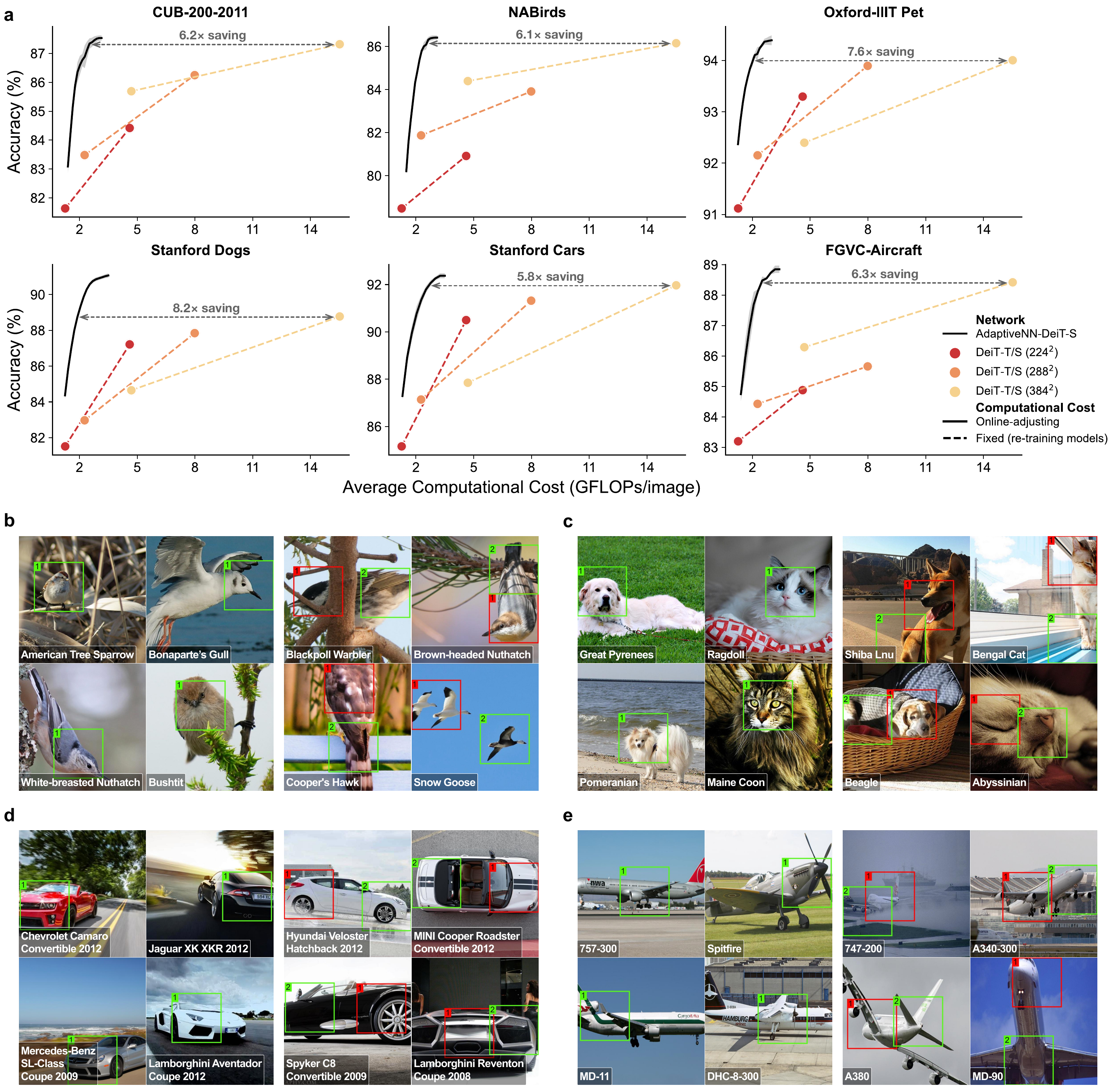}
    \caption{
        \textbf{Results on six fine-grained visual recognition benchmarks.}
        \textbf{(a)} Quantitative comparisons of AdaptiveNN and conventional non-adaptive models: Top-1 validation accuracy versus average computational cost for inferring the model. Datasets: CUB-200-2011 \cite{wah2011caltech}, NABirds \cite{van2015building}, Oxford-IIIT Pet \cite{parkhi2012cats}, Stanford Dogs \cite{khosla2011novel}, Stanford Cars \cite{krause20133d}, FGVC-Aircraft \cite{maji2013fine}. Error bars show the standard deviations of five independent trials with different random seeds. Non-adaptive models with varying costs are obtained by modifying model sizes and input resolutions. Here we set the maximum fixation number to be two, which is generally sufficient to accomplish the recognition tasks.
        \textbf{(b-e)} Qualitative evaluation of the visual fixations chosen by AdaptiveNN-DeiT-S across four datasets: CUB-200-2011, Oxford-IIIT Pet, Stanford Cars, and FGVC-Aircraft. The visualizations adhere to the setups established in Fig. \ref{fig:fig2}a.
        \label{fig:fig3}
        }
    \vskip -0.2in
\end{figure*}

\begin{figure*}[!h]
    \centering
    \includegraphics[width=1\textwidth]{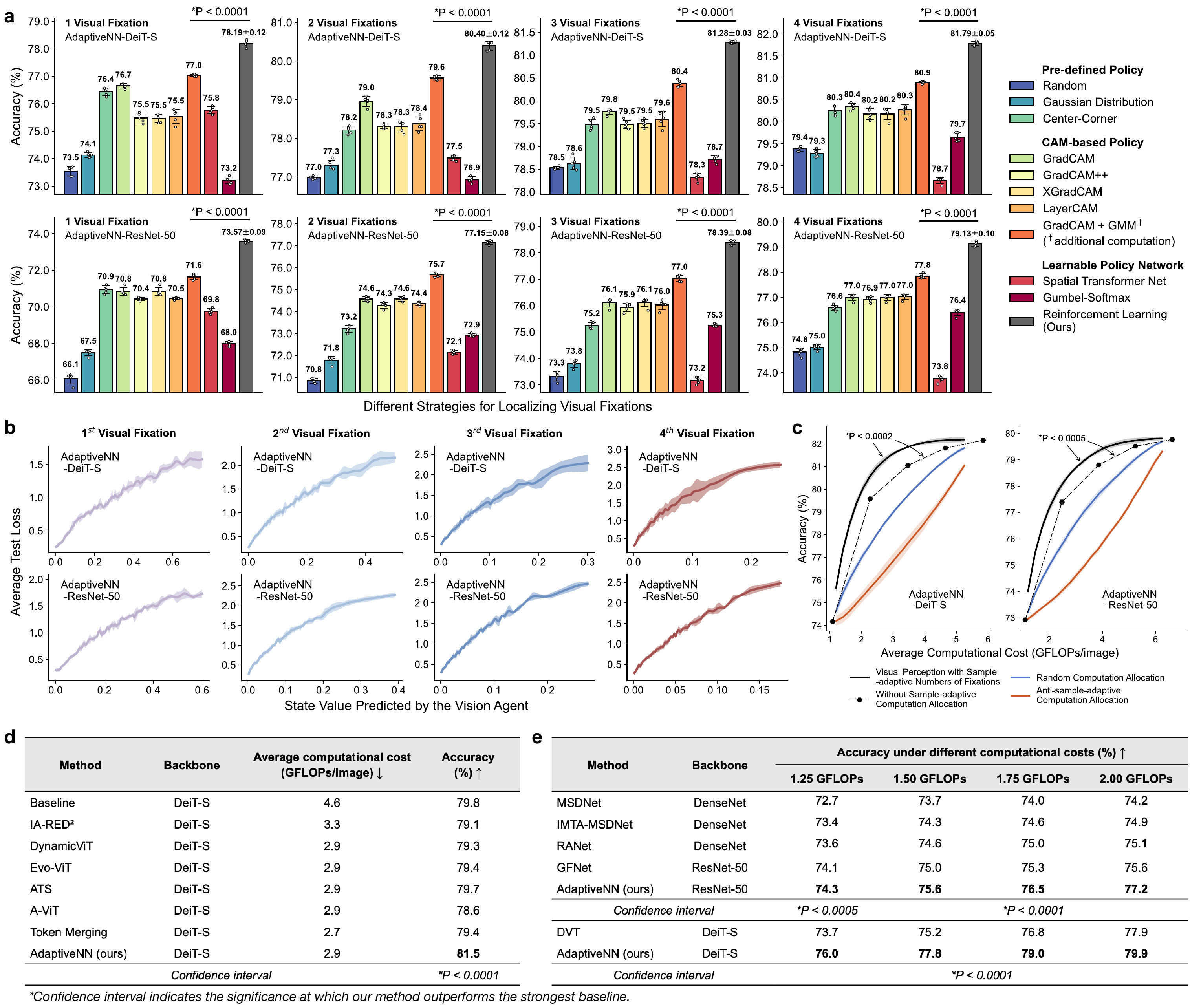}
    \caption{
        \textbf{Investigation and ablation studies of the design principles of AdaptiveNN.} 
        All the results are reported on ImageNet. See Supplementary Section \textcolor{blue}{B} for the details of comparative baselines.
        \textbf{(a)} Efficacy of different methodologies for establishing the fixation localization strategy within AdaptiveNN. For a clean comparison, we train a classifier using only the features from visual fixations, and assume all samples use the same number of fixations, such that the resulting validation accuracy serves as a well-controlled measure to assess the effectiveness of each variant. 
        Moreover, we consider an extensive variety of baselines for comparison, including selecting fixations using i) pre-defined rules; ii) class activation maps (CAMs); iii) CAMs augmented with a Gaussian mixture model (GMM); and iv) policy networks learned using other algorithms.
        \textbf{(b)} Average test loss corresponding to the validation data with different state values predicted by the \emph{Vision Agent} in AdaptiveNN. We examine the state values taken from every step of sequential perception processes.
        \textbf{(c)} Comparisons of different termination criteria for concluding the sequential perception process of AdaptiveNN. The term `anti-' refers to the inverse of our proposed method (detailed in Section \ref{sec:imple_detail_inference}), namely terminating the observation process for samples with relatively higher state values.
        \textbf{(d-e)} Comparisons with representative methodologies designed to improve deep learning models' computational efficiency. Specifically, \textbf{(d)} evaluates against baselines that leverage spatial redundancy in visual data. \textbf{(e)} examines models with multi-exit architectures that allow for online computational cost adjustments.
        *Independent samples $t$-test. Error bars show the standard deviations of five independent trials with different random seeds.
        \label{fig:fig6}
        }
\end{figure*}

\begin{figure*}[!h]
    \centering
        \includegraphics[width=\textwidth]{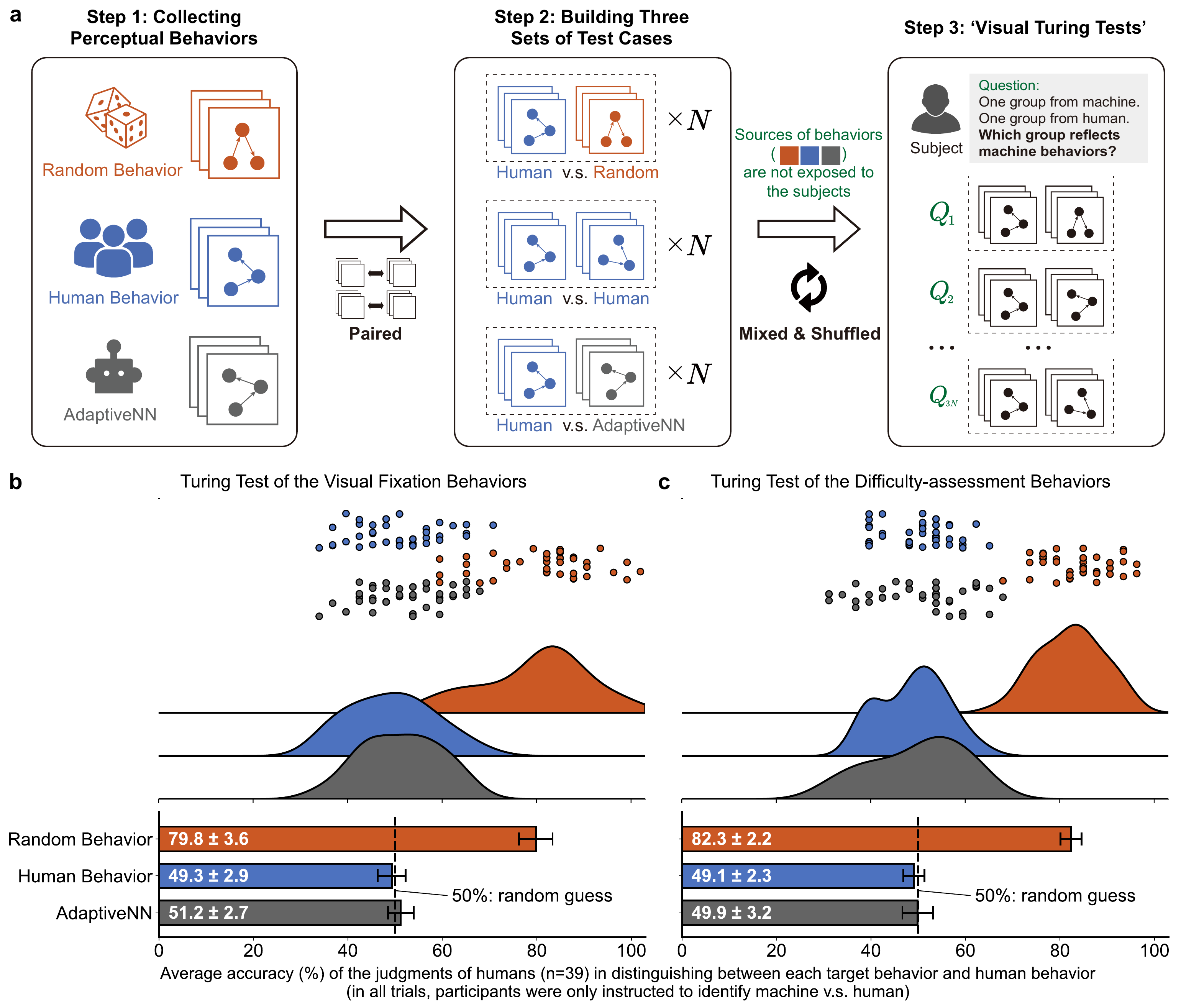}
    \caption{
        \textbf{Details of `visual Turing tests'.}
        \textbf{(a)} The full procedure of `visual Turing tests'. We first collect the visual perception behaviors from real humans, machine (AdaptiveNN), and random generation. Then, we construct multiple trials, each including paired examples of perceptual behaviors. We consider three types of trials: i) human v.s. machine; ii) human v.s. human; and iii) human v.s. random, each corresponding to $N$ trials (we use $N$=36), yielding totally $3N$ trials for each `visual Turing test'. Finally, these $3N$ trials are mixed and shuffled for every human judge ($n$=39). The participants are only instructed to identify the machine behaviors within each trial (for all i)-iii)). Each accuracy of i)-iii) is calculated per participant and aggregated across participants. As a result, i) offers the Turing test results, while ii) and iii) provide randomized control groups as baselines and also validate whether our experimental setups are reasonable.
        \textbf{(b)} Results of the two `visual Turing tests'. Each data point represents the average identification accuracy of a human judge. 
        Bars show the mean accuracy across human judges and the corresponding 95\% confidence interval. Ideal performance is 50\%, where the machine is indistinguishable from human behaviors in these binary choice tasks.
        \label{fig:visual_turing_details}
        }
\end{figure*}

\begin{figure*}[!h]
    \centering
        \includegraphics[width=\textwidth]{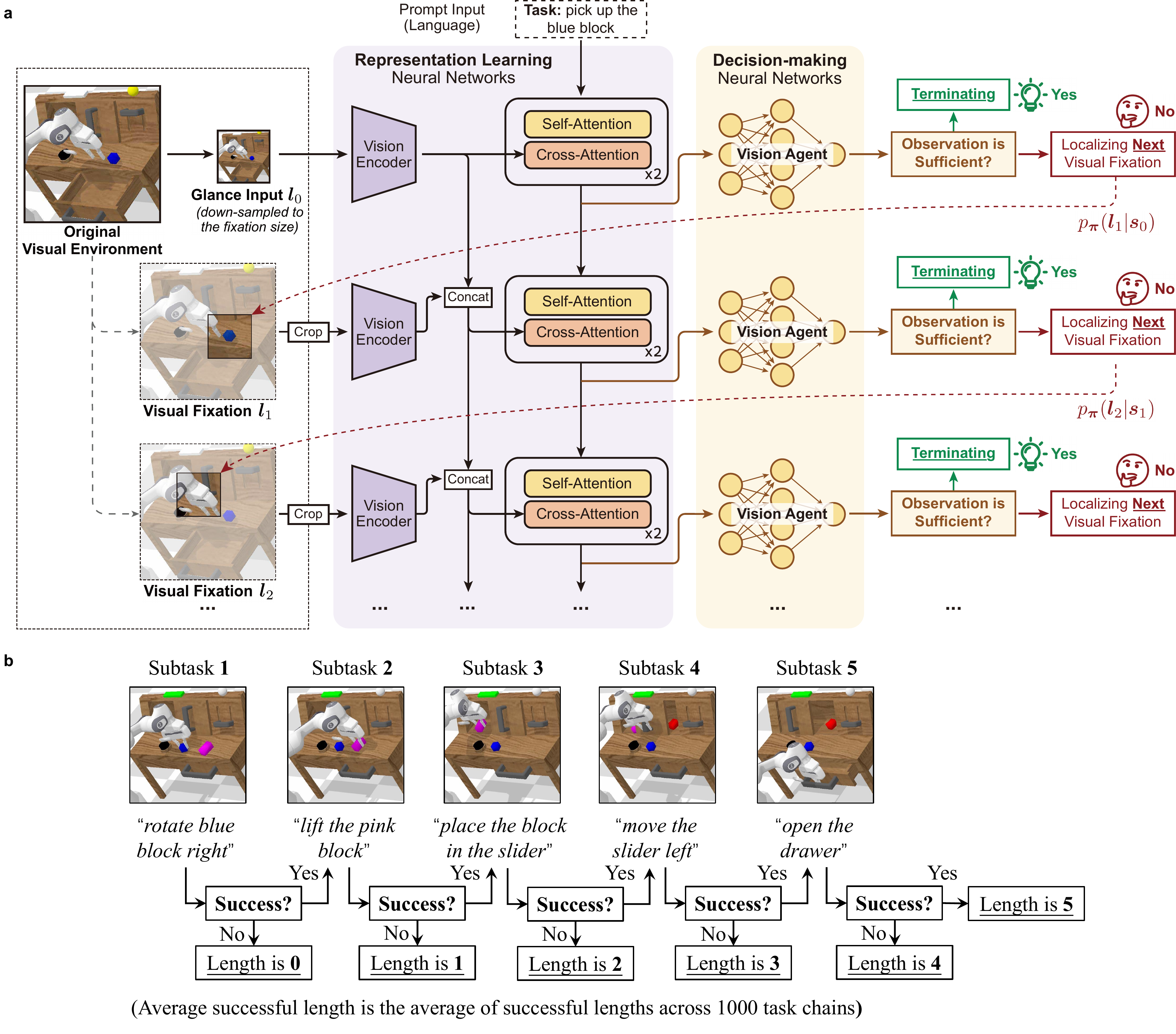}
    \caption{
        \textbf{Details of the experiments based on embodied multimodal large language models (MLLM).}
        \textbf{(a)} The network architecture and inference procedure of the AdaptiveNN-based embodied MLLM, which mainly follows RoboFlamingo \cite{li2024visionlanguage}. The backbone network is based on a pre-trained OpenFLamingo 3B \cite{awadalla2023openflamingo}. Each two adjacent network blocks coupled with the shared vision encoder are employed as the perception net of AdaptiveNN.
        \textbf{(b)} The metric employed in our experiments on CALVIN. The model performance is quantified as the average successful length (0 to 5) across 1000 5-task sequences.
        \label{fig:embodiedAI_details}
        }
\end{figure*}

\clearpage

\bibliography{main}
\bibliographystyle{unsrt}

\clearpage



\section*{Supplementary material}

# Supplementary materials for 'Emulating Human-like Adaptive Vision for Efficient, Flexible, and Interpretable Machine Visual Perception'

# A  Related works

To position our novel contributions in the state-of-the-art, here we provide a systematic comparison between our work and relevant existing literature.

## A.1  Computationally efficient computer vision backbones

Modern artificial vision models are typically built upon deep neural networks, such as convolutional neural networks [1, 2] and vision Transformers [3, 4]. Innovations in network architectures and learning algorithms have enabled these models to approach or even surpass human-level performance in various vision tasks [5–9]. However, models with state-of-the-art performance usually come at the price of a computationally intensive inference procedure. In many realistic scenarios such as wearable devices, mobile phones, robotics, embedded devices, and autonomous vehicles, computation directly translates into power consumption, latency, and carbon emissions, and all of them should be minimized under environmental, safety or economic considerations [10–16].

To mitigate these issues, numerous recent research efforts have been directed toward developing lightweight backbone networks, such as MobileNets [17–20], CondenseNets [21, 22], ShuffleNets [23, 24], EfficientNets [25, 26], GhostNets [27–29], Mobile-former [30], and EfficientFormers [31, 32]. These models are designed to maintain performance while reducing the computational inference cost. Moreover, motivated by that deep networks may contain a considerable number of redundant parameters, many approaches propose to prune less useful weights [33–37] or quantize the weights [38–40]. Another important technique is knowledge distillation, which trains smaller models to mirror the behaviors of larger models [41, 42].

In general, our work is orthogonal to these existing methods. The advanced architectures of backbone networks can be conveniently deployed as the the

feature-extraction modules of AdaptiveNN, while our model can be further enhanced by leveraging techniques such as pruning, quantization, and knowledge distillation.

## A.2 Dynamic neural networks

This article is relevant to the dynamic neural networks that allocate computational resources unevenly across different samples or different spatial locations [43]. However, our work differentiates itself from the existing literature in this direction in several critical aspects. *First*, we attain a natural and elegant combination of sample-wise and spatial-wise adaptive computation by mimicking human visual systems. Our model not only enhances computational efficiency but also improves the model's behavioral adaptability and interpretability. In contrast, most existing works only focus on a single sort of adaptive computation and do not secure a biologically plausible motivation, with only one or two out of the three aspects of our gains observed. *Second*, for the first time, we develop theoretical analyses concerning the learning principles of dynamic neural networks, revealing the natural emergence of reinforcement learning rules under the goal of optimizing perceptual behavior distributions for an arbitrary vision task. *Third*, to our best knowledge, we present the first work seeking to build up a bridge between dynamic neural networks and human visual cognition. This includes establishing methodologies for designing and learning a human-like deep learning framework, systematically comparing its behaviors with those of human visual systems, and demonstrating its potential for probing into human behavioral and learning processes.

In the following, we briefly review the current works on dynamic neural networks to further underscore the innovative aspects of our work.

### A.2.1 Sample-wise dynamic neural networks

Some recent works improve the efficiency of deep networks by customizing the architectures or parameters of the model conditioned on individual inputs. For example, multi-scale DenseNet (MSDNet) [44] and its improved version [45] introduce a multi-scale architecture equipped with multiple classifiers, allowing for utilizing small networks for easier samples and switching to large networks for more challenging ones. Similarly, dynamic vision Transformer (DVT) [46] processes each image with adaptive token numbers by cascading multiple vision Transformers and reusing computation. Dynamic perceiver [47] introduces an addition attention-based pathway that integrates and selectively exits from intermediate network representations based on the task demands. VideoIQ [48] dynamically adjusts the precision of its models based on the importance of each input. Additionally, several approaches have been developed that ensemble multiple models, selectively activating only a subset of

architectures [49–51] or dynamically bypassing unnecessary layers [52–54] or channels [55]. Beyond architectures, models like RANet [56], DR-Net [57], and AR-Net [58] process different inputs with varying amounts of computation by changing their resolutions.

### A.2.2 Spatial-wise dynamic neural networks

It has been observed that not all spatial regions within visual data hold equal relevance for specific tasks [59, 60]. Motivated by this phenomenon, some approaches have been developed to improve computational efficiency by allocating computation unevenly across different spatial regions according to their relative importance. One such technique, the spatially adaptive computation time (SACT) algorithm [61] dynamically modulates the number of executed network layers conditioned on different image regions. The algorithms proposed in [62] and [63] selectively bypass less critical regions of feature maps. In the context of vision Transformers, many works propose to discard or merge less important visual tokens progressively during the inference procedure of the model [64–68], aiming to reduce less necessary computation.

Furthermore, some methods autonomously identify and focus on the most informative details of the input data. These works tend to mainly consider improving performance, with less attention paid to efficiency. For example, RA-CNN [69], TASN [70], and NTS-Net [71] enhance the classification of fine-grained visual categories by concentrating on discriminative local patterns, such as the afterbrain colors of birds, which are critical for distinguishing subtle inter-species differences. S3N [72] and MGE-CNN [73] employ the class activation maps (CAMs) [74, 75] to isolate and emphasize the visual cues essential for fine-grained image recognition. Similar techniques are also explored in diverse applications including vehicle re-identification [76], image translation [77] and zero-shot recognition [78–80].

More recently, some works focus on augmenting large vision-language models with the capabilities to selectively attend to the task-relative regions within visual data [81–84]. Against this context, the innovative aspects of our work discussed at the beginning of this section still hold. In addition to them, however, our model is also promising in that it does not necessitate any specialized supervision signal to teach the model 'where to look', while these methods or some of their components usually rely on tailored training data with such localization supervision. This attribute potentially enhances the scalability of our approach, making it more adaptable to varied and less constrained circumstances. Moreover, given the flexibility of our framework and its learning algorithm, we believe AdaptiveNN is generally compatible with the vision tasks defined on top of large vision-language models, which may be an interesting topic for future work.

### A.3 Visual attention models

Attention constitutes one of the fundamental mechanisms of human vision: people sequentially fixate on regions of the visual environment that are relevant for the task of interest [85–89]. In the context of integrating this mechanism into computer vision models and accomplishing vision tasks, some works propose to predict and leverage soft attention masks tailored for each input images [90–92]. These methods mainly focus on modeling the results, or guidance, of human visual attention instead of its procedure. For example, an attention mask is usually generated through a simultaneous and computationally equivalent processing of all relevant image regions. Moreover, it may be costly to obtain proper attention masks on top of relatively large-scale inputs. Recent advances in this direction have seen the direct incorporation of these soft-attention mechanisms into the micro-architectures of backbone networks, often as modules that are based on local features with reduced sizes [3, 93–96]. To this end, our work is generally orthogonal to these efforts.

Another line of works related to us are visual hard attention models. For example, three representative methods, restricted Boltzmann machine (RBM) [97], recurrent attention model (RAM) [98], and deep recurrent attention model (DRAM) [99], share a similar spirit to AdaptiveNN in perceiving visual scenes through recurrent attention. However, they tend to exhibit limited scalability in scenarios beyond simple tasks like tiny digit classification [99, 100]. This problem is moderately alleviated by some recent research [100–103]. Nevertheless, our work considerably outperforms them in constructing a general, flexible, and scalable framework that emulates human adaptive vision, leading to substantial enhancements in efficiency, adaptability, and interpretability. Furthermore, our work is novel in developing theoretical analyses to reveal the emergence of broadly applicable reinforcement learning principles that do not rely on specialized supervision signals. Meanwhile, for the first time in the modern community of modeling visual hard attention, we comprehensively compare the perceptual behaviors of our vision models with those of humans, and discuss the notable potential of our work for probing human visual cognition.

## B  Baseline approaches for comparison

In this article, we compare AdaptiveNN against a broad range of baseline methods, aiming to comprehensively demonstrate the characteristics of our model. Here we briefly introduce these comparative baselines and describe how they are properly configured in our experiments.

Unless otherwise specified, AdaptiveNN is mainly compared against traditional, non-adaptive deep networks on top of the same backbone models.

According to the common practice of the deep learning community [1–4, 95], these baselines process all regions within visual inputs with an equivalent amount of computation, in a single step, and the baselines with different inference costs are obtained by modifying the model size and input resolution. It is important to note that AdaptiveNN does not leverage any advanced architectural innovations, nor does it employ additional data augmentation or regularization strategies that are not already parts of the baseline models. The only difference between AdaptiveNN and the baselines lies in the introduction of mechanisms designed to emulate the adaptive visual perception processes of humans. These experimental setups contribute to a focused and isolated evaluation of the benefits derived from our core contributions.

In Fig. 4c, we compare our model with two representative algorithms designed for learning to emulate human adaptive vision, recurrent attention model (RAM) [98] and deep recurrent attention model (DRAM) [99]. To ensure a fair comparison, AdaptiveNN is configured using neural network components analogous to those used in RAM and DRAM: simple convolutional layers combined with multilayer perceptrons (see Section 5.3.2 for details). The model sizes of AdaptiveNN, RAM, and DRAM are approximately equivalent (in fact, we find that further enlarging RAM or DRAM only results in minor improvements). Moreover, to ensure reproducing their work reasonably, we adopt the default hyper-parameters and training configurations recommended by their original papers as the starting point, and perform the same hyper-parameter search procedure as our model.

In Fig. 6a and Extended Data Fig. A2a, the fixation selection strategy learned by AdaptiveNN is compared against four categories of alternative design choices, including selecting fixations using i) non-adaptive pre-defined rules, as foundational baselines; ii) class activation maps (CAMs), which are widely acknowledged as a useful tool to localize the regions based on which neural networks make decisions [9, 104]; iii) CAMs augmented with a Gaussian mixture model (GMM), aimed at maximizing the utility of CAMs regardless of the computational cost; and iv) policy networks learned using other algorithms, to assess whether reinforcement learning is a proper algorithm for optimizing fixation selection strategies as we discuss in theory. For i), we examine random/Gaussian: sampling fixation locations following uniform or Gaussian distributions; and center-corner: first fixating on the center of images, and then traversing the corners. For ii), we consider several representative examples of widely used CAM algorithms, including GradCAM [75], GradCAM++ [105], XGradCAM [106], and LayerCAM [107]. For fair comparisons between them and AdaptiveNN, we solve CAMs utilizing the first internal vision representation $s_0$ of our model, derived from the initial quick glance (nevertheless, we find that solving CAMs using larger models or inputs does not notably influence

the results). Then we sample subsequent visual fixations based on the intensity values of CAMs. For iii), we further fit CAMs with a Gaussian mixture model [108], selecting the values of means of the top-$n$ principal components as the fixation locations. Note that GMM introduces additional computation, and this approach represents an estimate of the 'upper bound' for the performance of CAM-based methods. For iv), we consider two alternative choices of reinforcement learning: spatial transformer networks [109] and Gumbel-Softmax [110]. They approximate the non-differentiable fixation selection problem with interpolation and the Gumbel trick, respectively, such that the model can be trained in end-to-end.

In Extended Data Fig. A2d, A2e, we present a system-level comparison of AdaptiveNN against representative state-of-the-art methods designed to improve the computational efficiency of deep learning models. We mainly consider two categories of baselines relevant to our model, corresponding to Extended Data Fig. A2d and A2e, respectively. Extended Data Fig. A2d focuses on recently proposed algorithms leveraging the inherent spatial redundancy within visual data, including IA-RED$^2$ [111], DynamicViT [64], Evo-ViT [65], ATS [66], A-ViT [67], and Token Merging [68]. Extended Data Fig. A2e considers existing multi-exit models that can adjust their computational cost online in a similar way to our model, including MSDNet [44], IMTA-MSDNet [45], RANet [56], GFNet [102, 112], and DVT [46].

# C  Training hyper-parameters and configurations

Here we detail the specific configurations and hyper-parameters for training our AdaptiveNN models. For all datasets, we held out 20% of training data to perform a hyper-parameter search, and then put this data back to the training set, reporting final results.

On ImageNet, we train AdaptiveNN-DeiT-S and AdaptiveNN-ResNet-50 mainly utilizing the training pipelines proposed in the original papers of DeiT [4] and ResNet [1]. Our basic settings are summarized in Extended Data Tab. 1. Notably, all the baseline models adopt the same training configurations as AdaptiveNN to ensure fair comparisons. For the proximal policy optimization (PPO) algorithm [113] we employ to solve the reinforcement learning problem, we set $\gamma=0.5$, $\epsilon=0.2$, $c_1=1$, and $c_2=0$. We set $\lambda=0.84$ for generalized advantage estimation [114]. For fine-grained classification tasks and the medical diagnosis task, we fine-tune the models pre-trained on ImageNet, where we reduce the batch size and initial learning rate, and increase the stochastic depth rate, to alleviate overfitting.

| Training configs | AdaptiveNN-DeiT-S | AdaptiveNN-ResNet-50 |
|---|---|---|
| Data pre-processing | Random resized crop [1, 2] | |
| Size of images | $288^2$ | $320^2$ |
| Glance size | $112^2$ | |
| Size of visual fixations | $112^2$ | |
| Optimizer | AdamW | SGD |
| Optimizer hyper-parameters | $\beta_1, \beta_2$=0.9, 0.999 | momentum=0.9 |
| Initial learning rate | 4e-3 | 0.4 |
| Learning rate schedule | Cosine annealing | |
| Weight decay | 0.05 | 1e-4 |
| Batch size | 4,096 | 1,024 |
| Training epochs | 300 | 100 |
| Warmup epochs | 20 | 5 |
| Warmup schedule | Linear | |
| RandAug [115] | (9, 0.5) | – |
| Mixup [116] | 0.8 | – |
| Cutmix [117] | 1.0 | – |
| Random erasing [118] | 0.25 | – |
| Label smoothing [119] | 0.1 | |
| Stochastic depth [120] | 0.1 | – |
| Gradient clipping | 5.0 | – |

**Supplementary Data Tab. 1** Hyper-parameters and configurations for training AdaptiveNN models on ImageNet-1K.

For the traffic sign recognition task on STSD, we utilize an Adam optimizer with $\beta_1$, $\beta_2$=0.9, 0.999, and adopt a batch size of 32. We set the initial learning rate at 0.001 with a cosine annealing schedule, and implement an L2 weight decay coefficient of $10^{-5}$. The two backbone networks within our model is initialized from ImageNet pre-trained checkpoints. Training data is augmented with Mixup [116]. Additionally, due to the long-tailed nature of the dataset, class-balanced loss reweighting is utilized for both representation learning objectives and the rewards in reinforcement learning. The original high-resolution image (960×1280) are down-sampled to $192^2$ as the glance input, while the vision agent localizes $192^2$ visual fixation sequences with a maximum length of 2. We set the discount factor $\gamma$ as 0.2.

For the visual search tasks, we utilize an SGD optimizer with a momentum of 0.9, and adopt a batch size of 1024. We set the initial learning rate at 0.006 for the vision agent and 0.01 for other components, with a cosine annealing schedule. The model is trained for 200 epochs, with a 10-epoch warmup. We implement an L2 weight decay coefficient of $10^{-4}$. For retrieving the locations of digits, we view it as a classification problem over numerous candidates, employ a multilayer perceptron as the classification head for each target digit, and adopt a dropout rate of 0.7 in the classification head to alleviate overfitting. The input images, sized at $224^2$, are down-sampled to $42^2$ as the glance input, while the vision agent localizes $28^2$ visual fixation sequences with a maximum length of 5. Here only the features of visual fixations are used for visual search since glance inputs tend to provide little useful information for identifying the

small digits. We set the discount factor $\gamma$ as 1.0 to assign more attention to the final results incurred by the current action, rather than focusing on instant results (as our objective is to accomplish the task successfully). Additionally, to encourage the policy to focus on digits, we add a penalty term to the reward; specifically, if the visual fixation region selected by the current action contains any non-black pixel, the penalty is zero. Otherwise, a value of -1.5 is added to the current reward. Note that both baselines and our method are equipped with this technique.

# D Supplementary figures and tables

Human or Machine?

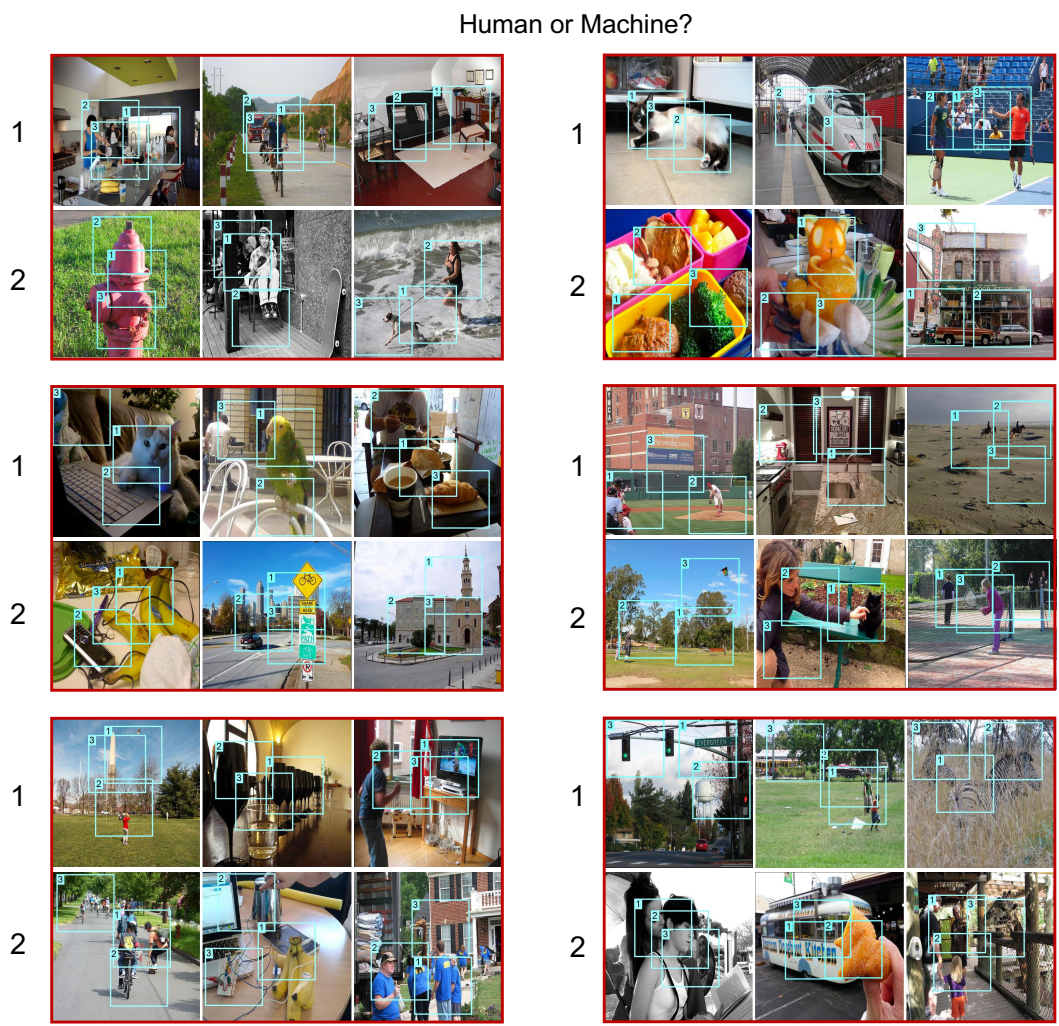

**Supplementary Data Fig. 1** Randomly selected examples of the 'visual Turing test' on the spatial-wise visual fixation behaviors of AdaptiveNN and humans. The three-sample groups in each pair that are obtained with a machine are (by row) 2, 2; 1, 1; 2, 1.

Human or Machine?

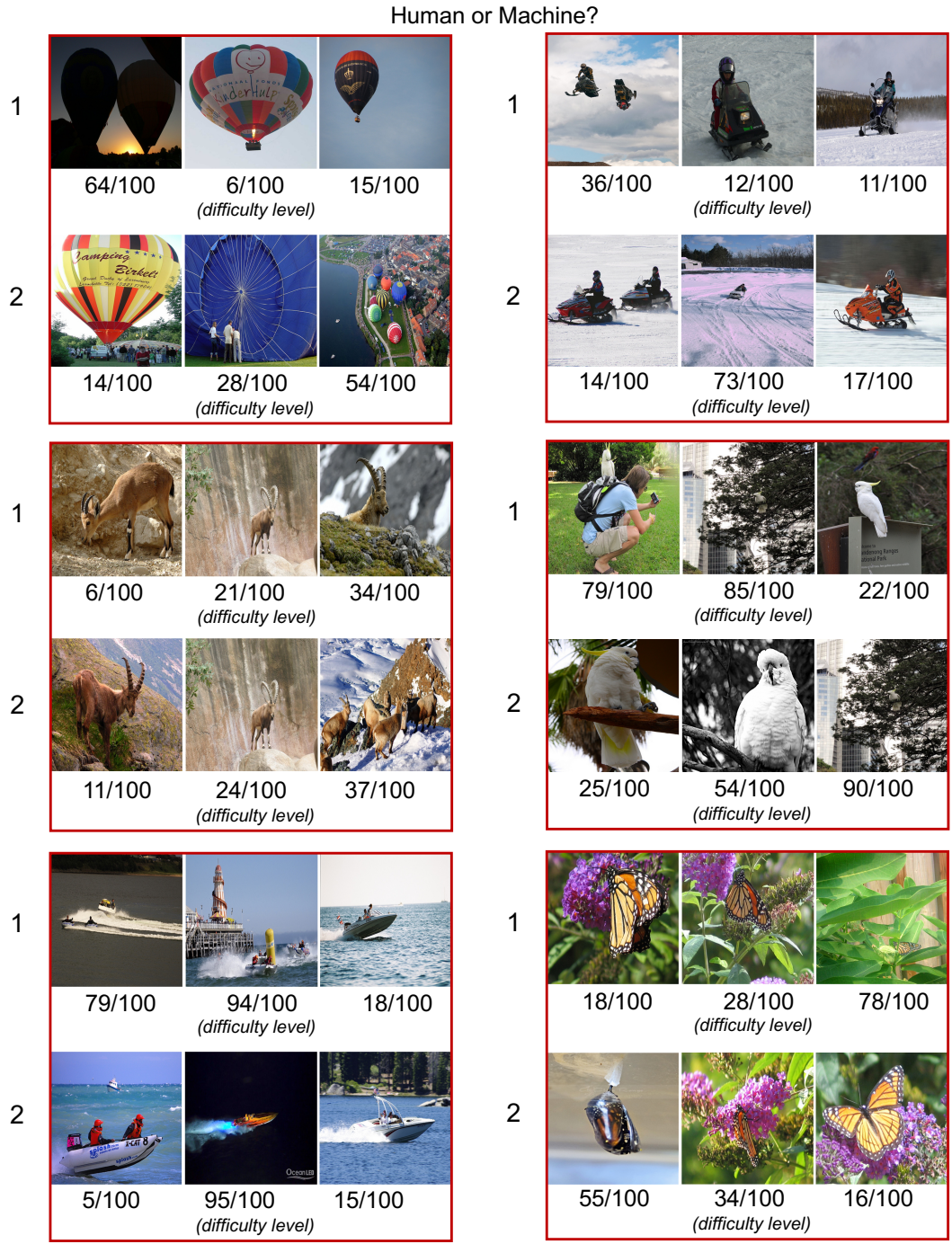

**Supplementary Data Fig. 2** Randomly selected examples of the 'visual Turing test' on the sample-wise visual difficulty assessment behaviors of AdaptiveNN and humans. The three-sample groups in each pair that are obtained with a machine are (by row) 2, 1; 2, 1; 2, 2.

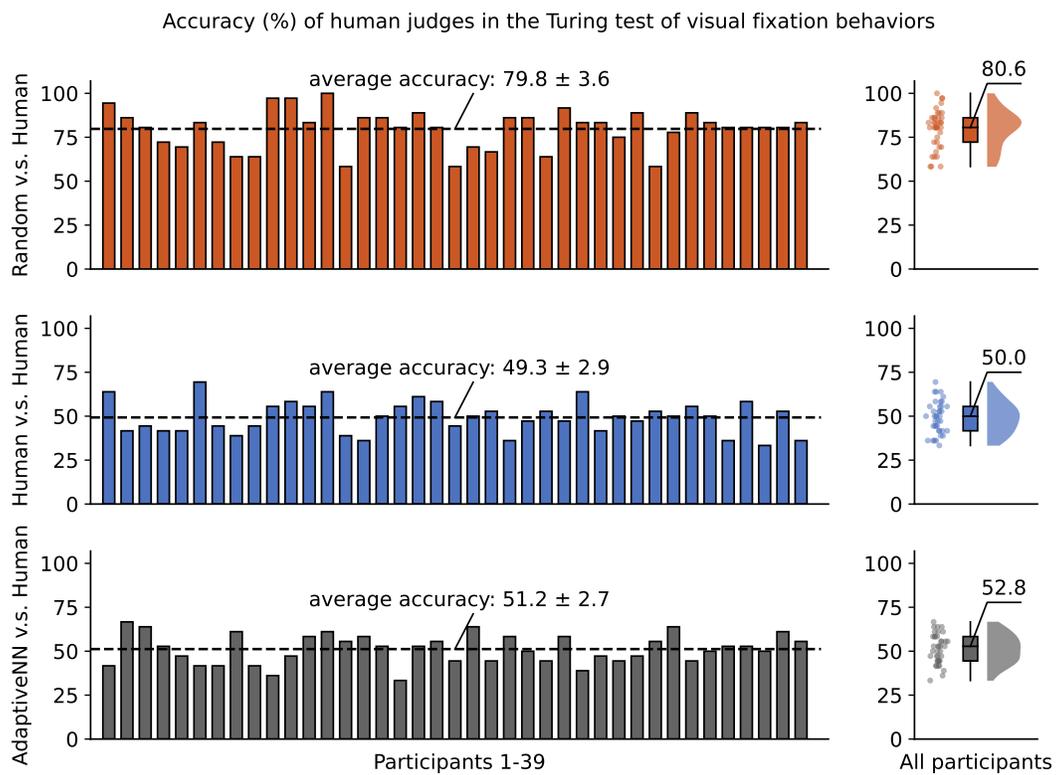

**Supplementary Data Fig. 3** Per-participant results of the 'visual Turing test' on the spatial-wise visual fixation behaviors. '±' indicates the 95% confidence interval. The distribution across all participants is shown on the right.

12

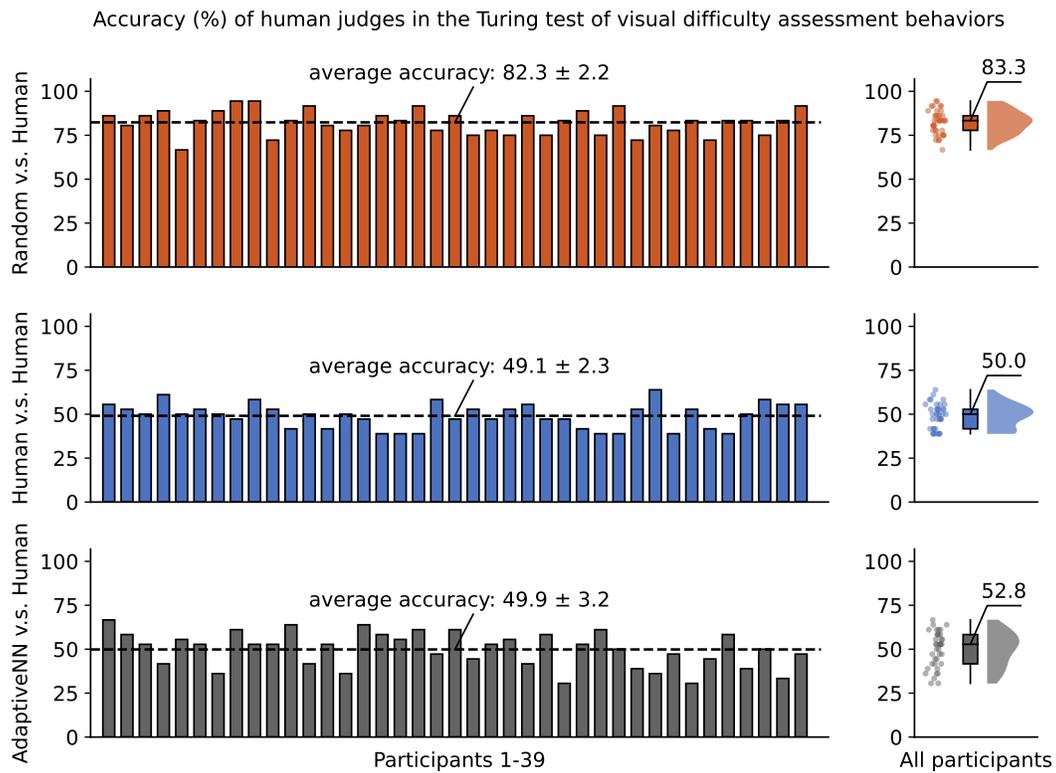

**Supplementary Data Fig. 4** Per-participant results of the 'visual Turing test' on the sample-wise visual difficulty assessment behaviors. '±' indicates the 95% confidence interval. The distribution across all participants is shown on the right.

| Model | Size of Image (fixation) | #Visual Fixations (average) | Computational Cost (GFLOPs/image) | Accuracy (%) |
|---|---|---|---|---|
| *Models with **fixed** computational cost.* | | | | |
| DeiT-T | $224^2$ (–) | – | 1.26 | 72.2 |
| DeiT-T | $288^2$ (–) | – | 2.27 | 75.0 |
| DeiT-T | $384^2$ (–) | – | 4.70 | 76.5 |
| DeiT-S | $224^2$ (–) | – | 4.61 | 79.9 |
| DeiT-S | $288^2$ (–) | – | 7.99 | 80.9 |
| DeiT-S | $384^2$ (–) | – | 15.52 | 81.6 |
| *AdaptiveNN with **online-adjustable** computational cost.* | | | | |
| AdaptiveNN-DeiT-S (ours) | $288^2$ ($112^2$) | 0.09 | 1.20 | $75.7 \pm 0.22$ |
| | | 0.43 | 1.61 | $78.3 \pm 0.20$ |
| | | 0.77 | 2.01 | $79.9 \pm 0.17$ |
| | | 1.11 | 2.42 | $80.8 \pm 0.18$ |
| | | 1.46 | 2.82 | $81.4 \pm 0.13$ |
| | | 1.80 | 3.23 | $81.7 \pm 0.06$ |
| | | 2.14 | 3.63 | $81.9 \pm 0.05$ |
| | | 2.48 | 4.04 | $82.1 \pm 0.09$ |
| | | 2.82 | 4.44 | $82.1 \pm 0.11$ |
| | | 3.15 | 4.85 | $82.2 \pm 0.11$ |
| | | 3.40 | 5.25 | $82.2 \pm 0.12$ |

**Supplementary Data Tab. 2** Quantitative comparisons of AdaptiveNN and traditional non-adaptive models on top of the identical DeiT backbones: Top-1 validation accuracy versus average computational cost for inferring the model. To obtain non-adaptive models with varying costs, we consider two common approaches: adjusting model sizes and input resolutions.

| Model | #Visual Fixations (average) | Computational Cost (GFLOPs) | Trial / Accuracy (%) | | | | | |
|---|---|---|---|---|---|---|---|---|
| | | | $1^{st}$ | $2^{nd}$ | $3^{rd}$ | $4^{th}$ | $5^{th}$ | Average |
| | *AdaptiveNN with **online-adjustable** computational cost.* | | | | | | | |
| | 0.09 | 1.20 | 75.3 | 75.8 | 75.8 | 75.8 | 75.5 | 75.7 ± 0.22 |
| | 0.43 | 1.61 | 78.1 | 78.3 | 78.4 | 78.4 | 78.6 | 78.3 ± 0.20 |
| | 0.77 | 2.01 | 79.7 | 79.8 | 79.9 | 79.9 | 80.2 | 79.9 ± 0.17 |
| | 1.11 | 2.42 | 80.7 | 80.6 | 80.8 | 80.9 | 81.1 | 80.8 ± 0.18 |
| | 1.46 | 2.82 | 81.3 | 81.2 | 81.4 | 81.5 | 81.6 | 81.4 ± 0.13 |
| AdaptiveNN -DeiT-S | 1.80 | 3.23 | 81.7 | 81.6 | 81.7 | 81.8 | 81.9 | 81.7 ± 0.06 |
| | 2.14 | 3.63 | 82.0 | 81.9 | 81.9 | 82.0 | 81.9 | 81.9 ± 0.05 |
| | 2.48 | 4.04 | 82.2 | 82.1 | 81.9 | 82.1 | 81.8 | 82.1 ± 0.09 |
| | 2.82 | 4.44 | 82.4 | 82.2 | 82.0 | 82.2 | 82.1 | 82.1 ± 0.11 |
| | 3.15 | 4.85 | 82.4 | 82.3 | 82.1 | 82.3 | 82.2 | 82.2 ± 0.11 |
| | 3.40 | 5.25 | 82.4 | 82.4 | 82.1 | 82.4 | 82.2 | 82.2 ± 0.12 |

**Supplementary Data Tab. 3** AdaptiveNN built on DeiT-S: Top-1 validation accuracy versus the average computational cost of model inference, with results averaged over 5 trials.

| Model | Size of Image (fixation) | #Visual Fixations (average) | Computational Cost (GFLOPs/image) | Accuracy (%) |
|---|---|---|---|---|
| *Models with **fixed** computational cost.* | | | | |
| ResNet-18 | $224^2$ (–) | – | 1.82 | 71.2 |
| ResNet-18 | $288^2$ (–) | – | 3.01 | 72.3 |
| ResNet-18 | $384^2$ (–) | – | 5.34 | 72.7 |
| ResNet-50 | $224^2$ (–) | – | 4.11 | 77.6 |
| ResNet-50 | $288^2$ (–) | – | 6.80 | 78.6 |
| ResNet-50 | $384^2$ (–) | – | 12.08 | 79.1 |
| *AdaptiveNN with **online-adjustable** computational cost.* | | | | |
| AdaptiveNN-ResNet-50 (ours) | $320^2$ ($112^2$) | 0.08 | 1.20 | $74.0 \pm 0.07$ |
| | | 0.44 | 1.71 | $76.3 \pm 0.13$ |
| | | 0.81 | 2.21 | $77.6 \pm 0.16$ |
| | | 1.17 | 2.72 | $78.4 \pm 0.09$ |
| | | 1.54 | 3.22 | $78.9 \pm 0.07$ |
| | | 1.90 | 3.73 | $79.3 \pm 0.08$ |
| | | 2.27 | 4.23 | $79.5 \pm 0.10$ |
| | | 2.63 | 4.74 | $79.6 \pm 0.10$ |
| | | 2.98 | 5.24 | $79.7 \pm 0.10$ |
| | | 3.27 | 5.75 | $79.8 \pm 0.08$ |
| | | 3.47 | 6.25 | $79.8 \pm 0.07$ |

**Supplementary Data Tab. 4** Quantitative comparisons of AdaptiveNN and traditional non-adaptive models on top of the identical ResNet backbones.

| Model | #Visual Fixations (average) | Computational Cost (GFLOPs) | Trial / Accuracy (%) | | | | | |
|---|---|---|---|---|---|---|---|---|
| | | | 1ˢᵗ | 2ⁿᵈ | 3ʳᵈ | 4ᵗʰ | 5ᵗʰ | Average |
| | *AdaptiveNN with* **online-adjustable** *computational cost.* | | | | | | | |
| AdaptiveNN -ResNet-50 | 0.08 | 1.20 | 73.9 | 74.0 | 74.0 | 74.0 | 74.1 | 74.0 ± 0.07 |
| | 0.44 | 1.71 | 76.4 | 76.5 | 76.2 | 76.2 | 76.5 | 76.3 ± 0.13 |
| | 0.81 | 2.21 | 77.6 | 77.7 | 77.4 | 77.5 | 77.7 | 77.6 ± 0.16 |
| | 1.17 | 2.72 | 78.4 | 78.4 | 78.2 | 78.4 | 78.4 | 78.4 ± 0.09 |
| | 1.54 | 3.22 | 78.9 | 78.9 | 78.8 | 78.9 | 79.0 | 78.9 ± 0.07 |
| | 1.90 | 3.73 | 79.3 | 79.2 | 79.2 | 79.2 | 79.4 | 79.3 ± 0.08 |
| | 2.27 | 4.23 | 79.5 | 79.5 | 79.4 | 79.5 | 79.7 | 79.5 ± 0.10 |
| | 2.63 | 4.74 | 79.7 | 79.6 | 79.5 | 79.7 | 79.8 | 79.6 ± 0.10 |
| | 2.98 | 5.24 | 79.8 | 79.7 | 79.6 | 79.8 | 79.8 | 79.7 ± 0.10 |
| | 3.27 | 5.75 | 79.9 | 79.7 | 79.7 | 79.8 | 79.8 | 79.8 ± 0.08 |
| | 3.47 | 6.25 | 79.9 | 79.7 | 79.7 | 79.8 | 79.8 | 79.8 ± 0.07 |

**Supplementary Data Tab. 5** AdaptiveNN built on ResNet-50: Top-1 validation accuracy versus the average computational cost of model inference, with results averaged over 5 trials.

| Model | #Visual Fixations | Accuracy (%) | | | | | |
|---|---|---|---|---|---|---|---|
| | | 1st trial | 2nd trial | 3rd trial | 4th trial | 5th trial | Average |
| | | *AdaptiveNN with **fixed** numbers of visual fixations.* | | | | | |
| AdaptiveNN -DeiT-S | No | 74.0 | 74.4 | 74.3 | 74.2 | 73.9 | 74.2 ± 0.22 |
| | 1 | 79.5 | 79.6 | 79.6 | 79.5 | 79.6 | 79.6 ± 0.06 |
| | 2 | 81.1 | 81.0 | 81.0 | 81.0 | 81.1 | 81.0 ± 0.06 |
| | 3 | 81.9 | 81.8 | 81.8 | 81.8 | 81.8 | 81.8 ± 0.05 |
| | 4 | 82.3 | 82.3 | 82.0 | 82.2 | 82.0 | 82.2 ± 0.12 |

**Supplementary Data Tab. 6** Relationship between validation accuracy and the number of visual fixations for AdaptiveNN-DeiT-S, assuming all samples use the same fixed number of visual fixations.

| Model | #Visual Fixations | Accuracy (%) | | | | | |
|---|---|---|---|---|---|---|---|
| | | 1st trial | 2nd trial | 3rd trial | 4th trial | 5th trial | Average |
| | | *AdaptiveNN with **fixed** numbers of visual fixations.* | | | | | |
| AdaptiveNN -ResNet-50 | No | 72.9 | 72.9 | 72.9 | 72.9 | 73.0 | 72.9 ± 0.05 |
| | 1 | 77.4 | 77.4 | 77.3 | 77.4 | 77.4 | 77.4 ± 0.05 |
| | 2 | 78.7 | 78.8 | 78.7 | 78.8 | 79.0 | 78.8 ± 0.10 |
| | 3 | 79.6 | 79.5 | 79.5 | 79.4 | 79.6 | 79.5 ± 0.09 |
| | 4 | 79.9 | 79.7 | 79.7 | 79.7 | 79.8 | 79.8 ± 0.08 |

**Supplementary Data Tab. 7** Relationship between validation accuracy and the number of visual fixations for AdaptiveNN-ResNet-50, assuming all samples use the same fixed number of visual fixations.

| Model | Computational Cost (GFLOPs/image) | #Visual Fixations / Data Proportion (%) | | | | |
|---|---|---|---|---|---|---|
| | | No | 1 | 2 | 3 | 4 |
| | 1.20 | 91.2 ± 0.37 | 8.4 ± 0.71 | 0.4 ± 0.33 | 0.0 ± 0.04 | 0.0 ± 0.01 |
| | 1.41 | 75.0 ± 0.68 | 23.0 ± 1.35 | 1.7 ± 0.71 | 0.2 ± 0.09 | 0.1 ± 0.05 |
| | 1.63 | 60.9 ± 0.89 | 33.9 ± 1.78 | 4.5 ± 0.93 | 0.6 ± 0.32 | 0.1 ± 0.08 |
| | 1.84 | 48.7 ± 1.56 | 41.7 ± 2.46 | 7.9 ± 1.77 | 1.5 ± 1.01 | 0.3 ± 0.25 |
| | 2.05 | 38.0 ± 2.91 | 46.7 ± 5.15 | 12.1 ± 3.24 | 2.7 ± 1.61 | 0.5 ± 0.38 |
| | 2.27 | 29.6 ± 2.22 | 49.3 ± 3.11 | 14.8 ± 3.11 | 5.1 ± 2.21 | 1.2 ± 0.65 |
| | 2.48 | 25.8 ± 3.82 | 46.2 ± 6.01 | 15.3 ± 1.08 | 10.7 ± 2.83 | 2.0 ± 1.06 |
| | 2.69 | 19.6 ± 4.41 | 46.6 ± 5.83 | 16.4 ± 2.88 | 14.2 ± 4.15 | 3.2 ± 2.05 |
| | 2.91 | 14.3 ± 1.54 | 48.1 ± 2.89 | 14.7 ± 3.20 | 16.4 ± 5.28 | 6.5 ± 3.25 |
| AdaptiveNN -DeiT-S | 3.12 | 10.5 ± 1.58 | 48.3 ± 3.75 | 15.4 ± 5.38 | 11.7 ± 4.29 | 14.1 ± 2.96 |
| | 3.33 | 7.8 ± 3.33 | 42.8 ± 6.69 | 19.1 ± 8.05 | 13.7 ± 5.13 | 16.6 ± 2.67 |
| | 3.54 | 7.7 ± 2.73 | 35.5 ± 6.29 | 21.2 ± 8.97 | 13.9 ± 3.42 | 21.7 ± 2.98 |
| | 3.78 | 6.7 ± 3.53 | 32.5 ± 7.43 | 17.6 ± 3.30 | 16.3 ± 10.96 | 26.9 ± 6.97 |
| | 3.97 | 3.8 ± 1.86 | 29.5 ± 7.14 | 19.3 ± 4.61 | 15.3 ± 13.32 | 32.1 ± 8.95 |
| | 4.18 | 3.9 ± 1.58 | 25.5 ± 7.35 | 18.3 ± 8.91 | 10.8 ± 10.64 | 41.4 ± 8.68 |
| | 4.40 | 2.3 ± 1.31 | 23.1 ± 7.42 | 15.1 ± 9.40 | 13.0 ± 17.21 | 46.4 ± 10.98 |
| | 4.61 | 2.8 ± 0.95 | 17.7 ± 5.83 | 8.8 ± 5.56 | 22.0 ± 18.59 | 48.7 ± 12.41 |
| | 4.82 | 3.4 ± 3.03 | 10.2 ± 5.02 | 7.7 ± 6.98 | 27.0 ± 21.45 | 51.7 ± 13.35 |
| | 5.04 | 2.8 ± 3.19 | 8.4 ± 3.75 | 5.6 ± 4.24 | 23.5 ± 12.00 | 59.7 ± 9.21 |
| | 5.25 | 2.8 ± 3.06 | 5.5 ± 4.26 | 6.0 ± 5.61 | 20.5 ± 14.28 | 65.3 ± 11.92 |

**Supplementary Data Tab. 8** For AdaptiveNN-DeiT-S, distribution of data using different numbers of visual fixations across varying computational cost constraints.

| Model | Computational Cost (GFLOPs/image) | #Visual Fixations / Data Proportion (%) | | | | |
|---|---|---|---|---|---|---|
| | | No | 1 | 2 | 3 | 4 |
| | 1.20 | 92.8 ± 0.16 | 6.9 ± 0.30 | 0.2 ± 0.14 | 0.0 ± 0.02 | 0.0 ± 0.01 |
| | 1.47 | 75.3 ± 0.47 | 23.0 ± 0.78 | 1.4 ± 0.29 | 0.2 ± 0.13 | 0.1 ± 0.06 |
| | 1.73 | 60.0 ± 0.41 | 35.3 ± 0.25 | 3.6 ± 0.86 | 0.9 ± 0.50 | 0.2 ± 0.16 |
| | 2.00 | 48.4 ± 2.44 | 40.9 ± 5.56 | 8.3 ± 3.89 | 1.9 ± 0.91 | 0.5 ± 0.32 |
| | 2.26 | 35.9 ± 2.34 | 48.3 ± 4.51 | 12.1 ± 2.73 | 3.0 ± 0.85 | 0.7 ± 0.43 |
| | 2.53 | 26.8 ± 3.41 | 49.2 ± 7.06 | 18.6 ± 5.16 | 4.4 ± 2.04 | 1.0 ± 0.69 |
| | 2.79 | 18.3 ± 3.70 | 52.1 ± 4.49 | 20.3 ± 4.22 | 7.3 ± 4.72 | 2.0 ± 1.76 |
| | 3.06 | 13.6 ± 5.15 | 47.3 ± 11.04 | 24.9 ± 8.73 | 11.7 ± 5.15 | 2.4 ± 2.20 |
| | 3.33 | 13.8 ± 4.55 | 35.1 ± 10.26 | 31.5 ± 14.13 | 15.2 ± 8.41 | 4.4 ± 3.34 |
| AdaptiveNN -ResNet-50 | 3.59 | 8.1 ± 2.88 | 35.5 ± 8.09 | 29.0 ± 13.73 | 22.9 ± 8.99 | 4.5 ± 3.04 |
| | 3.86 | 6.0 ± 1.38 | 34.0 ± 5.92 | 22.0 ± 11.37 | 30.7 ± 8.79 | 7.4 ± 4.74 |
| | 4.12 | 2.5 ± 1.42 | 32.8 ± 7.18 | 20.4 ± 8.96 | 32.1 ± 13.33 | 12.3 ± 8.72 |
| | 4.39 | 2.4 ± 2.78 | 26.6 ± 4.78 | 17.3 ± 6.18 | 37.8 ± 21.54 | 15.9 ± 13.13 |
| | 4.66 | 1.0 ± 0.75 | 21.9 ± 6.08 | 15.3 ± 7.11 | 42.6 ± 20.32 | 19.3 ± 12.32 |
| | 4.92 | 1.0 ± 0.64 | 15.2 ± 8.34 | 16.6 ± 10.29 | 40.7 ± 19.47 | 26.5 ± 12.97 |
| | 5.19 | 1.2 ± 1.05 | 12.8 ± 5.24 | 14.3 ± 10.88 | 33.1 ± 14.00 | 38.5 ± 8.07 |
| | 5.45 | 1.2 ± 1.16 | 11.6 ± 5.56 | 10.6 ± 9.85 | 29.0 ± 17.49 | 47.6 ± 13.43 |
| | 5.72 | 1.5 ± 1.52 | 9.0 ± 5.67 | 9.2 ± 7.72 | 23.0 ± 20.60 | 57.4 ± 17.86 |
| | 5.98 | 0.6 ± 0.44 | 4.5 ± 5.36 | 7.5 ± 7.38 | 28.2 ± 19.68 | 59.2 ± 18.82 |
| | 6.25 | 0.4 ± 0.28 | 3.4 ± 5.33 | 8.3 ± 7.16 | 24.7 ± 21.72 | 63.2 ± 21.47 |

**Supplementary Data Tab. 9** For AdaptiveNN-ResNet-50, distribution of data using different numbers of visual fixations across varying computational cost constraints.

| Model | Size of Image (fixation) | Computational Cost (GFLOPs/image) | Accuracy (%) |
|---|---|---|---|
| *Models with **fixed** computational cost.* | | | |
| DeiT-T | $224^2$ (–) | 1.26 | 81.6 |
| DeiT-T | $288^2$ (–) | 2.27 | 83.5 |
| DeiT-T | $384^2$ (–) | 4.70 | 85.7 |
| DeiT-S | $224^2$ (–) | 4.61 | 84.4 |
| DeiT-S | $288^2$ (–) | 7.99 | 86.2 |
| DeiT-S | $384^2$ (–) | 15.52 | 87.3 |
| *AdaptiveNN with **online-adjustable** computational cost.* | | | |
| AdaptiveNN-DeiT-S (ours) | $288^2$ ($112^2$) | 1.40 | $83.1 \pm 0.22$ |
| | | 1.58 | $84.6 \pm 0.24$ |
| | | 1.75 | $85.7 \pm 0.19$ |
| | | 1.93 | $86.4 \pm 0.33$ |
| | | 2.10 | $86.7 \pm 0.25$ |
| | | 2.28 | $86.9 \pm 0.39$ |
| | | 2.45 | $87.2 \pm 0.24$ |
| | | 2.63 | $87.4 \pm 0.15$ |
| | | 2.80 | $87.5 \pm 0.13$ |
| | | 2.98 | $87.5 \pm 0.09$ |
| | | 3.15 | $87.5 \pm 0.09$ |

**Supplementary Data Tab. 10** Quantitative comparisons of AdaptiveNN and traditional non-adaptive models: Top-1 validation accuracy versus average computational cost for model inference on the CUB-200-2011 dataset.

| Model | Size of Image (fixation) | Computational Cost (GFLOPs/image) | Accuracy (%) |
|---|---|---|---|
| *Models with **fixed** computational cost.* | | | |
| DeiT-T | $224^2$ (−) | 1.26 | 78.5 |
| DeiT-T | $288^2$ (−) | 2.27 | 81.9 |
| DeiT-T | $384^2$ (−) | 4.70 | 84.4 |
| DeiT-S | $224^2$ (−) | 4.61 | 80.9 |
| DeiT-S | $288^2$ (−) | 7.99 | 83.9 |
| DeiT-S | $384^2$ (−) | 15.52 | 86.2 |
| *AdaptiveNN with **online-adjustable** computational cost.* | | | |
| AdaptiveNN-DeiT-S (ours) | $288^2$ ($112^2$) | 1.50 | $80.2 \pm 0.13$ |
| | | 1.66 | $82.0 \pm 0.25$ |
| | | 1.82 | $83.3 \pm 0.28$ |
| | | 1.98 | $84.4 \pm 0.09$ |
| | | 2.14 | $85.2 \pm 0.07$ |
| | | 2.30 | $85.8 \pm 0.10$ |
| | | 2.46 | $86.0 \pm 0.20$ |
| | | 2.62 | $86.3 \pm 0.14$ |
| | | 2.78 | $86.4 \pm 0.11$ |
| | | 2.94 | $86.4 \pm 0.11$ |
| | | 3.10 | $86.4 \pm 0.10$ |

**Supplementary Data Tab. 11** Quantitative comparisons of AdaptiveNN and traditional non-adaptive models: Top-1 validation accuracy versus average computational cost for model inference on the NABirds dataset.

| Model | Size of Image (fixation) | Computational Cost (GFLOPs/image) | Accuracy (%) |
|---|---|---|---|
| *Models with **fixed** computational cost.* | | | |
| DeiT-T | $224^2$ (–) | 1.26 | 91.1 |
| DeiT-T | $288^2$ (–) | 2.27 | 92.2 |
| DeiT-T | $384^2$ (–) | 4.70 | 92.4 |
| DeiT-S | $224^2$ (–) | 4.61 | 93.3 |
| DeiT-S | $288^2$ (–) | 7.99 | 93.9 |
| DeiT-S | $384^2$ (–) | 15.52 | 94.0 |
| *AdaptiveNN with **online-adjustable** computational cost.* | | | |
| AdaptiveNN-DeiT-S (ours) | $288^2$ ($112^2$) | 1.25 | $92.4 \pm 0.06$ |
| | | 1.43 | $93.0 \pm 0.02$ |
| | | 1.60 | $93.4 \pm 0.08$ |
| | | 1.78 | $93.7 \pm 0.05$ |
| | | 1.95 | $93.9 \pm 0.07$ |
| | | 2.13 | $94.1 \pm 0.05$ |
| | | 2.30 | $94.2 \pm 0.10$ |
| | | 2.48 | $94.3 \pm 0.03$ |
| | | 2.65 | $94.4 \pm 0.06$ |
| | | 2.83 | $94.4 \pm 0.08$ |
| | | 3.00 | $94.4 \pm 0.10$ |

**Supplementary Data Tab. 12** Quantitative comparisons of AdaptiveNN and traditional non-adaptive models: Top-1 validation accuracy versus average computational cost for model inference on the Oxford-IIIT Pet dataset.

| Model | Size of Image (fixation) | Computational Cost (GFLOPs/image) | Accuracy (%) |
|---|---|---|---|
| | | *Models with **fixed** computational cost.* | |
| DeiT-T | $224^2$ (−) | 1.26 | 81.5 |
| DeiT-T | $288^2$ (−) | 2.27 | 83.0 |
| DeiT-T | $384^2$ (−) | 4.70 | 84.7 |
| DeiT-S | $224^2$ (−) | 4.61 | 87.2 |
| DeiT-S | $288^2$ (−) | 7.99 | 87.8 |
| DeiT-S | $384^2$ (−) | 15.52 | 88.8 |
| | | *AdaptiveNN with **online-adjustable** computational cost.* | |
| AdaptiveNN-DeiT-S (ours) | $288^2$ ($112^2$) | 1.25 | $84.4 \pm 0.19$ |
| | | 1.48 | $86.3 \pm 0.22$ |
| | | 1.70 | $87.8 \pm 0.19$ |
| | | 1.93 | $88.9 \pm 0.20$ |
| | | 2.15 | $89.7 \pm 0.17$ |
| | | 2.38 | $90.3 \pm 0.10$ |
| | | 2.60 | $90.7 \pm 0.11$ |
| | | 2.83 | $90.8 \pm 0.10$ |
| | | 3.05 | $90.9 \pm 0.10$ |
| | | 3.28 | $91.0 \pm 0.10$ |
| | | 3.50 | $91.1 \pm 0.12$ |

**Supplementary Data Tab. 13** Quantitative comparisons of AdaptiveNN and traditional non-adaptive models: Top-1 validation accuracy versus average computational cost for model inference on the Stanford Dogs dataset.

| Model | Size of Image (fixation) | Computational Cost (GFLOPs/image) | Accuracy (%) |
|---|---|---|---|
| *Models with **fixed** computational cost.* | | | |
| DeiT-T | $224^2$ (–) | 1.26 | 85.2 |
| DeiT-T | $288^2$ (–) | 2.27 | 87.1 |
| DeiT-T | $384^2$ (–) | 4.70 | 87.9 |
| DeiT-S | $224^2$ (–) | 4.61 | 90.5 |
| DeiT-S | $288^2$ (–) | 7.99 | 91.3 |
| DeiT-S | $384^2$ (–) | 15.52 | 92.0 |
| *AdaptiveNN with **online-adjustable** computational cost.* | | | |
| AdaptiveNN-DeiT-S (ours) | $288^2$ ($112^2$) | 1.30 | $87.3 \pm 0.13$ |
| | | 1.52 | $88.8 \pm 0.29$ |
| | | 1.74 | $89.8 \pm 0.31$ |
| | | 1.96 | $90.5 \pm 0.30$ |
| | | 2.18 | $91.1 \pm 0.25$ |
| | | 2.40 | $91.6 \pm 0.17$ |
| | | 2.62 | $91.9 \pm 0.13$ |
| | | 2.84 | $92.1 \pm 0.10$ |
| | | 3.06 | $92.3 \pm 0.08$ |
| | | 3.28 | $92.4 \pm 0.10$ |
| | | 3.50 | $92.4 \pm 0.10$ |

**Supplementary Data Tab. 14** Quantitative comparisons of AdaptiveNN and traditional non-adaptive models: Top-1 validation accuracy versus average computational cost for model inference on the Stanford Cars dataset.

| Model | Size of Image (fixation) | Computational Cost (GFLOPs/image) | Accuracy (%) |
|---|---|---|---|
| *Models with **fixed** computational cost.* | | | |
| DeiT-T | $224^2$ (–) | 1.26 | 83.2 |
| DeiT-T | $288^2$ (–) | 2.27 | 84.4 |
| DeiT-T | $384^2$ (–) | 4.70 | 86.3 |
| DeiT-S | $224^2$ (–) | 4.61 | 84.9 |
| DeiT-S | $288^2$ (–) | 7.99 | 85.7 |
| DeiT-S | $384^2$ (–) | 15.52 | 88.4 |
| *AdaptiveNN with **online-adjustable** computational cost.* | | | |
| AdaptiveNN-DeiT-S (ours) | $288^2$ ($112^2$) | 1.40 | $84.7 \pm 0.22$ |
| | | 1.60 | $85.8 \pm 0.55$ |
| | | 1.80 | $86.7 \pm 0.57$ |
| | | 2.00 | $87.5 \pm 0.46$ |
| | | 2.20 | $88.0 \pm 0.15$ |
| | | 2.40 | $88.3 \pm 0.10$ |
| | | 2.60 | $88.5 \pm 0.08$ |
| | | 2.80 | $88.6 \pm 0.03$ |
| | | 3.00 | $88.8 \pm 0.06$ |
| | | 3.20 | $88.8 \pm 0.12$ |
| | | 3.40 | $88.8 \pm 0.12$ |

**Supplementary Data Tab. 15** Quantitative comparisons of AdaptiveNN and traditional non-adaptive models: Top-1 validation accuracy versus average computational cost for model inference on the FGVC-Aircraft dataset.

| Model | Size of Image (fixation) | Computational Cost (GFLOPs/image) | Accuracy (%) |
|---|---|---|---|
| *Models with **fixed** computational cost.* | | | |
| ResNet-18 | $384^2$ (–) | 5.44 | 80.1 |
| ResNet-18 | $720^2$ (–) | 18.98 | 86.8 |
| ResNet-18 | $960^2$ (–) | 33.50 | 89.4 |
| ResNet-50 | $384^2$ (–) | 12.38 | 82.6 |
| ResNet-50 | $720^2$ (–) | 42.79 | 89.0 |
| ResNet-50 | $960^2$ (–) | 75.88 | 90.2 |
| *AdaptiveNN with **online-adjustable** computational cost.* | | | |
| AdaptiveNN-ResNet-18 (ours) | $960 \times 1280$ ($192^2$) | 2.20 | $78.5 \pm 1.39$ |
| | | 2.29 | $80.0 \pm 1.41$ |
| | | 2.38 | $82.1 \pm 1.11$ |
| | | 2.47 | $83.9 \pm 1.10$ |
| | | 2.56 | $86.0 \pm 1.22$ |
| | | 2.65 | $88.4 \pm 1.27$ |
| | | 2.74 | $90.4 \pm 1.09$ |
| | | 2.83 | $91.1 \pm 0.49$ |
| | | 2.92 | $91.3 \pm 0.36$ |
| | | 3.01 | $91.4 \pm 0.30$ |
| | | 3.10 | $91.5 \pm 0.28$ |

**Supplementary Data Tab. 16** Comparisons of AdaptiveNN and traditional non-adaptive models in processing complicated, non-object-centric real-world scenes: Top-1 validation accuracy versus average computational cost for inferring the model. We consider the traffic sign recognition task on the Swedish traffic signs dataset (STSD), composed of 960x1,280 road-scene images collected on real moving vehicles.

| Model | Computational Cost (GFLOPs/image) | Trial / Accuracy (%) | | | | | |
|-------|----------------|-----|-----|-----|-----|-----|---------|
| | | 1st | 2nd | 3rd | 4th | 5th | Average |
| *AdaptiveNN with **online-adjustable** computational cost.* | | | | | | | |
| AdaptiveNN -ResNet-18 | 2.20 | 77.5 | 80.1 | 79.8 | 77.1 | 77.9 | 78.5 ± 1.39 |
| | 2.29 | 78.7 | 82.1 | 80.8 | 79.0 | 79.4 | 80.0 ± 1.41 |
| | 2.38 | 82.4 | 83.4 | 82.7 | 81.0 | 80.9 | 82.1 ± 1.11 |
| | 2.47 | 84.1 | 84.7 | 85.1 | 83.1 | 82.5 | 83.9 ± 1.10 |
| | 2.56 | 86.8 | 86.2 | 87.4 | 85.5 | 84.3 | 86.0 ± 1.22 |
| | 2.65 | 89.6 | 88.8 | 88.9 | 88.1 | 86.3 | 88.4 ± 1.27 |
| | 2.74 | 91.8 | 90.6 | 89.7 | 91.0 | 89.0 | 90.4 ± 1.09 |
| | 2.83 | 91.9 | 90.8 | 91.1 | 91.2 | 90.6 | 91.1 ± 0.49 |
| | 2.92 | 91.9 | 91.1 | 91.3 | 91.3 | 90.9 | 91.3 ± 0.36 |
| | 3.01 | 91.9 | 91.4 | 91.3 | 91.4 | 91.1 | 91.4 ± 0.30 |
| | 3.10 | 91.9 | 91.4 | 91.4 | 91.4 | 91.2 | 91.5 ± 0.28 |

**Supplementary Data Tab. 17** Comparisons of AdaptiveNN and non-adaptive models on complicated, non-object-centric visual data from real driving scenarios: Top-1 accuracy versus average computational cost for traffic sign recognition on the STSD dataset, using ResNet backbones for high-resolution inputs, with results averaged over 5 trials.

| Visual Search Task | Method | Success Rates (%) | | | | | |
|---|---|---|---|---|---|---|---|
| | | 1ˢᵗ trial | 2ⁿᵈ trial | 3ʳᵈ trial | 4ᵗʰ trial | 5ᵗʰ trial | Average |
| 1 digit | RAM | 17.5 | 15.0 | 16.8 | 16.8 | 15.8 | $16.4 \pm 0.88$ |
| | DRAM | 20.4 | 17.8 | 23.2 | 20.9 | 19.1 | $20.3 \pm 1.81$ |
| | **AdaptiveNN** | **96.4** | **98.0** | **94.0** | **93.2** | **94.2** | $\mathbf{95.2 \pm 1.77}$ |
| 2 digits | DRAM | 22.3 | 17.7 | 20.5 | 18.6 | 19.0 | $19.6 \pm 1.62$ |
| | **AdaptiveNN** | **93.0** | **97.8** | **91.2** | **90.1** | **94.3** | $\mathbf{93.3 \pm 2.68}$ |
| 3 digits | DRAM | 18.9 | 17.3 | 19.8 | 18.6 | 20.7 | $19.1 \pm 1.15$ |
| | **AdaptiveNN** | **89.2** | **89.7** | **91.7** | **92.1** | **90.8** | $\mathbf{90.7 \pm 1.12}$ |
| 5 digits | DRAM | 17.2 | 17.9 | 19.5 | 21.1 | 18.8 | $18.9 \pm 1.35$ |
| | **AdaptiveNN** | **87.4** | **87.1** | **88.4** | **87.9** | **84.4** | $\mathbf{87.0 \pm 1.39}$ |

**Supplementary Data Tab. 18**  Success rates of visual search tasks.

| Model | Size of Image (fixation) | #Visual Fixations (average) | Computational Cost (GFLOPs/action) | Avg. Successful Length |
|---|---|---|---|---|
| | | *Models with **fixed** computational cost.* | | |
| RoboFlamingo 6 layers | $224^2$ (–) | – | 123.2 | 2.49 |
| RoboFlamingo 12 layers | $224^2$ (–) | – | 131.0 | 2.65 |
| RoboFlamingo 24 layers | $224^2$ (–) | – | 146.6 | 2.71 |
| | | *AdaptiveNN with **online-adjustable** computational cost.* | | |
| AdaptiveNN (ours) | $224^2$ $(70^2)$ | 0.13 | 22.8 | 2.56 ± 0.12 |
| | | 0.30 | 24.9 | 2.71 ± 0.12 |
| | | 0.36 | 25.7 | 2.79 ± 0.12 |
| | | 0.54 | 27.9 | 2.82 ± 0.09 |
| | | 0.86 | 31.8 | 2.88 ± 0.12 |
| | | 1.01 | 33.6 | 2.90 ± 0.09 |
| | | 1.16 | 35.4 | 2.93 ± 0.12 |
| | | 1.60 | 38.1 | 2.93 ± 0.09 |
| | | 2.00 | 40.6 | 2.94 ± 0.10 |

**Supplementary Data Tab. 19** Quantitative comparisons of AdaptiveNN-based MLLM and non-adaptive MLLM on CALVIN: average successful length for D→D versus average computational cost for inferring the model. For the non-adaptive models, computational costs are modulated by adjusting model sizes.

| Model | Size of Image (fixation) | #Visual Fixations (average) | Computational Cost (GFLOPs/action) | Avg. Successful Length |
|---|---|---|---|---|
| | | *Models with **fixed** computational cost.* | | |
| RoboFlamingo 6 layers | $224^2$ (–) | – | 123.2 | 2.43 |
| RoboFlamingo 12 layers | $224^2$ (–) | – | 131.0 | 3.51 |
| RoboFlamingo 24 layers | $224^2$ (–) | – | 146.6 | 4.07 |
| | | *AdaptiveNN with **online-adjustable** computational cost.* | | |
| AdaptiveNN (ours) | $224^2$ $(70^2)$ | 0.13 | 22.8 | 3.66 ± 0.10 |
| | | 0.20 | 23.7 | 3.71 ± 0.11 |
| | | 0.28 | 24.7 | 3.78 ± 0.10 |
| | | 0.38 | 25.9 | 3.84 ± 0.12 |
| | | 0.43 | 26.5 | 3.90 ± 0.11 |
| | | 0.61 | 28.7 | 4.02 ± 0.09 |
| | | 0.78 | 30.8 | 4.03 ± 0.06 |
| | | 0.95 | 32.9 | 4.05 ± 0.06 |
| | | 1.20 | 35.9 | 4.07 ± 0.10 |

**Supplementary Data Tab. 20** Quantitative comparisons of AdaptiveNN-based MLLM and non-adaptive MLLM on CALVIN: average successful length for ABCD→D versus average computational cost for inferring the model. For the non-adaptive models, computational costs are modulated by adjusting model sizes.

| Task | Computational Cost (GFLOPs/action) / Success Rate (%) | | | | | |
|------|------|------|------|------|------|------|
| | **AdaptiveNN** | | | | | **Non-adaptive** |
| | 22.9 | 24.9 | 25.7 | 31.8 | 35.4 | 146.6 |
| close drawer | 87.7 | 89.3 | 89.5 | 90.1 | 87.3 | 91.1 |
| lift blue block drawer | 84.6 | 70.6 | 63.6 | 84.6 | 92.9 | 100.0 |
| lift blue block slider | 69.0 | 75.0 | 70.3 | 67.3 | 67.0 | 62.9 |
| lift blue block table | 87.4 | 88.1 | 89.7 | 87.0 | 89.9 | 81.6 |
| lift pink block drawer | 77.8 | 72.7 | 66.7 | 54.5 | 83.3 | 80.0 |
| lift pink block slider | 72.0 | 80.4 | 77.3 | 74.5 | 80.0 | 75.5 |
| lift pink block table | 78.4 | 76.6 | 80.5 | 75.0 | 82.3 | 76.6 |
| lift red block drawer | 75.0 | 87.5 | 69.2 | 86.7 | 66.7 | 77.8 |
| lift red block slider | 76.9 | 78.8 | 78.6 | 80.0 | 83.2 | 72.9 |
| lift red block table | 76.9 | 79.8 | 83.2 | 83.4 | 85.1 | 84.1 |
| move slider left | 93.2 | 93.0 | 93.1 | 92.5 | 93.2 | 94.5 |
| move slider right | 91.0 | 93.2 | 92.8 | 88.9 | 94.3 | 89.4 |
| open drawer | 87.8 | 84.1 | 84.3 | 83.9 | 87.1 | 85.6 |
| place in drawer | 94.9 | 97.5 | 96.3 | 94.0 | 96.4 | 95.0 |
| place in slider | 48.0 | 51.2 | 58.6 | 60.4 | 59.8 | 49.4 |
| push blue block left | 71.7 | 70.0 | 73.0 | 77.0 | 75.4 | 62.9 |
| push blue block right | 38.7 | 47.6 | 46.2 | 63.1 | 52.3 | 46.2 |
| push into drawer | 64.8 | 58.4 | 55.1 | 64.5 | 69.5 | 59.8 |
| push pink block left | 85.7 | 85.3 | 82.1 | 86.2 | 86.2 | 83.9 |
| push pink block right | 66.7 | 64.9 | 65.6 | 65.5 | 86.2 | 70.7 |
| push red block left | 59.4 | 64.2 | 72.1 | 73.2 | 64.8 | 74.2 |
| push red block right | 70.1 | 53.7 | 60.3 | 51.4 | 54.0 | 52.9 |
| rotate blue block left | 72.4 | 84.1 | 81.2 | 81.0 | 81.0 | 72.1 |
| rotate blue block right | 79.1 | 87.7 | 77.9 | 85.5 | 81.2 | 82.4 |
| rotate pink block left | 70.0 | 83.0 | 81.6 | 91.7 | 81.2 | 77.1 |
| rotate pink block right | 65.2 | 70.3 | 78.1 | 76.1 | 71.6 | 80.0 |
| rotate red block left | 68.6 | 83.1 | 82.1 | 85.2 | 79.7 | 66.0 |
| rotate red block right | 80.9 | 83.8 | 84.1 | 80.0 | 84.5 | 78.8 |
| stack block | 31.5 | 35.8 | 31.0 | 34.6 | 36.2 | 27.0 |
| turn off led | 90.7 | 89.3 | 93.8 | 91.5 | 92.7 | 86.8 |
| turn off lightbulb | 90.6 | 92.6 | 87.3 | 90.3 | 89.0 | 90.9 |
| turn on led | 94.6 | 93.8 | 93.9 | 97.0 | 94.2 | 92.7 |
| turn on lightbulb | 88.0 | 92.2 | 87.7 | 87.8 | 88.9 | 90.6 |
| unstack block | 87.0 | 82.1 | 84.2 | 88.9 | 91.7 | 93.8 |
| **Average** | 75.8 | 77.6 | 76.8 | 78.6 | 79.7 | 76.6 |

**Supplementary Data Tab. 21** Quantitative comparisons of AdaptiveNN-based MLLM and non-adaptive MLLM: individual success rate of each D→D task versus average computational cost for inferring the model.

| Task | Computational Cost (GFLOPs/action) / Success Rate (%) | | | | | |
|---|---|---|---|---|---|---|
| | **AdaptiveNN** | | | | | **Non-adaptive** |
| | 24.7 | 25.9 | 26.5 | 28.7 | 35.9 | 146.6 |
| close drawer | 100.0 | 100.0 | 100.0 | 99.5 | 98.6 | 98.0 |
| lift blue block drawer | 85.0 | 93.8 | 100.0 | 94.1 | 94.7 | 89.5 |
| lift blue block slider | 88.6 | 90.6 | 88.5 | 91.4 | 89.0 | 90.7 |
| lift blue block table | 88.3 | 85.7 | 90.6 | 84.2 | 83.9 | 92.3 |
| lift pink block drawer | 92.3 | 93.3 | 92.9 | 85.7 | 93.3 | 86.7 |
| lift pink block slider | 91.5 | 89.4 | 91.5 | 89.6 | 90.3 | 91.2 |
| lift pink block table | 88.0 | 86.3 | 84.2 | 88.8 | 89.4 | 85.5 |
| lift red block drawer | 94.1 | 95.2 | 85.7 | 94.1 | 95.0 | 100.0 |
| lift red block slider | 94.6 | 93.9 | 91.8 | 94.0 | 91.0 | 92.6 |
| lift red block table | 96.5 | 95.4 | 96.0 | 96.7 | 96.7 | 93.3 |
| move slider left | 99.1 | 98.8 | 98.8 | 100.0 | 99.6 | 99.6 |
| move slider right | 99.6 | 99.6 | 100.0 | 100.0 | 99.7 | 100.0 |
| open drawer | 99.7 | 99.7 | 98.8 | 98.8 | 100.0 | 98.8 |
| place in drawer | 98.8 | 96.8 | 98.1 | 97.6 | 95.4 | 98.2 |
| place in slider | 69.2 | 79.2 | 84.3 | 88.7 | 88.6 | 89.1 |
| push blue block left | 69.2 | 84.4 | 76.9 | 84.8 | 89.7 | 83.1 |
| push blue block right | 71.2 | 83.1 | 80.3 | 83.3 | 81.4 | 87.0 |
| push into drawer | 69.2 | 69.5 | 67.2 | 73.4 | 77.9 | 70.6 |
| push pink block left | 93.2 | 91.9 | 87.8 | 89.3 | 94.7 | 89.5 |
| push pink block right | 86.7 | 87.7 | 90.2 | 85.9 | 87.9 | 83.6 |
| push red block left | 75.0 | 76.9 | 88.3 | 85.9 | 89.6 | 84.4 |
| push red block right | 84.3 | 81.4 | 77.5 | 82.9 | 68.6 | 84.5 |
| rotate blue block left | 80.0 | 80.6 | 83.8 | 86.6 | 88.1 | 91.2 |
| rotate blue block right | 79.5 | 78.7 | 75.7 | 77.3 | 87.8 | 86.1 |
| rotate pink block left | 89.1 | 92.5 | 94.6 | 87.7 | 96.4 | 92.6 |
| rotate pink block right | 91.0 | 89.7 | 88.9 | 84.5 | 90.3 | 85.5 |
| rotate red block left | 79.7 | 79.7 | 83.9 | 85.5 | 80.6 | 87.1 |
| rotate red block right | 89.9 | 88.9 | 94.3 | 91.8 | 91.7 | 87.5 |
| stack block | 49.7 | 51.1 | 54.5 | 55.5 | 56.9 | 54.5 |
| turn off led | 99.4 | 98.8 | 99.4 | 100.0 | 100.0 | 100.0 |
| turn off lightbulb | 98.4 | 98.5 | 100.0 | 98.6 | 98.6 | 99.3 |
| turn on led | 98.7 | 97.5 | 98.3 | 99.4 | 99.4 | 99.4 |
| turn on lightbulb | 99.4 | 98.9 | 100.0 | 99.4 | 100.0 | 98.2 |
| unstack block | 89.2 | 92.1 | 87.8 | 93.5 | 84.1 | 95.1 |
| **Average** | 87.6 | 88.8 | 89.1 | 89.7 | 90.3 | 90.1 |

**Supplementary Data Tab. 22** Quantitative comparisons of AdaptiveNN-based MLLM and non-adaptive MLLM: individual success rate of each ABCD→D task versus average computational cost for inferring the model.

| Model | #Visual Fixations | Success Rate (%) | | | | | |
|---|---|---|---|---|---|---|---|
| | | 1$^{st}$ trial | 2$^{nd}$ trial | 3$^{rd}$ trial | 4$^{th}$ trial | 5$^{th}$ trial | Average |
| | | *Subtask 1* | | | | | |
| | No | 93.3 | 91.9 | 93.3 | 91.1 | 92.0 | 92.3 ± 0.87 |
| | 1 | 95.1 | 93.3 | 88.8 | 91.5 | 94.6 | 92.7 ± 2.29 |
| | 2 | 93.8 | 93.3 | 93.8 | 93.3 | 96.4 | 94.1 ± 1.18 |
| | | *Subtask 2* | | | | | |
| | No | 67.0 | 75.9 | 77.0 | 66.5 | 66.1 | 67.9 ± 1.86 |
| | 1 | 75.9 | 75.4 | 69.2 | 72.3 | 73.7 | 73.3 ± 2.42 |
| | 2 | 71.9 | 77.2 | 70.5 | 73.7 | 79.5 | 74.6 ± 3.33 |
| | | *Subtask 3* | | | | | |
| AdaptiveNN | No | 40.2 | 49.6 | 46.4 | 46.9 | 47.8 | 46.2 ± 3.18 |
| | 1 | 54.9 | 53.6 | 50.4 | 54.5 | 54.0 | 53.5 ± 1.58 |
| | 2 | 55.8 | 60.7 | 51.8 | 52.2 | 59.4 | 56.0 ± 3.62 |
| | | *Subtask 4* | | | | | |
| | No | 29.0 | 34.8 | 31.7 | 34.4 | 33.5 | 32.7 ± 2.12 |
| | 1 | 43.3 | 41.5 | 39.3 | 39.3 | 41.5 | 41.0 ± 1.53 |
| | 2 | 42.4 | 46.4 | 40.6 | 38.8 | 47.3 | 43.1 ± 3.28 |
| | | *Subtask 5* | | | | | |
| | 0 | 23.2 | 25.4 | 23.7 | 22.3 | 25.9 | 24.1 ± 1.35 |
| | 1 | 32.6 | 29.5 | 29.9 | 31.7 | 33.0 | 31.3 ± 1.42 |
| | 2 | 33.0 | 37.1 | 28.6 | 33.0 | 33.0 | 32.9 ± 2.68 |

**Supplementary Data Tab. 23** Relationship between the success rate of each subtask within the 5-task sequence and the number of visual fixations for AdaptiveNN-based MLLM (D→D), assuming all samples use the same fixed number of visual fixations.

| Model | #Visual Fixations | Success Rate (%) | | | | | |
|---|---|---|---|---|---|---|---|
| | | 1$^{st}$ trial | 2$^{nd}$ trial | 3$^{rd}$ trial | 4$^{th}$ trial | 5$^{th}$ trial | Average |
| | | *Subtask 1* | | | | | |
| | No | 96.9 | 96.4 | 95.1 | 96.0 | 95.5 | 96.0 $\pm$ 0.63 |
| | 1 | 97.3 | 95.1 | 98.2 | 95.5 | 97.3 | 96.7 $\pm$ 1.18 |
| | 2 | 98.2 | 96.9 | 97.8 | 96.4 | 97.3 | 97.3 $\pm$ 0.63 |
| | | *Subtask 2* | | | | | |
| | No | 88.4 | 83.5 | 85.3 | 84.4 | 82.6 | 84.8 $\pm$ 2.00 |
| | 1 | 88.0 | 84.4 | 87.5 | 86.2 | 86.2 | 86.4 $\pm$ 1.25 |
| | 2 | 92.4 | 87.5 | 84.4 | 90.2 | 89.7 | 88.8 $\pm$ 2.72 |
| | | *Subtask 3* | | | | | |
| AdaptiveNN | No | 74.6 | 69.2 | 73.2 | 73.2 | 69.6 | 72.0 $\pm$ 2.14 |
| | 1 | 79.9 | 76.3 | 75.4 | 77.2 | 75.4 | 76.9 $\pm$ 1.66 |
| | 2 | 85.3 | 82.1 | 77.2 | 79.5 | 79.9 | 80.8 $\pm$ 2.72 |
| | | *Subtask 4* | | | | | |
| | No | 62.9 | 60.7 | 57.1 | 62.5 | 59.4 | 60.5 $\pm$ 2.12 |
| | 1 | 72.8 | 68.3 | 65.6 | 70.1 | 62.5 | 67.9 $\pm$ 3.55 |
| | 2 | 76.3 | 71.9 | 67.0 | 69.2 | 68.8 | 70.6 $\pm$ 3.26 |
| | | *Subtask 5* | | | | | |
| | 0 | 50.4 | 48.7 | 44.2 | 49.6 | 49.1 | 48.4 $\pm$ 2.18 |
| | 1 | 63.4 | 58.5 | 59.8 | 59.4 | 51.3 | 58.5 $\pm$ 3.94 |
| | 2 | 71.0 | 63.8 | 58.9 | 63.8 | 62.9 | 64.1 $\pm$ 3.89 |

**Supplementary Data Tab. 24**   Relationship between the success rate of each subtask within the 5-task sequence and the number of visual fixations for AdaptiveNN-based MLLM (ABCD→D), assuming all samples use the same fixed number of visual fixations.

| Method | Zero-shot Normalized Human-like Score (ImageNet-1K → SALICON-split-1) | |
|---|---|---|
| | Mean ± Std | Median |
| Gaussian Distribution | 0.15 ± 0.13 | 0.15 |
| Center-Corner | 0.08 ± 0.12 | 0.08 |
| GradCAM | −0.09 ± 0.18 | −0.07 |
| LayerCAM | 0.02 ± 0.18 | 0.03 |
| GradCAM+GMM | 0.44 ± 0.17 | 0.46 |
| **AdaptiveNN-DeiT-S** | **1.09 ± 0.16** | **1.11** |

**Supplementary Data Tab. 25** Normalized human-like scores which quantify the probability that human ground truth gaze centers (estimated from ∼60 observers' visual perception behaviors) fall within the visual fixation regions identified by AdaptiveNN or other methods. Evaluated on the SALICON dataset using a 'zero-shot' setting, AdaptiveNN is trained on ImageNet and tested on the unseen SALICON-split-1 data.

| Method | Zero-shot Normalized Human-like Score (ImageNet-1K → SALICON-split-2) | |
|---|---|---|
| | Mean ± Std | Median |
| Gaussian Distribution | 0.17 ± 0.15 | 0.16 |
| Center-Corner | 0.13 ± 0.13 | 0.11 |
| GradCAM | −0.09 ± 0.21 | −0.10 |
| LayerCAM | −0.02 ± 0.17 | −0.01 |
| GradCAM+GMM | 0.48 ± 0.19 | 0.47 |
| **AdaptiveNN-DeiT-S** | **1.11 ± 0.18** | **1.11** |

**Supplementary Data Tab. 26** Normalized human-like scores which quantify the probability that human ground truth gaze centers (estimated from ∼60 observers' visual perception behaviors) fall within the visual fixation regions identified by AdaptiveNN or other methods. Evaluated on the SALICON dataset using a 'zero-shot' setting, AdaptiveNN is trained on ImageNet and tested on the unseen SALICON-split-2 data.

| Human Judge | Pearson correlation coefficient between human-assessed difficulty level and AdaptiveNN predictions | | | | | |
|:---:|:---:|:---:|:---:|:---:|:---:|:---:|
| | Monarch butterfly | Balloon | Speedboat | Ibex | Snowplow | Cockatoo |
| 1 | 0.744 | 0.644 | 0.646 | 0.799 | 0.632 | 0.673 |
| 2 | 0.724 | 0.542 | 0.420 | 0.435 | 0.452 | 0.567 |
| 3 | 0.650 | 0.614 | 0.530 | 0.558 | 0.525 | 0.515 |
| 4 | 0.803 | 0.763 | 0.518 | 0.677 | 0.380 | 0.577 |
| 5 | 0.847 | 0.651 | 0.618 | 0.733 | 0.423 | 0.849 |
| 6 | 0.674 | 0.600 | 0.329 | 0.540 | 0.362 | 0.715 |
| 7 | 0.735 | 0.524 | 0.588 | 0.705 | 0.233 | 0.678 |
| 8 | 0.731 | 0.552 | 0.676 | 0.558 | 0.571 | 0.547 |
| 9 | 0.799 | 0.655 | 0.427 | 0.576 | 0.247 | 0.724 |
| 10 | 0.720 | 0.593 | 0.408 | 0.601 | 0.121 | 0.872 |
| aggregate | 0.809 | 0.697 | 0.626 | 0.763 | 0.544 | 0.791 |

**Supplementary Data Tab. 27** Correlation of human-assessed difficulty scores and difficulty levels (state values) evaluated by the Vision Agent of AdaptiveNN.

| Human Judge | Accuracy (%) of distinguishing between each target behavior and human behavior | | |
|:---:|:---:|:---:|:---:|
| | **Random** | **Human** | **AdaptiveNN** |
| 1 | 94.4 | 63.9 | 41.7 |
| 2 | 86.1 | 41.7 | 66.7 |
| 3 | 80.6 | 44.4 | 63.9 |
| 4 | 72.2 | 41.7 | 52.8 |
| 5 | 69.4 | 41.7 | 47.2 |
| 6 | 83.3 | 69.4 | 41.7 |
| 7 | 72.2 | 44.4 | 41.7 |
| 8 | 63.9 | 38.9 | 61.1 |
| 9 | 63.9 | 44.4 | 41.7 |
| 10 | 97.2 | 55.6 | 36.1 |
| 11 | 97.2 | 58.3 | 47.2 |
| 12 | 83.3 | 55.6 | 58.3 |
| 13 | 100.0 | 63.9 | 61.1 |
| 14 | 58.3 | 38.9 | 55.6 |
| 15 | 86.1 | 36.1 | 58.3 |
| 16 | 86.1 | 50.0 | 52.8 |
| 17 | 80.6 | 55.6 | 33.3 |
| 18 | 88.9 | 61.1 | 52.8 |
| 19 | 80.6 | 58.3 | 55.6 |
| 20 | 58.3 | 44.4 | 44.4 |
| 21 | 69.4 | 50.0 | 63.9 |
| 22 | 66.7 | 52.8 | 44.4 |
| 23 | 86.1 | 36.1 | 58.3 |
| 24 | 86.1 | 47.2 | 50.0 |
| 25 | 63.9 | 52.8 | 44.4 |
| 26 | 91.7 | 47.2 | 58.3 |
| 27 | 83.3 | 63.9 | 38.9 |
| 28 | 83.3 | 41.7 | 47.2 |
| 29 | 75.0 | 50.0 | 44.4 |
| 30 | 88.9 | 47.2 | 47.2 |
| 31 | 58.3 | 52.8 | 55.6 |
| 32 | 77.8 | 50.0 | 63.9 |
| 33 | 88.9 | 55.6 | 44.4 |
| 34 | 83.3 | 50.0 | 50.0 |
| 35 | 80.6 | 36.1 | 52.8 |
| 36 | 80.6 | 58.3 | 52.8 |
| 37 | 80.6 | 33.3 | 50.0 |
| 38 | 80.6 | 52.8 | 61.1 |
| 39 | 83.3 | 36.1 | 55.6 |
| Average CI | 79.8 ± 3.6 | 49.3 ± 2.9 | 51.2 ± 2.7 |

**Supplementary Data Tab. 28**  Results of 'Turing Test of the Visual Fixation Behaviors': human judges (n=39) were randomly presented with paired examples of visual perception behaviors from 'humans' and one of AdaptiveNN, humans, random behaviors. Judges were tasked with identifying the machine, including in control pairs ('random v.s. human' or 'human v.s. human'). Values represent individual and mean accuracy across judges.

| Human Judge | Accuracy (%) of distinguishing between each target behavior and human behavior | | |
|:---:|:---:|:---:|:---:|
| | **Random** | **Human** | **AdaptiveNN** |
| 1 | 86.1 | 55.6 | 66.7 |
| 2 | 80.6 | 52.8 | 58.3 |
| 3 | 86.1 | 50.0 | 52.8 |
| 4 | 88.9 | 61.1 | 41.7 |
| 5 | 66.7 | 50.0 | 55.6 |
| 6 | 83.3 | 52.8 | 52.8 |
| 7 | 88.9 | 50.0 | 36.1 |
| 8 | 94.4 | 47.2 | 61.1 |
| 9 | 94.4 | 58.3 | 52.8 |
| 10 | 72.2 | 52.8 | 52.8 |
| 11 | 83.3 | 41.7 | 63.9 |
| 12 | 91.7 | 50.0 | 41.7 |
| 13 | 80.6 | 41.7 | 52.8 |
| 14 | 77.8 | 50.0 | 36.1 |
| 15 | 80.6 | 47.2 | 63.9 |
| 16 | 86.1 | 38.9 | 58.3 |
| 17 | 83.3 | 38.9 | 55.6 |
| 18 | 91.7 | 38.9 | 61.1 |
| 19 | 77.8 | 58.3 | 47.2 |
| 20 | 86.1 | 47.2 | 61.1 |
| 21 | 75.0 | 52.8 | 44.4 |
| 22 | 77.8 | 47.2 | 52.8 |
| 23 | 75.0 | 52.8 | 55.6 |
| 24 | 86.1 | 55.6 | 41.7 |
| 25 | 75.0 | 47.2 | 58.3 |
| 26 | 83.3 | 47.2 | 30.6 |
| 27 | 88.9 | 41.7 | 52.8 |
| 28 | 75.0 | 38.9 | 61.1 |
| 29 | 91.7 | 38.9 | 50.0 |
| 30 | 72.2 | 52.8 | 38.9 |
| 31 | 80.6 | 63.9 | 36.1 |
| 32 | 77.8 | 38.9 | 47.2 |
| 33 | 83.3 | 52.8 | 30.6 |
| 34 | 72.2 | 41.7 | 44.4 |
| 35 | 83.3 | 38.9 | 58.3 |
| 36 | 83.3 | 50.0 | 38.9 |
| 37 | 75.0 | 58.3 | 50.0 |
| 38 | 83.3 | 55.6 | 33.3 |
| 39 | 91.7 | 55.6 | 47.2 |
| Average CI | 82.3 ± 2.2 | 49.1 ± 2.3 | 49.9 ± 3.2 |

**Supplementary Data Tab. 29** Results of 'Turing Test of the Difficulty-assessment Behaviors': human judges (n=39) were randomly presented with paired examples of visual perception behaviors from 'humans' and one of AdaptiveNN, humans, random behaviors. Judges were tasked with identifying the machine, including in control pairs ('random v.s. human' or 'human v.s. human'). Values represent individual and mean accuracy across judges.

| Policy | Trial / Accuracy (%) | | | | | |
|--------|:----:|:----:|:----:|:----:|:----:|:--------:|
| | 1st | 2nd | 3rd | 4th | 5th | Average |
| *Pre-defined Policy* | | | | | | |
| Random | 73.6 | 73.3 | 73.4 | 73.7 | 73.7 | 73.5 ± 0.16 |
| Gaussian Distribution | 74.1 | 74.1 | 74.0 | 74.3 | 74.1 | 74.1 ± 0.08 |
| Center-Corner | 76.5 | 76.3 | 76.2 | 76.6 | 76.5 | 76.4 ± 0.12 |
| *CAM-based Policy* | | | | | | |
| GradCAM | 76.7 | 76.5 | 76.6 | 76.7 | 76.7 | 76.7 ± 0.07 |
| GradCAM++ | 75.6 | 75.4 | 75.3 | 75.5 | 75.7 | 75.5 ± 0.16 |
| XGradCAM | 75.6 | 75.3 | 75.3 | 75.5 | 75.6 | 75.5 ± 0.14 |
| LayerCAM | 75.8 | 75.4 | 75.2 | 75.7 | 75.7 | 75.5 ± 0.23 |
| GradCAM + GMM | 77.1 | 77.1 | 77.0 | 77.0 | 77.0 | 77.0 ± 0.04 |
| *Learnable Policy Network* | | | | | | |
| Spatial Transformer Net | 75.9 | 75.7 | 75.6 | 75.7 | 75.9 | 75.8 ± 0.12 |
| Gumbel-Softmax | 73.2 | 73.1 | 73.1 | 73.4 | 73.2 | 73.2 ± 0.11 |
| AdaptiveNN (Ours) | **78.3** | **78.2** | **78.0** | **78.2** | **78.3** | **78.2 ± 0.11** |

**Supplementary Data Tab. 30** Comparisons of different policies for establishing the fixation localization strategy on ImageNet-1K. 1 visual fixation is used on the top of AdaptiveNN-DeiT-S.

| Policy | Trial / Accuracy (%) | | | | | |
|--------|:----:|:----:|:----:|:----:|:----:|:--------:|
| | 1st | 2nd | 3rd | 4th | 5th | Average |
| *Pre-defined Policy* | | | | | | |
| Random | 77.0 | 76.9 | 77.0 | 77.0 | 77.0 | 77.0 ± 0.04 |
| Gaussian Distribution | 77.2 | 77.2 | 77.4 | 77.5 | 77.2 | 77.3 ± 0.11 |
| Center-Corner | 78.3 | 78.1 | 78.1 | 78.2 | 78.3 | 78.2 ± 0.09 |
| *CAM-based Policy* | | | | | | |
| GradCAM | 79.1 | 79.0 | 79.0 | 78.7 | 79.0 | 79.0 ± 0.12 |
| GradCAM++ | 78.4 | 78.3 | 78.2 | 78.3 | 78.4 | 78.3 ± 0.05 |
| XGradCAM | 78.4 | 78.2 | 78.4 | 78.1 | 78.4 | 78.3 ± 0.13 |
| LayerCAM | 78.6 | 78.3 | 78.2 | 78.5 | 78.3 | 78.4 ± 0.16 |
| GradCAM + GMM | 79.6 | 79.6 | 79.5 | 79.5 | 79.6 | 79.6 ± 0.05 |
| *Learnable Policy Network* | | | | | | |
| Spatial Transformer Net | 77.6 | 77.4 | 77.4 | 77.5 | 77.5 | 77.5 ± 0.06 |
| Gumbel-Softmax | 77.0 | 76.8 | 77.0 | 76.9 | 76.9 | 76.9 ± 0.07 |
| AdaptiveNN (Ours) | **80.5** | **80.3** | **80.5** | **80.3** | **80.5** | **80.4 ± 0.10** |

**Supplementary Data Tab. 31** Comparisons of different policies for establishing the fixation localization strategy on ImageNet-1K. 2 visual fixations are used on the top of AdaptiveNN-DeiT-S.

| Policy | Trial / Accuracy (%) | | | | | |
|---|---|---|---|---|---|---|
| | **1**st | **2**nd | **3**rd | **4**th | **5**th | **Average** |
| *Pre-defined Policy* | | | | | | |
| Random | 78.5 | 78.5 | 78.5 | 78.6 | 78.5 | 78.5 ± 0.03 |
| Gaussian Distribution | 78.7 | 78.5 | 78.5 | 78.8 | 78.6 | 78.6 ± 0.12 |
| Center-Corner | 79.6 | 79.4 | 79.4 | 79.4 | 79.6 | 79.5 ± 0.11 |
| *CAM-based Policy* | | | | | | |
| GradCAM | 79.8 | 79.8 | 79.8 | 79.7 | 79.8 | 79.8 ± 0.06 |
| GradCAM++ | 79.5 | 79.5 | 79.4 | 79.4 | 79.6 | 79.5 ± 0.08 |
| XGradCAM | 79.6 | 79.5 | 79.5 | 79.4 | 79.6 | 79.5 ± 0.08 |
| LayerCAM | 79.8 | 79.5 | 79.3 | 79.6 | 79.7 | 79.6 ± 0.15 |
| GradCAM + GMM | 80.4 | 80.4 | 80.3 | 80.3 | 80.5 | 80.4 ± 0.07 |
| *Learnable Policy Network* | | | | | | |
| Spatial Transformer Net | 78.3 | 78.4 | 78.2 | 78.3 | 78.3 | 78.3 ± 0.07 |
| Gumbel-Softmax | 78.8 | 78.7 | 78.6 | 78.7 | 78.8 | 78.7 ± 0.07 |
| AdaptiveNN (Ours) | **81.3** | **81.3** | **81.3** | **81.3** | **81.3** | **81.3 ± 0.03** |

**Supplementary Data Tab. 32** Comparisons of different policies for establishing the fixation localization strategy on ImageNet-1K. 3 visual fixations are used on the top of AdaptiveNN-DeiT-S.

| Policy | Trial / Accuracy (%) | | | | | |
|---|---|---|---|---|---|---|
| | **1**st | **2**nd | **3**rd | **4**th | **5**th | **Average** |
| *Pre-defined Policy* | | | | | | |
| Random | 79.3 | 79.4 | 79.5 | 79.4 | 79.3 | 79.4 ± 0.05 |
| Gaussian Distribution | 79.3 | 79.3 | 79.2 | 79.4 | 79.2 | 79.3 ± 0.07 |
| Center-Corner | 80.4 | 80.2 | 80.1 | 80.2 | 80.4 | 80.3 ± 0.09 |
| *CAM-based Policy* | | | | | | |
| GradCAM | 80.4 | 80.4 | 80.3 | 80.2 | 80.4 | 80.4 ± 0.07 |
| GradCAM++ | 80.2 | 80.3 | 80.1 | 80.0 | 80.3 | 80.2 ± 0.11 |
| XGradCAM | 80.2 | 80.0 | 80.2 | 80.3 | 80.2 | 80.2 ± 0.11 |
| LayerCAM | 80.4 | 80.3 | 80.1 | 80.3 | 80.3 | 80.3 ± 0.11 |
| GradCAM + GMM | 80.9 | 80.9 | 80.9 | 80.9 | 80.9 | 80.9 ± 0.02 |
| *Learnable Policy Network* | | | | | | |
| Spatial Transformer Net | 78.7 | 78.7 | 78.5 | 78.7 | 78.7 | 78.7 ± 0.06 |
| Gumbel-Softmax | 79.8 | 79.6 | 79.6 | 79.6 | 79.8 | 79.7 ± 0.09 |
| AdaptiveNN (Ours) | **81.8** | **81.8** | **81.7** | **81.7** | **81.8** | **81.8 ± 0.05** |

**Supplementary Data Tab. 33** Comparisons of different policies for establishing the fixation localization strategy on ImageNet-1K. 4 visual fixations are used on the top of AdaptiveNN-DeiT-S.

| Policy | Trial / Accuracy (%) | | | | | |
|---|---|---|---|---|---|---|
| | **1**st | **2**nd | **3**rd | **4**th | **5**th | **Average** |
| *Pre-defined Policy* | | | | | | |
| Random | 66.3 | 65.6 | 66.2 | 66.2 | 66.1 | 66.1 ± 0.26 |
| Gaussian Distribution | 67.6 | 67.2 | 67.6 | 67.6 | 67.5 | 67.5 ± 0.15 |
| Center-Corner | 71.1 | 71.0 | 71.2 | 70.8 | 70.7 | 70.9 ± 0.20 |
| *CAM-based Policy* | | | | | | |
| GradCAM | 71.2 | 70.7 | 70.8 | 70.8 | 70.8 | 70.8 ± 0.19 |
| GradCAM++ | 70.5 | 70.3 | 70.5 | 70.4 | 70.4 | 70.4 ± 0.07 |
| XGradCAM | 71.2 | 70.7 | 70.8 | 70.8 | 70.8 | 70.8 ± 0.19 |
| LayerCAM | 70.5 | 70.4 | 70.5 | 70.5 | 70.5 | 70.5 ± 0.05 |
| GradCAM + GMM | 71.8 | 71.4 | 71.7 | 71.5 | 71.7 | 71.6 ± 0.14 |
| *Learnable Policy Network* | | | | | | |
| Spatial Transformer Net | 70.0 | 69.6 | 69.9 | 69.7 | 69.7 | 69.8 ± 0.14 |
| Gumbel-Softmax | 68.1 | 67.8 | 68.1 | 68.0 | 67.9 | 68.0 ± 0.11 |
| **AdaptiveNN (Ours)** | **73.5** | **73.5** | **73.6** | **73.7** | **73.6** | **73.6 ± 0.08** |

**Supplementary Data Tab. 34** Comparisons of different policies for establishing the fixation localization strategy on ImageNet-1K. 1 visual fixation is used on the top of AdaptiveNN-ResNet-50.

| Policy | Trial / Accuracy (%) | | | | | |
|---|---|---|---|---|---|---|
| | **1**st | **2**nd | **3**rd | **4**th | **5**th | **Average** |
| *Pre-defined Policy* | | | | | | |
| Random | 70.8 | 70.7 | 70.8 | 71.0 | 71.0 | 70.8 ± 0.10 |
| Gaussian Distribution | 71.8 | 71.5 | 71.8 | 71.9 | 71.9 | 71.8 ± 0.15 |
| Center-Corner | 73.4 | 73.0 | 73.3 | 73.2 | 73.2 | 73.2 ± 0.14 |
| *CAM-based Policy* | | | | | | |
| GradCAM | 74.7 | 74.4 | 74.6 | 74.6 | 74.5 | 74.6 ± 0.08 |
| GradCAM++ | 74.4 | 74.1 | 74.4 | 74.3 | 74.3 | 74.3 ± 0.12 |
| XGradCAM | 74.7 | 74.4 | 74.6 | 74.6 | 74.5 | 74.6 ± 0.09 |
| LayerCAM | 74.4 | 74.2 | 74.3 | 74.4 | 74.4 | 74.4 ± 0.07 |
| GradCAM + GMM | 75.7 | 75.6 | 75.8 | 75.7 | 75.6 | 75.7 ± 0.08 |
| *Learnable Policy Network* | | | | | | |
| Spatial Transformer Net | 72.2 | 72.1 | 72.2 | 72.1 | 72.1 | 72.1 ± 0.08 |
| Gumbel-Softmax | 72.9 | 73.0 | 73.0 | 72.9 | 72.9 | 72.9 ± 0.06 |
| **AdaptiveNN (Ours)** | **77.0** | **77.1** | **77.2** | **77.2** | **77.1** | **77.1 ± 0.07** |

**Supplementary Data Tab. 35** Comparisons of different policies for establishing the fixation localization strategy on ImageNet-1K. 2 visual fixations are used on the top of AdaptiveNN-ResNet-50.

| Policy | Trial / Accuracy (%) | | | | | |
|---|---|---|---|---|---|---|
| | **1ˢᵗ** | **2ⁿᵈ** | **3ʳᵈ** | **4ᵗʰ** | **5ᵗʰ** | **Average** |
| *Pre-defined Policy* | | | | | | |
| Random | 73.3 | 73.1 | 73.4 | 73.5 | 73.5 | 73.3 ± 0.16 |
| Gaussian Distribution | 73.9 | 73.6 | 73.9 | 73.8 | 73.8 | 73.8 ± 0.13 |
| Center-Corner | 75.4 | 75.1 | 75.4 | 75.1 | 75.2 | 75.2 ± 0.11 |
| *CAM-based Policy* | | | | | | |
| GradCAM | 76.2 | 75.8 | 76.2 | 76.2 | 76.2 | 76.1 ± 0.15 |
| GradCAM++ | 76.1 | 75.7 | 75.9 | 76.0 | 75.9 | 75.9 ± 0.14 |
| XGradCAM | 76.2 | 75.8 | 76.2 | 76.2 | 76.2 | 76.1 ± 0.15 |
| LayerCAM | 76.2 | 75.7 | 76.1 | 76.1 | 76.1 | 76.0 ± 0.16 |
| GradCAM + GMM | 77.2 | 76.9 | 77.1 | 77.1 | 76.9 | 77.0 ± 0.10 |
| *Learnable Policy Network* | | | | | | |
| Spatial Transformer Net | 73.3 | 73.0 | 73.3 | 73.1 | 73.2 | 73.2 ± 0.12 |
| Gumbel-Softmax | 75.3 | 75.2 | 75.5 | 75.2 | 75.2 | 75.3 ± 0.05 |
| **AdaptiveNN (Ours)** | **78.3** | **78.5** | **78.5** | **78.4** | **78.3** | **78.4 ± 0.07** |

**Supplementary Data Tab. 36** Comparisons of different policies for establishing the fixation localization strategy on ImageNet-1K. 3 visual fixations are used on the top of AdaptiveNN-ResNet-50.

| Policy | Trial / Accuracy (%) | | | | | |
|---|---|---|---|---|---|---|
| | **1ˢᵗ** | **2ⁿᵈ** | **3ʳᵈ** | **4ᵗʰ** | **5ᵗʰ** | **Average** |
| *Pre-defined Policy* | | | | | | |
| Random | 74.8 | 74.6 | 74.9 | 74.9 | 74.9 | 74.8 ± 0.14 |
| Gaussian Distribution | 75.1 | 74.8 | 75.0 | 75.1 | 75.0 | 75.0 ± 0.09 |
| Center-Corner | 76.7 | 76.5 | 76.6 | 76.6 | 76.6 | 76.6 ± 0.09 |
| *CAM-based Policy* | | | | | | |
| GradCAM | 77.0 | 76.8 | 77.1 | 77.1 | 77.0 | 77.0 ± 0.11 |
| GradCAM++ | 77.0 | 76.8 | 77.0 | 77.0 | 76.9 | 76.9 ± 0.10 |
| XGradCAM | 77.0 | 76.8 | 77.1 | 77.1 | 77.0 | 77.0 ± 0.11 |
| LayerCAM | 77.1 | 76.8 | 77.1 | 77.1 | 77.1 | 77.0 ± 0.10 |
| GradCAM + GMM | 77.8 | 77.8 | 77.8 | 78.0 | 77.8 | 77.8 ± 0.08 |
| *Learnable Policy Network* | | | | | | |
| Spatial Transformer Net | 73.9 | 73.5 | 73.9 | 73.8 | 73.7 | 73.8 ± 0.12 |
| Gumbel-Softmax | 76.5 | 76.2 | 76.5 | 76.5 | 76.4 | 76.4 ± 0.12 |
| **AdaptiveNN (Ours)** | **79.0** | **79.2** | **79.2** | **79.3** | **79.0** | **79.1 ± 0.09** |

**Supplementary Data Tab. 37** Comparisons of different policies for establishing the fixation localization strategy on ImageNet-1K. 4 visual fixations are used on the top of AdaptiveNN-ResNet-50.

| Model | Predicted State Value | Trial / Test Loss | | | | | |
|---|---|---|---|---|---|---|---|
| | | 1st | 2nd | 3rd | 4th | 5th | Average |
| AdaptiveNN -DeiT-S | 0.002 | 0.264 | 0.270 | 0.256 | 0.269 | 0.232 | $0.258 \pm 0.0156$ |
| | 0.077 | 0.546 | 0.511 | 0.515 | 0.518 | 0.520 | $0.522 \pm 0.0141$ |
| | 0.152 | 0.820 | 0.767 | 0.794 | 0.722 | 0.760 | $0.773 \pm 0.0370$ |
| | 0.226 | 0.894 | 0.947 | 0.934 | 0.902 | 0.788 | $0.893 \pm 0.0626$ |
| | 0.301 | 0.947 | 1.085 | 0.960 | 0.986 | 1.069 | $1.009 \pm 0.0635$ |
| | 0.376 | 1.085 | 1.082 | 1.191 | 1.352 | 1.178 | $1.178 \pm 0.1101$ |
| | 0.451 | 1.239 | 1.271 | 1.224 | 1.333 | 1.264 | $1.266 \pm 0.0418$ |
| | 0.526 | 1.285 | 1.306 | 1.345 | 1.363 | 1.420 | $1.344 \pm 0.0526$ |
| | 0.600 | 1.452 | 1.422 | 1.485 | 1.556 | 1.572 | $1.497 \pm 0.0648$ |
| | 0.675 | 1.465 | 1.692 | 1.613 | 1.519 | 1.507 | $1.559 \pm 0.0920$ |
| | 0.750 | 1.629 | 1.403 | 1.604 | 1.523 | 1.739 | $1.580 \pm 0.1254$ |

**Supplementary Data Tab. 38**   Average test loss corresponding to the validation data with different state values predicted by AdaptiveNN-DeiT-S (1-st fixation).

| Model | Predicted State Value | Trial / Test Loss | | | | | |
|---|---|---|---|---|---|---|---|
| | | 1st | 2nd | 3rd | 4th | 5th | Average |
| AdaptiveNN -DeiT-S | 0.002 | 0.267 | 0.298 | 0.287 | 0.258 | 0.269 | $0.276 \pm 0.0162$ |
| | 0.052 | 0.708 | 0.667 | 0.713 | 0.710 | 0.702 | $0.700 \pm 0.0189$ |
| | 0.102 | 0.914 | 0.967 | 1.014 | 1.009 | 1.089 | $0.999 \pm 0.0646$ |
| | 0.151 | 1.251 | 1.149 | 1.198 | 1.108 | 1.255 | $1.192 \pm 0.0640$ |
| | 0.201 | 1.285 | 1.453 | 1.433 | 1.318 | 1.413 | $1.380 \pm 0.0743$ |
| | 0.251 | 1.682 | 1.598 | 1.664 | 1.555 | 1.695 | $1.639 \pm 0.0600$ |
| | 0.301 | 1.767 | 1.735 | 1.722 | 1.793 | 1.821 | $1.767 \pm 0.0409$ |
| | 0.351 | 1.772 | 1.767 | 1.894 | 1.937 | 1.935 | $1.861 \pm 0.0855$ |
| | 0.400 | 1.862 | 1.899 | 2.062 | 1.892 | 2.089 | $1.961 \pm 0.1059$ |
| | 0.450 | 2.008 | 2.097 | 2.123 | 2.084 | 2.420 | $2.146 \pm 0.1587$ |
| | 0.500 | 2.217 | 2.044 | 2.141 | 2.106 | 2.313 | $2.164 \pm 0.1039$ |

**Supplementary Data Tab. 39**   Average test loss corresponding to the validation data with different state values predicted by AdaptiveNN-DeiT-S (2-nd fixation).

| Model | Predicted State Value | Trial / Test Loss | | | | | |
|---|---|---|---|---|---|---|---|
| | | 1st | 2nd | 3rd | 4th | 5th | Average |
| AdaptiveNN -DeiT-S | 0.002 | 0.341 | 0.343 | 0.339 | 0.280 | 0.264 | $0.313 \pm 0.0383$ |
| | 0.032 | 0.674 | 0.671 | 0.719 | 0.734 | 0.843 | $0.728 \pm 0.0696$ |
| | 0.062 | 0.898 | 1.127 | 1.089 | 1.053 | 1.196 | $1.073 \pm 0.1110$ |
| | 0.091 | 1.225 | 1.312 | 1.264 | 1.219 | 1.520 | $1.308 \pm 0.1240$ |
| | 0.121 | 1.281 | 1.387 | 1.571 | 1.500 | 1.597 | $1.467 \pm 0.1319$ |
| | 0.151 | 1.562 | 1.586 | 1.640 | 1.697 | 1.923 | $1.682 \pm 0.1447$ |
| | 0.181 | 1.774 | 1.726 | 1.850 | 1.742 | 1.898 | $1.798 \pm 0.0735$ |
| | 0.211 | 1.884 | 1.699 | 1.942 | 1.899 | 2.256 | $1.936 \pm 0.2014$ |
| | 0.240 | 2.056 | 2.155 | 2.226 | 1.990 | 2.357 | $2.157 \pm 0.1437$ |
| | 0.270 | 2.071 | 2.159 | 2.271 | 2.194 | 2.508 | $2.241 \pm 0.1660$ |
| | 0.300 | 2.068 | 2.117 | 2.309 | 2.386 | 2.555 | $2.287 \pm 0.1995$ |

**Supplementary Data Tab. 40**   Average test loss corresponding to the validation data with different state values predicted by AdaptiveNN-DeiT-S (3-rd fixation).

| Model | Predicted State Value | Trial / Test Loss | | | | | |
|---|---|---|---|---|---|---|---|
| | | 1st | 2nd | 3rd | 4th | 5th | Average |
| AdaptiveNN -DeiT-S | 0.001 | 0.307 | 0.309 | 0.289 | 0.312 | 0.291 | $0.301 \pm 0.0109$ |
| | 0.026 | 0.762 | 0.761 | 0.857 | 0.852 | 1.086 | $0.864 \pm 0.1330$ |
| | 0.051 | 1.098 | 1.153 | 1.186 | 1.232 | 1.516 | $1.237 \pm 0.1634$ |
| | 0.076 | 1.321 | 1.427 | 1.537 | 1.534 | 1.810 | $1.526 \pm 0.1820$ |
| | 0.101 | 1.589 | 1.640 | 1.835 | 1.703 | 2.177 | $1.789 \pm 0.2356$ |
| | 0.126 | 1.873 | 1.767 | 1.997 | 2.007 | 2.192 | $1.967 \pm 0.1596$ |
| | 0.150 | 2.053 | 2.055 | 2.183 | 2.022 | 2.414 | $2.145 \pm 0.1624$ |
| | 0.175 | 2.319 | 2.306 | 2.294 | 2.356 | 2.508 | $2.356 \pm 0.0877$ |
| | 0.200 | 2.417 | 2.420 | 2.438 | 2.453 | 2.601 | $2.466 \pm 0.0771$ |
| | 0.225 | 2.446 | 2.449 | 2.582 | 2.544 | 2.647 | $2.534 \pm 0.0871$ |
| | 0.250 | 2.474 | 2.478 | 2.634 | 2.631 | 2.647 | $2.573 \pm 0.0885$ |

**Supplementary Data Tab. 41** Average test loss corresponding to the validation data with different state values predicted by AdaptiveNN-DeiT-S (4-th fixation).

| Model | Predicted State Value | Trial / Test Loss | | | | | |
|---|---|---|---|---|---|---|---|
| | | 1st | 2nd | 3rd | 4th | 5th | Average |
| AdaptiveNN -ResNet-50 | 0.002 | 0.252 | 0.343 | 0.301 | 0.277 | 0.332 | $0.301 \pm 0.0379$ |
| | 0.062 | 0.561 | 0.519 | 0.510 | 0.525 | 0.579 | $0.539 \pm 0.0297$ |
| | 0.122 | 0.731 | 0.863 | 0.677 | 0.811 | 0.707 | $0.758 \pm 0.0773$ |
| | 0.181 | 0.996 | 1.101 | 0.858 | 0.959 | 1.051 | $0.993 \pm 0.0928$ |
| | 0.241 | 1.234 | 1.203 | 1.090 | 1.082 | 1.208 | $1.163 \pm 0.0719$ |
| | 0.301 | 1.412 | 1.328 | 1.253 | 1.332 | 1.285 | $1.322 \pm 0.0597$ |
| | 0.361 | 1.493 | 1.490 | 1.495 | 1.430 | 1.430 | $1.468 \pm 0.0344$ |
| | 0.421 | 1.528 | 1.627 | 1.464 | 1.528 | 1.488 | $1.527 \pm 0.0620$ |
| | 0.480 | 1.706 | 1.649 | 1.670 | 1.701 | 1.663 | $1.678 \pm 0.0244$ |
| | 0.540 | 1.778 | 1.783 | 1.713 | 1.602 | 1.707 | $1.716 \pm 0.0733$ |
| | 0.600 | 1.846 | 1.633 | 1.766 | 1.724 | 1.693 | $1.732 \pm 0.0797$ |

**Supplementary Data Tab. 42** Average test loss corresponding to the validation data with different state values predicted by AdaptiveNN-ResNet-50 (1st fixation).

| Model | Predicted State Value | Trial / Test Loss | | | | | |
|---|---|---|---|---|---|---|---|
| | | 1st | 2nd | 3rd | 4th | 5th | Average |
| AdaptiveNN -ResNet-50 | 0.002 | 0.256 | 0.258 | 0.250 | 0.265 | 0.290 | $0.264 \pm 0.0158$ |
| | 0.041 | 0.727 | 0.836 | 0.769 | 0.809 | 0.784 | $0.785 \pm 0.0416$ |
| | 0.080 | 1.135 | 1.159 | 1.156 | 1.215 | 1.094 | $1.152 \pm 0.0437$ |
| | 0.118 | 1.348 | 1.499 | 1.428 | 1.432 | 1.374 | $1.416 \pm 0.0583$ |
| | 0.157 | 1.640 | 1.540 | 1.550 | 1.621 | 1.565 | $1.583 \pm 0.0446$ |
| | 0.196 | 1.587 | 1.770 | 1.700 | 1.707 | 1.820 | $1.717 \pm 0.0878$ |
| | 0.235 | 1.953 | 1.873 | 2.141 | 1.995 | 1.934 | $1.979 \pm 0.1003$ |
| | 0.274 | 2.180 | 2.006 | 2.074 | 2.124 | 2.069 | $2.091 \pm 0.0654$ |
| | 0.312 | 2.224 | 2.229 | 2.077 | 2.154 | 2.119 | $2.161 \pm 0.0658$ |
| | 0.351 | 2.282 | 2.242 | 2.147 | 2.185 | 2.232 | $2.218 \pm 0.0525$ |
| | 0.390 | 2.340 | 2.256 | 2.216 | 2.216 | 2.345 | $2.275 \pm 0.0642$ |

**Supplementary Data Tab. 43** Average test loss corresponding to the validation data with different state values predicted by AdaptiveNN-ResNet-50 (2nd fixation).

| Model | Predicted State Value | Trial / Test Loss | | | | | |
|-------|----------------------|-----|-----|-----|-----|-----|---------|
| | | 1$^{st}$ | 2$^{nd}$ | 3$^{rd}$ | 4$^{th}$ | 5$^{th}$ | Average |
| AdaptiveNN -ResNet-50 | 0.002 | 0.305 | 0.296 | 0.288 | 0.273 | 0.380 | 0.308 ± 0.0415 |
| | 0.029 | 0.706 | 0.900 | 0.819 | 0.686 | 0.725 | 0.767 ± 0.0900 |
| | 0.057 | 1.086 | 1.232 | 1.180 | 1.019 | 1.020 | 1.108 ± 0.0958 |
| | 0.084 | 1.392 | 1.350 | 1.347 | 1.505 | 1.312 | 1.381 ± 0.0748 |
| | 0.111 | 1.612 | 1.581 | 1.517 | 1.718 | 1.554 | 1.596 ± 0.0764 |
| | 0.139 | 1.757 | 1.800 | 1.801 | 1.756 | 1.785 | 1.780 ± 0.0221 |
| | 0.166 | 2.102 | 2.081 | 2.190 | 2.063 | 2.056 | 2.098 ± 0.0540 |
| | 0.193 | 2.131 | 2.182 | 2.194 | 2.134 | 2.123 | 2.153 ± 0.0327 |
| | 0.220 | 2.166 | 2.285 | 2.247 | 2.333 | 2.188 | 2.244 ± 0.0686 |
| | 0.248 | 2.298 | 2.388 | 2.344 | 2.532 | 2.345 | 2.381 ± 0.0901 |
| | 0.275 | 2.429 | 2.437 | 2.404 | 2.549 | 2.502 | 2.464 ± 0.0596 |

**Supplementary Data Tab. 44** Average test loss corresponding to the validation data with different state values predicted by AdaptiveNN-ResNet-50 (3rd fixation).

| Model | Predicted State Value | Trial / Test Loss | | | | | |
|-------|----------------------|-----|-----|-----|-----|-----|---------|
| | | 1$^{st}$ | 2$^{nd}$ | 3$^{rd}$ | 4$^{th}$ | 5$^{th}$ | Average |
| AdaptiveNN -ResNet-50 | 0.001 | 0.321 | 0.250 | 0.298 | 0.266 | 0.308 | 0.289 ± 0.0297 |
| | 0.018 | 0.658 | 0.687 | 0.617 | 0.753 | 0.731 | 0.689 ± 0.0547 |
| | 0.036 | 0.866 | 1.023 | 0.931 | 1.056 | 0.872 | 0.949 ± 0.0866 |
| | 0.053 | 1.339 | 1.303 | 1.259 | 1.275 | 1.428 | 1.321 ± 0.0674 |
| | 0.071 | 1.579 | 1.434 | 1.476 | 1.514 | 1.593 | 1.519 ± 0.0673 |
| | 0.088 | 1.792 | 1.826 | 1.678 | 1.833 | 1.755 | 1.777 ± 0.0633 |
| | 0.105 | 1.807 | 1.873 | 1.930 | 1.834 | 1.970 | 1.883 ± 0.0671 |
| | 0.123 | 2.092 | 2.232 | 2.116 | 2.153 | 2.089 | 2.136 ± 0.0592 |
| | 0.140 | 2.283 | 2.234 | 2.418 | 2.209 | 2.383 | 2.306 ± 0.0917 |
| | 0.158 | 2.378 | 2.296 | 2.497 | 2.362 | 2.429 | 2.392 ± 0.0752 |
| | 0.175 | 2.473 | 2.358 | 2.575 | 2.515 | 2.463 | 2.477 ± 0.0797 |

**Supplementary Data Tab. 45** Average test loss corresponding to the validation data with different state values predicted by AdaptiveNN-ResNet-50 (4th fixation).

| Model | Accuracy | Computational Cost (GFLOPs/image) | |
| :---: | :---: | :---: | :---: |
| | | w.o. Sample-adaptive Allocation | Sample-adaptive Allocation |
| AdaptiveNN -DeiT-S | 79.57 | 2.28 | **1.90** |
| | 81.05 | 3.47 | **2.55** |
| | 81.81 | 4.66 | **3.37** |
| | 82.16 | 5.85 | **4.67** |

**Supplementary Data Tab. 46** Effectiveness of sample-adaptive computation allocation on top of AdaptiveNN-DeiT-S, highlighting computational costs required to achieve equivalent accuracies.

| Model | Computational Cost (GFLOPs/image) | Accuracy (%) | | |
| :---: | :---: | :---: | :---: | :---: |
| | | Random Allocation | Anti-sample-adaptive Allocation | Sample-adaptive Allocation |
| AdaptiveNN -DeiT-S | 1.20 | 74.6 ± 0.22 | 74.2 ± 0.22 | **75.7 ± 0.22** |
| | 1.61 | 76.0 ± 0.19 | 74.6 ± 0.24 | **78.4 ± 0.20** |
| | 2.01 | 77.0 ± 0.16 | 75.2 ± 0.23 | **79.9 ± 0.17** |
| | 2.42 | 77.9 ± 0.09 | 75.9 ± 0.27 | **80.8 ± 0.18** |
| | 2.82 | 78.7 ± 0.09 | 76.5 ± 0.32 | **81.4 ± 0.13** |
| | 3.23 | 79.3 ± 0.07 | 77.2 ± 0.32 | **81.7 ± 0.06** |
| | 3.63 | 80.0 ± 0.06 | 77.9 ± 0.28 | **81.9 ± 0.05** |
| | 4.04 | 80.5 ± 0.04 | 78.6 ± 0.24 | **82.1 ± 0.09** |
| | 4.44 | 81.0 ± 0.07 | 79.4 ± 0.16 | **82.1 ± 0.11** |
| | 4.85 | 81.5 ± 0.07 | 80.2 ± 0.14 | **82.2 ± 0.11** |
| | 5.25 | 81.8 ± 0.07 | 81.0 ± 0.11 | **82.2 ± 0.12** |

**Supplementary Data Tab. 47** Comparisons of different sample-wise computation allocation strategies on top of AdaptiveNN-DeiT-S. We report accuracy (%) and standard deviation (±) for each allocation method at various computational costs (GFLOPs/image).

| Backbone | Accuracy | Computational Cost (GFLOPs/image) | |
|---|---|---|---|
| | | **w.o. Sample-adaptive Allocation** | **Sample-adaptive Allocation** |
| AdaptiveNN -ResNet-50 | 77.39 | 2.48 | **2.12** |
| | 78.80 | 3.86 | **3.14** |
| | 79.52 | 5.25 | **4.26** |
| | 79.77 | 6.63 | **5.64** |

**Supplementary Data Tab. 48** Effectiveness of sample-adaptive computation allocation on top of AdaptiveNN-ResNet-50, highlighting computational costs required to achieve equivalent accuracies.

| Model | Computational Cost (GFLOPs/image) | Accuracy (%) | | |
|---|---|---|---|---|
| | | **Random Allocation** | **Anti-sample-adaptive Allocation** | **Sample-adaptive Allocation** |
| AdaptiveNN -ResNet-50 | 1.20 | 73.2 ± 0.07 | 73.0 ± 0.08 | **74.0 ± 0.07** |
| | 1.71 | 74.4 ± 0.11 | 73.4 ± 0.04 | **76.4 ± 0.13** |
| | 2.21 | 75.3 ± 0.13 | 73.8 ± 0.09 | **77.6 ± 0.16** |
| | 2.72 | 76.2 ± 0.13 | 74.3 ± 0.10 | **78.4 ± 0.09** |
| | 3.22 | 76.9 ± 0.11 | 74.9 ± 0.10 | **78.9 ± 0.08** |
| | 3.73 | 77.5 ± 0.15 | 75.5 ± 0.09 | **79.3 ± 0.08** |
| | 4.23 | 78.1 ± 0.10 | 76.1 ± 0.09 | **79.5 ± 0.10** |
| | 4.74 | 78.5 ± 0.08 | 76.8 ± 0.09 | **79.7 ± 0.11** |
| | 5.24 | 79.0 ± 0.09 | 77.7 ± 0.11 | **79.7 ± 0.10** |
| | 5.75 | 79.4 ± 0.04 | 78.5 ± 0.11 | **79.8 ± 0.08** |
| | 6.25 | 79.7 ± 0.05 | 79.3 ± 0.07 | **79.8 ± 0.07** |

**Supplementary Data Tab. 49** Comparisons of different sample-wise computation allocation strategies on top of AdaptiveNN-ResNet-50. We report accuracy (%) and standard deviation (±) for each allocation method at various computational costs (GFLOPs/image).



\end{document}

%% file: config.tex
\usepackage{microtype}
\usepackage{graphicx}
\usepackage{subfigure}
\usepackage{booktabs} 

\usepackage{hyperref}

\usepackage[utf8]{inputenc} 
\usepackage[T1]{fontenc}    
\usepackage{url}            
\usepackage{amsfonts}       
\usepackage{nicefrac}       
\usepackage[table]{xcolor}         
\usepackage{colortbl}

\usepackage{amsmath}
\usepackage{amssymb}
\usepackage{mathtools}
\usepackage{amsthm}
\usepackage{mathrsfs}
\usepackage{xspace}
\usepackage{academicons}

\usepackage[capitalize,noabbrev]{cleveref}

\theoremstyle{plain}
\newtheorem{theorem}{Theorem}

\theoremstyle{definition}

\theoremstyle{remark}

\theoremstyle{plain}
\newtheorem*{theorem*}{Theorem}
\newtheorem*{proposition*}{Proposition}
\newtheorem*{lemma*}{Lemma}
\newtheorem*{corollary*}{Corollary}

\theoremstyle{definition}
\newtheorem*{definition*}{Definition}
\newtheorem*{assumption*}{Assumption}

\theoremstyle{remark}
\newtheorem*{remark*}{Remark}

\definecolor{bigaired}{RGB}{156, 0, 0}
\definecolor{uclablue}{RGB}{39, 116, 174}
\definecolor{thupurple}{RGB}{102, 8, 116}
\definecolor{pkured}{RGB}{139, 0, 18}
\definecolor{panton}{RGB}{217, 51, 121}

\definecolor{darkred}{RGB}{200, 0, 0}
\definecolor{darkblue}{RGB}{0, 0, 200}
\definecolor{blue}{RGB}{0, 0, 200}

\definecolor{light}{RGB}{225, 250, 250}
\definecolor{lightgray}{RGB}{0.9, 0.9, 0.9}
\definecolor{lightred}{RGB}{250, 200, 200}
\definecolor{lightblue}{RGB}{210, 220, 250}
\definecolor{lightpurple}{RGB}{218,210,255}

\definecolor{doderblue}{RGB}{30, 144, 255}
\definecolor{select}{RGB}{222, 235, 247}
\definecolor{unselect}{RGB}{247, 207, 206}

\definecolor{myLinkColor}{RGB}{0, 0, 200}     
\definecolor{myCiteColor}{RGB}{0, 0, 200}     
\definecolor{myURLColor}{RGB}{0, 0, 200}      
\hypersetup{
  colorlinks=true,
  linkcolor=myLinkColor,
  citecolor=myCiteColor,
  urlcolor=myURLColor
}

\usepackage[textsize=tiny]{todonotes}

\usepackage[most]{tcolorbox}

\usepackage{listings}

\usepackage{fontawesome5}  
\usepackage{listings}      
\usepackage[misc]{ifsym}
\usepackage{wrapfig}    
\usepackage[T1]{fontenc}

\usepackage{tcolorbox}
\usepackage{relsize}

\usepackage{adjustbox}
\usepackage{fancyhdr}
\usepackage{lipsum}
\usepackage{newtxtext}
\usepackage{wrapfig}
\usepackage{enumitem}
\definecolor{azblue}{RGB}{27,117,187}      

\usepackage[noend]{algpseudocode}   

\definecolor{bestcol}{RGB}{  0,102,204} 
\definecolor{goodcol}{RGB}{ 34,139, 34} 
\definecolor{deltaBg}{RGB}{220,230,255} 

\usepackage{lmodern}  

\usepackage{tikz}
\usetikzlibrary{tikzmark,decorations.pathreplacing,calc}
\usepackage{stackengine}
\stackMath
\definecolor{lightgreen}{RGB}{0,150,0}  

\usepackage{titletoc}

\newtheoremstyle{rqstyle}%
  {\topsep}            
  {\topsep}            
  {}                   
  {}                   
  {\bfseries}    
  {:}                  
  {.5em}               
  {}                   

\theoremstyle{rqstyle}

\crefname{researchquestion}{Research Question}{Research Questions}

\definecolor{propose}{HTML}{EF8E8D}
\definecolor{solve}{HTML}{5755A3}

\definecolor{humanred}{RGB}{180, 50, 50}
\definecolor{envgreen}{RGB}{50, 140, 80}

\definecolor{paleviolet}{HTML}{E1EEFC}
\definecolor{lightgrey}{RGB}{247, 247, 247}
\newenvironment{leapabstract}{
  \begin{tcolorbox}[
    colback=lightgrey,
    colframe=white,
    boxrule=0pt,
    arc=10pt,
    left=16pt,
    right=16pt,
    top=12pt,
    bottom=12pt,
    width=\textwidth,
    enlarge left by=0mm,
    before skip=10pt,
    after skip=10pt
  ]
  \normalsize
}{
  \end{tcolorbox}
}

\makeatletter
\DeclareRobustCommand\onedot{\futurelet\@let@token\@onedot}
\def\@onedot{\ifx\@let@token.\else.\null\fi\xspace}

\makeatother